\documentclass[10pt,twocolumn,letbxterpaper]{article}

\usepackage{cvpr}
\usepackage{times}
\usepackage{epsfig}
\usepackage{graphicx}
\usepackage{amsmath,bm}
\usepackage{amssymb}
\usepackage{kotex}
\usepackage{algorithm,algpseudocode}
\usepackage{caption}
\usepackage[labelformat = empty,position=top]{subcaption}
\usepackage[export]{adjustbox}
\usepackage[titletoc,title]{appendix}
\usepackage{multirow}
\usepackage{amsthm}

\newif\ifdraft\draftfalse
\ifdraft

\newcommand\yj[1]{\textcolor{black}{#1}}

\newcommand\sd[1]{\textcolor{green}{#1}}
\newcommand\nj[1]{\textcolor{black}{#1}}
\newcommand\knj[1]{\textcolor{red}{#1}}
\else

\newcommand\yj[1]{#1}

\newcommand\sd[1]{#1}
\newcommand\nj[1]{#1}
\newcommand\knj[1]{#1}
\fi

\newtheorem{theorem}{Theorem}
\newtheorem{proposition}{Proposition}

\usepackage[pagebackref=true,breaklinks=true,letterpaper=true,colorlinks,bookmarks=false]{hyperref}

\cvprfinalcopy 

\def\cvprPaperID{2760} 

\ifcvprfinal\fi
\begin{document}

\title{Butterfly Effect: Bidirectional Control of Classification Performance\\ 
by Small Additive Perturbation}

\author{YoungJoon Yoo\textsuperscript{1} \quad Seonguk Park\textsuperscript{1} \quad Junyoung Choi\textsuperscript{1} \quad Sangdoo Yun\textsuperscript{2} \quad Nojun Kwak \textsuperscript{1}\\
\textsuperscript{1}Graduate School of Convergence Science and Technology, Seoul National University, South Korea\\
\textsuperscript{2}ASRI, Dept. of Electrical and Computer Eng., Seoul National University, South Korea\\
{\tt\small 
\textsuperscript{1}yjyoo3312@gmail.com 
\textsuperscript{2}\{swpark0703, djcola814, yunsd101, nojunk\}@snu.ac.kr 
}
}

\maketitle

\begin{abstract}
This paper proposes a new algorithm for controlling classification results by generating a small additive perturbation without changing the classifier network. Our work is inspired by existing works generating adversarial perturbation that worsens classification performance.
In contrast to the existing methods, our work aims to generate perturbations that can enhance overall classification performance.
To solve this performance enhancement problem, we newly propose a perturbation generation network (PGN) influenced by \nj{the adversarial learning strategy}.
In our problem, the information in a large external dataset is \nj{summarized by }
a small additive perturbation, \nj{which} helps to improve the performance of the classifier trained with the target dataset.
In addition to \nj{this performance enhancement} problem, we show that the proposed PGN can be adopted to solve the classical adversarial problem without \nj{utilizing the information on} the target classifier.
The mentioned characteristics of our method are verified through \nj{extensive experiments on publicly available} visual datasets.

\end{abstract}

\section{Introduction}
\label{sec:intro}
In recent years, deep convolutional neural networks \nj{(CNN)}~\cite{lecun1989backpropagation,krizhevsky2012imagenet} \nj{have} become one of the most powerful ways to handle visual information and have been applied to almost all areas of computer \nj{vision,} 
including classification~\cite{simonyan2014very,szegedy2015going,he2016deep}, detection~\cite{redmon2016you,girshick2015fast,ren2015faster}, and segmentation~\cite{long2015fully,noh2015learning}, among others.
\nj{It has been} shown that deep networks stacking multiple layers \nj{provide} sufficient capacity to extract essential features from visual data for \nj{a} computer vision task.
To efficiently estimate the large number of the model's network parameters, \nj{stochastic gradient descent (SGD) and its variants}~\cite{tieleman2012lecture,kingma2014adam}, which update the network parameters through the gradient obtained by backpropagation~\cite{lecun1989backpropagation}, \nj{have been} proposed. 

However, recent studies~\cite{goodfellow2014explaining,nguyen2015deep,Moosavi-Dezfooli_2017_CVPR,szegedy2013intriguing} suggest that the estimated network parameters are not optimal, and the trained networks are easily fooled by adding \nj{a} small perturbation vector generated by solving an optimization problem~\cite{szegedy2013intriguing} or by \nj{one-step} gradient ascent~\cite{goodfellow2014explaining}, as shown in Figure~\ref{fig:teaser}. 
Also, the generation of universal perturbation that can degrade arbitrary images and networks \nj{has been} proposed~\cite{Moosavi-Dezfooli_2017_CVPR}.
From the results, we can \nj{conjecture} that the trained \nj{networks} are over-fitted in some sense. 

\begin{figure}[t]
\begin{center}
   \includegraphics[width=0.99\linewidth]{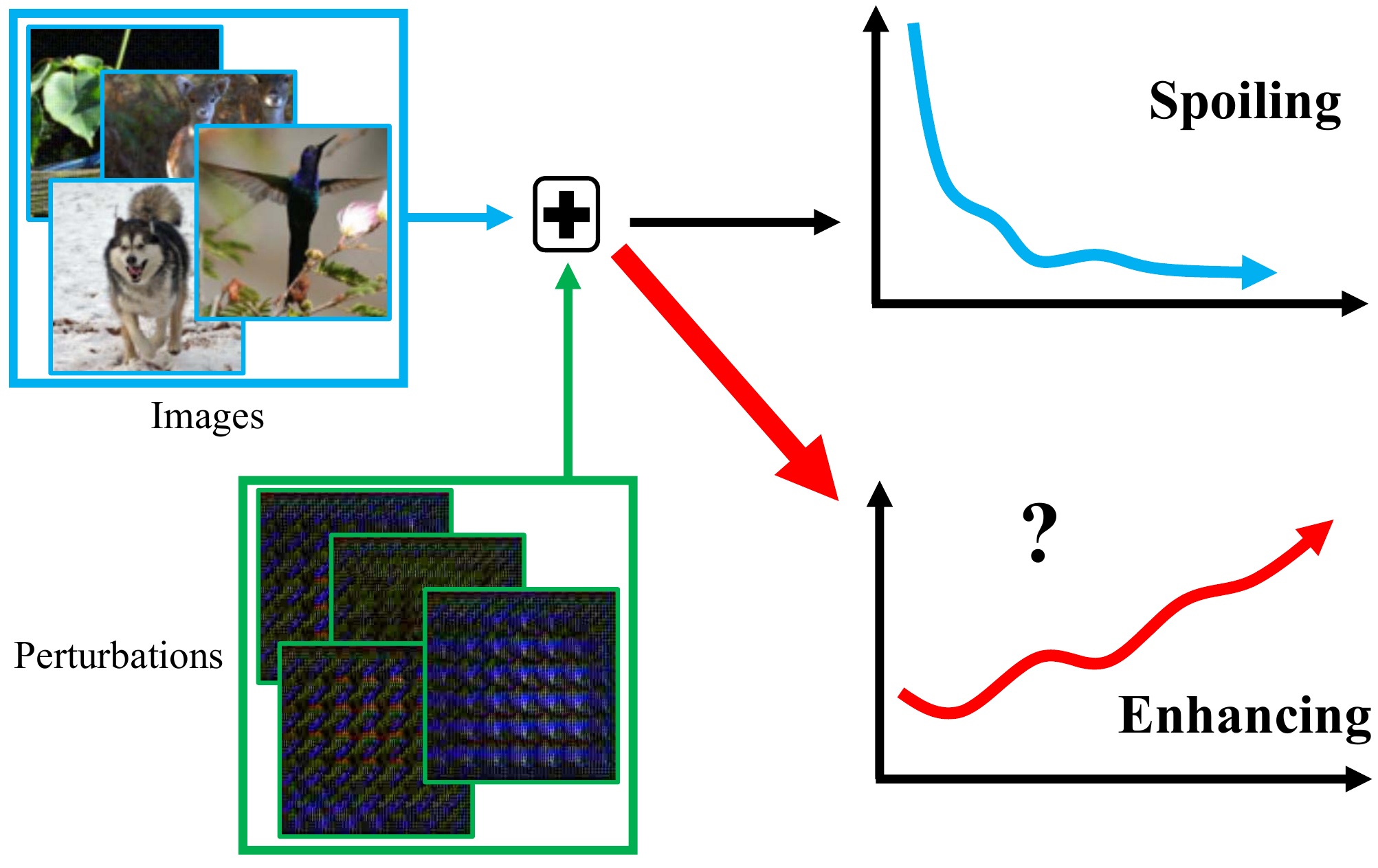}
\end{center}
   \vspace{-2mm}
   \caption{Bidirectional control of classification performance using small additive perturbation. Various approaches have shown that CNN based classifiers can be easily fooled.
   Our model aims to find the perturbation that can control the CNN classification performance in both directions: enhancement, and degradation.}
   \vspace{-2mm}
\label{fig:teaser}
\end{figure}

These works show that it is possible to control the target performance through small external changes without modifying the \nj{values of network parameters}, suggesting that it is possible to generate a perturbation to improve the performance of the model.
Regarding the issue of generating adversarial perturbation, studies including privacy applications~\cite{orekondy2017towards,lu2017no,oh2017adversarial} and defenses of the adversarial perturbation~\cite{gu2014towards, metzen2017detecting} \nj{have been} proposed so far. 
However, \nj{to the best of our knowledge,} designing a perturbation that enhances the performance of a model has not been proposed yet.

In this paper, we propose a new \nj{general framework for generating a perturbation vector that can either enhance or worsen the classification accuracy.} 
The proposed algorithm \nj{solves} two main problems. 
First and most importantly, our model generates \nj{a} perturbation that enhances classification performance without changing the network parameters (\textbf{enhancing problem}).
It is \nj{worth noting} that this is \nj{the} first attempt to show that performance-enhancing \nj{perturbations} exist.
Second, our algorithm generates \nj{a} perturbation vector \nj{so as to lower the classification performance of the classifier} (\textbf{adversarial problem}).
For the adversarial problem, our algorithm can generate \nj{perturbations} without knowing the \nj{structure} of the network \nj{being used}, which is difficult for existing adversarial algorithms~\cite{Moosavi-Dezfooli_2017_CVPR,moosavi2016deepfool,goodfellow2014explaining}.

To solve this problem, we propose a perturbation vector generation network \nj{(PGN)} consisting of two \nj{sub-networks:} generator and discriminator.
The generator network generates a perturbation vector that coordinates the target performance in the \nj{desired} direction, 
and the discriminator network \nj{distinguishes} whether the \nj{generated perturbation is good or bad.}
Both networks are trained through minimax games inspired by \nj{\textit{generative adversarial nets}} (GAN)~\cite{goodfellow2014generative}, and the resultant perturbation vector from the generator controls the result of the target classifier networks.
However, unlike those of the variants of GAN~\cite{goodfellow2014generative,mao2016least,radford2015unsupervised,bradshaw2017adversarial}, the \knj{purpose} of the proposed minimax framework is to generate additive \knj{noises} that help the input data satisfy the desired \knj{goal of performance-enhancement or degradation, not generating a plausible data samples.}
The main contributions of the proposed work can be summarized as follows: 
\begin{itemize}
\item We show the \nj{existence of a} perturbation vector \nj{that enhances the overall classification result of a dataset.}
\item We propose a unified framework \nj{of PGN that can solve} the performance-enhancement, and the adversarial problem.
\item We show that the proposed method can generate perturbation vectors that can solve the adversarial problem without knowing the \nj{structure of the network}.
\end{itemize}

The proposed method has been validated with \nj{a couple of} public \nj{datasets:} the STL-10 dataset~\cite{coates2011analysis}, and subsets of the ImageNet dataset~\cite{russakovsky2015imagenet}. 
Also, \nj{widely used classifier networks} 
such as ResNet~\cite{he2016deep}, VGGnet~\cite{simonyan2014very}, and DenseNet~\cite{huang2016densely} \nj{have been tested }\nj{as the target classifier}.

\section{Related Work}
\label{sec:rel}

In contrast to the great success of CNN in \nj{various} image recognition \nj{tasks}~\cite{simonyan2014very,szegedy2015going,he2016deep,szegedy2017inception}, many studies~\cite{biggio2013evasion,szegedy2013intriguing,moosavi2016deepfool,bastani2016measuring,sabour2015adversarial,tabacof2016exploring,fawzi2015analysis,fawzi2016robustness,rodner2016fine,rozsa2016adversarial,goodfellow2014explaining,Moosavi-Dezfooli_2017_CVPR,moosavi2017analysis,mopuri2017fast} have indicated that CNNs are not robust and are \knj{easily fooled.}
Szegedy~\etal~\cite{szegedy2013intriguing} discovered that such classification networks are vulnerable to well-designed additive small perturbations.
These perturbation vectors can be estimated either by solving an optimization problem~\cite{szegedy2013intriguing,moosavi2016deepfool,bastani2016measuring} or \nj{by one-step} gradient ascent of the network~\cite{goodfellow2014explaining}.
Also, studies~\cite{mordvintsev2015deepdream,nguyen2015deep} \nj{have} been published that \nj{show} the difference \nj{between CNNs and humans in understanding an image.}
These works generate a perturbation vector depending \nj{both} on \nj{the} input image and on \nj{the} network \nj{used}.
\nj{On the other hand}, \nj{the work in ~\cite{hayes2017machine} 
generate an image-specific universal adversarial perturbation vector valid for arbitrary networks,
while \cite{Moosavi-Dezfooli_2017_CVPR,moosavi2017analysis, mopuri2017fast} find a universal adversarial perturbation vector independent of images.
}
%

The discovery of \nj{an} adversarial example has attracted a great deal of attention in relation to privacy issues~\cite{orekondy2017towards,lu2017no,oh2017adversarial}, and many studies have been published on the privacy and defense~\cite{gu2014towards, metzen2017detecting,zantedeschi2017efficient, lu2017safetynet, das2017keeping, hosseini2017blocking} of \nj{adversarial examples}. 
Studies \nj{have also been proposed for tasks such as} transferring the adversarial example \nj{to other networks}~\cite{liu2016delving, bradshaw2017adversarial}, transforming \nj{an input image} to its target class \nj{by adding a perturbation}~\cite{baluja2017adversarial}, and generating an adversarial perturbation vector~\cite{xie2017adversarial} for segmentation and detection.

The main issues we deal with in this paper are different from those in the studies mentioned above.
Unlike the previous works focusing on the adversarial problem and its defense, our works mainly \nj{aim} to propose a network that generates \nj{a} perturbation vector that can enhance the overall classification performance of the target classifier.
Furthermore, in addition to the enhancing problem, the proposed network is designed so that it is also applicable to the adversarial problem 
with an unknown black-box target classifier.

\begin{figure*}[t]
\begin{center}
   \includegraphics[width=0.99\linewidth]{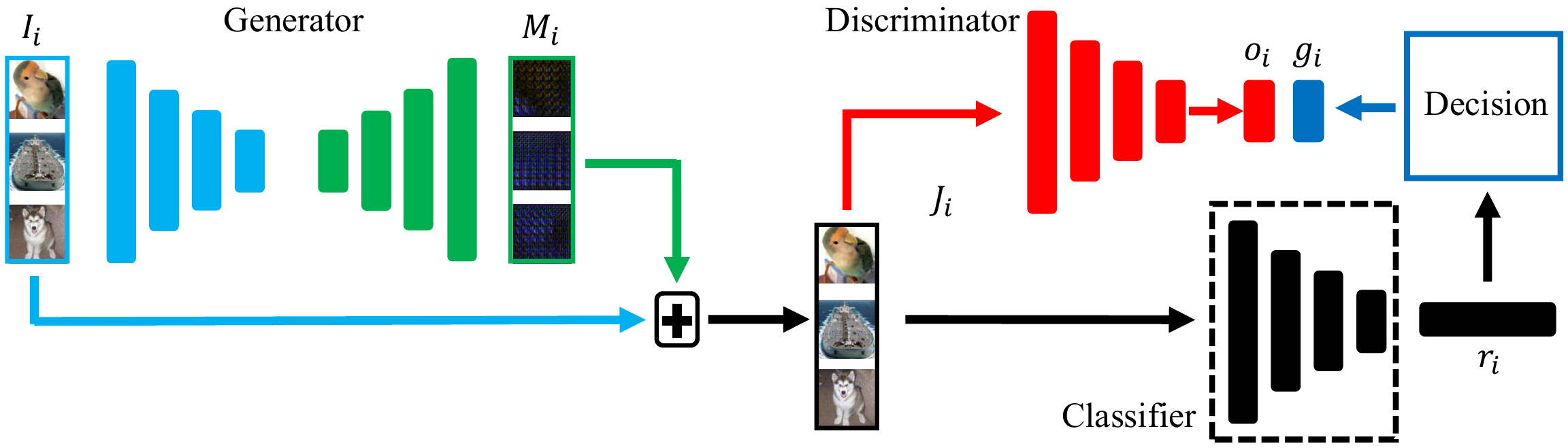}
\end{center}
\vspace{-2mm}
   \caption{Framework of the proposed method. For controlling the performance of the target classifier (Black), a perturbation $M_i$ (Green) is generated from a base image $I_i$ (Blue) by the generator. The discriminator (Red) then judges if the perturbation has adjusted the classification result \nj{as desired}. 
}
\vspace{-2mm}
\label{fig:framework}
\end{figure*}

\section{Proposed Method}
\label{sec:prop}
\subsection{Overview}
\label{sec:prop_overview}
In \nj{an} image classification framework, we are given a set of labeled images $T = \{(I_i,l_i),i=1,...,N\}$ and classifier $r_i = f_c(I_i)$, 
where \nj{$I_i$ and $l_i$ \knj{denote} the image and the corresponding label, respectively,} and the resultant label $r_i$ \knj{belongs to} one of the class labels $r_i \in\{c_1,...,c_K\}$. 
Given this condition, our goal is to generate an additive perturbation vector $M_i$ that can control the classification result of the classifier $f_c(\cdot)$, given an image $I_i$.
The generated perturbation vector $M_i$ is added to the image $I_i$ \nj{as}  
%
\begin{equation} 
\label{eq:additive_noise}
J_i = I_i + \lambda M_i,
\end{equation}
and the classification result $f_c(J_i)$ of the perturbed image $J_i$ should be controlled by the vector $M_i$ so \nj{as to} solve the two listed problems: \nj{the} enhancing problem\, and \nj{the} adversarial problem.

For \nj{the} enhancing problem, we \knj{aim} to generate a perturbation vector under the condition where the classifier network is accessible, but \nj{with fixed parameters}. 
For the adversarial problem, our model \knj{solves} the problem under the situation \nj{that} the classifier network is not accessible at all (Black Box).

\subsection{Perturbation Generation Network}
\label{sec:prop_pgn}

Figure~\ref{fig:framework} describes the overall framework of the proposed \nj{PGN, which}
mainly consists of three networks: \nj{a} generator, \nj{a} discriminator \nj{and a} classifier.
In our problem, only the network parameters \nj{in} \knj{the} generator and \knj{the} discriminator will be updated.

\knj{As in equation~(\ref{eq:generator}), the generator network $G(I_i;W_G)$ generates a perturbation vector $M_i$ with the same size as the input image $I_i$,}
where $W_G$ refers to the network parameters of the generator;
\begin{equation} 
\label{eq:generator}
M_i = G(I_i;W_G).
\end{equation}
In our model, $G(I_i;W_G)$ is composed of an encoder network $v = E(I_i;W_{enc})$ and decoder network \nj{$M_i = D(v;W_{dec})$}, where $W_G = \{W_{enc},W_{dec}\}$.
Using the vector $M_i$ and the image $I_i$, we get a perturbed image $J_i$ as in equation~(\ref{eq:additive_noise}).
The perturbed image $J_i$ then \nj{bifurcates as inputs to} the classifier $f_c(J_i)$ \nj{as well as} the discriminator $D(J_i)$.
In our model, the discriminator $D(J_i;W_D)$ is designed as a \knj{network with a sigmoid output} as in equation~(\ref{eq:discriminator}) to judge whether \nj{$J_i$} is generated according to our purpose, by using the classification result of the given target classifier $f_c(\cdot)$\nj{;}
\begin{equation} 
\label{eq:discriminator}
o_i = D(J_i;W_D),~o_i\in[0,1]\nj{.}
\end{equation}
\nj{Here,} the term $W_D$ denotes the network parameters of the discriminator, and the term $o_i$ is a sigmoid scalar unit.
The important thing here is to set the target variable $g_i$ for the output $o_i$ of the discriminator to fit the purpose of the problem we aim to solve.
Then, the loss functions for training the generator and discriminator \nj{networks are} defined using $g_i$ and $o_i$.
For each of the two problems we want to solve in this work, detailed explanations will be presented in \nj{what follows}.
The overall algorithm will be presented with \nj{the case of the enhancing problem}, and the \knj{case of adversarial problem} will be addressed based on the discussion.

\textbf{Enhancing \nj{Problem:} }
In order to enhance the \nj{performance of the classifier}, we first define a discriminator loss to let the network 
\knj{determine whether the generated perturbed image $J_i$ is good or bad.}
When the \nj{classification result of the generated image $J_i$ matches the ground truth $l_i$}, 
we set the \knj{target $g_i$ as $1$ (good)}, and otherwise \knj{$0$ (bad)} as follows:
\begin{equation} 
\label{eq:disc_gt}
g_i = 
\begin{cases}
1,~r_i = l_i,\\
0,~r_i \neq l_i.
\end{cases}
\end{equation}

\nj{Here,} $r_i = f_c(J_i)$ and $l_i$ is the \nj{ground truth} class label for $I_i$.
Using the target variable $g_i$, \nj{ the discriminator and the generator losses are defined in the sense of mean squared error as in equations~(\ref{eq:dsc_loss}) and (\ref{eq:fool_loss}), respectively;}

\nj{
\begin{eqnarray} 
&L_{d} = \frac{1}{2}\mathbf{E}_{p_g}[{(D(J)-1)}^2]+\frac{1}{2}\mathbf{E}_{p_{\bar{g}}}[{(D(J))}^2] \label{eq:dsc_loss}, \\
&L_{g} = \frac{1}{2}\mathbf{E}_{p_g}[{(D(J)-1)}^2]+\frac{1}{2}\mathbf{E}_{p_{\bar{g}}}[{(D(J)-1)}^2] \label{eq:fool_loss}.
\end{eqnarray}
}
\nj{
The distributions 
$p_g$ and $p_{\bar{g}}$ denote 
$p(g = 1) = p(r=l)$ and $p(g = 0) = p(r \neq l)$, respectively.
Note that both $r$ and $g$ depends on the generated sample $M$.
}

\nj{In practice, the expectations in (\ref{eq:dsc_loss}) and (\ref{eq:fool_loss}) are replaced with empirical \knj{means} as follows:
\begin{eqnarray} 
&L_{d} = \frac{1}{2} \sum_{i} [g_i{(D(J_i)-1)}^2 + (1-g_i) {D(J_i)}^2] \label{eq:dsc_loss_emp} \\
&L_{g} = \frac{1}{2} \sum_{i} {(D(J_i)-1)}^2. \label{eq:fool_loss_emp}
\end{eqnarray}
}

These generator and discriminator losses are inspired by least-square GAN (LSGAN)~\cite{mao2016least}, and we train the discriminator and generator \knj{networks} to minimize each loss with respect to $W_D$ and $W_G$, \nj{respectively}.
However, our formulation is different from \nj{that of} \cite{mao2016least} \nj{as clearly shown in (\ref{eq:fool_loss_emp})}.
The proposed scheme is designed to make every \nj{$D(J_i)$} converge to $1$, which means that our learning scheme reaches the proposed goal \nj{of enforcing correct classification.}
In implementation, $\ell_1$ regularization loss $L_{r} = \frac{1}{N}\sum_{i}{\|M_i \|}_1$ is added to the generator loss in equation~(\ref{eq:fool_loss_emp}) to control the intensity of the perturbation, as \nj{the following:}
\begin{equation} 
\label{eq:gen_loss}
L_{g'} = L_{g} + \gamma L_{r}.
\end{equation}

Qualitatively, minimizing the loss in equation~(\ref{eq:dsc_loss}) means that the output of the discriminator $D(J_i)$ goes to $1$ when $g_i$ equals to $1$ \knj{(good)}, and goes to $0$ in the opposite case \knj{(bad)}. 
Similarly, minimizing $W_G$ of the generator implies that \nj{$M_i = G(I_i)$} is trained so that the output of the discriminator $D(J_i)$ goes to $1$ in every case, by deceiving the discriminator. 
We have shown the proposed minimax game using the equations~(\ref{eq:dsc_loss}),(\ref{eq:fool_loss}) theoretically makes $p_g(J_i)$ converge to $1$. \knj{We have also proven that similar to~\cite{goodfellow2014generative}, this scheme is valid when a 
cross-entropy loss is applied instead of least-square loss in (\ref{eq:dsc_loss}) and (\ref{eq:fool_loss}).} Detailed explanation and proof are provided in Appendix~\ref{app:convergence}.

\textbf{Adversarial Problem: }
\nj{We can} generate an adversarial perturbation vector $M_i$ without much changing \nj{
the previously described model for the performance enhancement problem}.
In \nj{the case of the adversarial problem}, the discriminator should count the vector $M_i$ as success when the classification result $f_c(J_i)$ becomes different from the ground truth $l_i$. Therefore, in this case, the target vector $g_i$ is defined as in equation~(\ref{eq:adv_gt}),
\begin{equation} 
\label{eq:adv_gt}
g_i = 
\begin{cases}
1,~r_i \neq l_i,\\
0,~r_i = l_i.
\end{cases}
\end{equation}
From the experiments, the proposed minimax framework with the discriminator and the generator \nj{losses defined as in equations (\ref{eq:dsc_loss_emp}) and (\ref{eq:fool_loss_emp})} has sufficient capacity to drop the classification performance.
\knj{One thing that}
is worth mentioning 
\knj{is} that the existing works solve the problem \nj{based} on the assumption that the network framework is given \nj{while we can do so without knowing the network framework}.

\begin{algorithm} [t]
\caption{Training procedure of the proposed PGN.}
\begin{algorithmic}[1]
\Require 
Training data \(\{I_{i}, l_{i}\}\), target classifier $f_c(\cdot)$.
\Ensure Trained PGN weights \(W_{G}\) and \(W_{D}\)
\State Initialize \(W_{G}\) and \(W_{D}\)
\Repeat 
	\State $J_i \leftarrow I_i + \lambda M_i,~M_i\leftarrow G(I_i;W_G)$, in eq.~(\ref{eq:additive_noise}), (\ref{eq:generator}).
	\State $r_i\leftarrow f_c(J_i)$.
    \State $o_i\leftarrow D(J_i; W_D)$, in eq.~(\ref{eq:discriminator}).
    \State Get \(g_{i}\) \nj{using} $r_i$ \nj{by} eq.~(\ref{eq:disc_gt}) or (\ref{eq:adv_gt})
    \State Get \(L_{d}\) from \(o_{i}\) and \(g_{i}\), using eq.~(\ref{eq:dsc_loss_emp}).
    \State Update \(W_{D}\) using \(L_{d}\)
    \State $o_i\leftarrow D(J_i; W_D)$, in eq.~(\ref{eq:discriminator}).
    \State Get \(L_{g}\) from \(o_{i}\), \(r_{i}\), and \(l_{i}\), using eq.~(\ref{eq:gen_loss}).
    \State Update \(W_{G}\) using \(L_{g}\).
\Until{\(L_{g}\) converges}
\end{algorithmic}
\label{algo:train_pgn}
\end{algorithm}

\begin{figure*}[t]
\centering
\begin{subfigure}[t]{0.03\textwidth}
\textbf{(A)}
\end{subfigure}
\begin{subfigure}[t]{0.23\textwidth}
\includegraphics[width=\linewidth,valign=t]{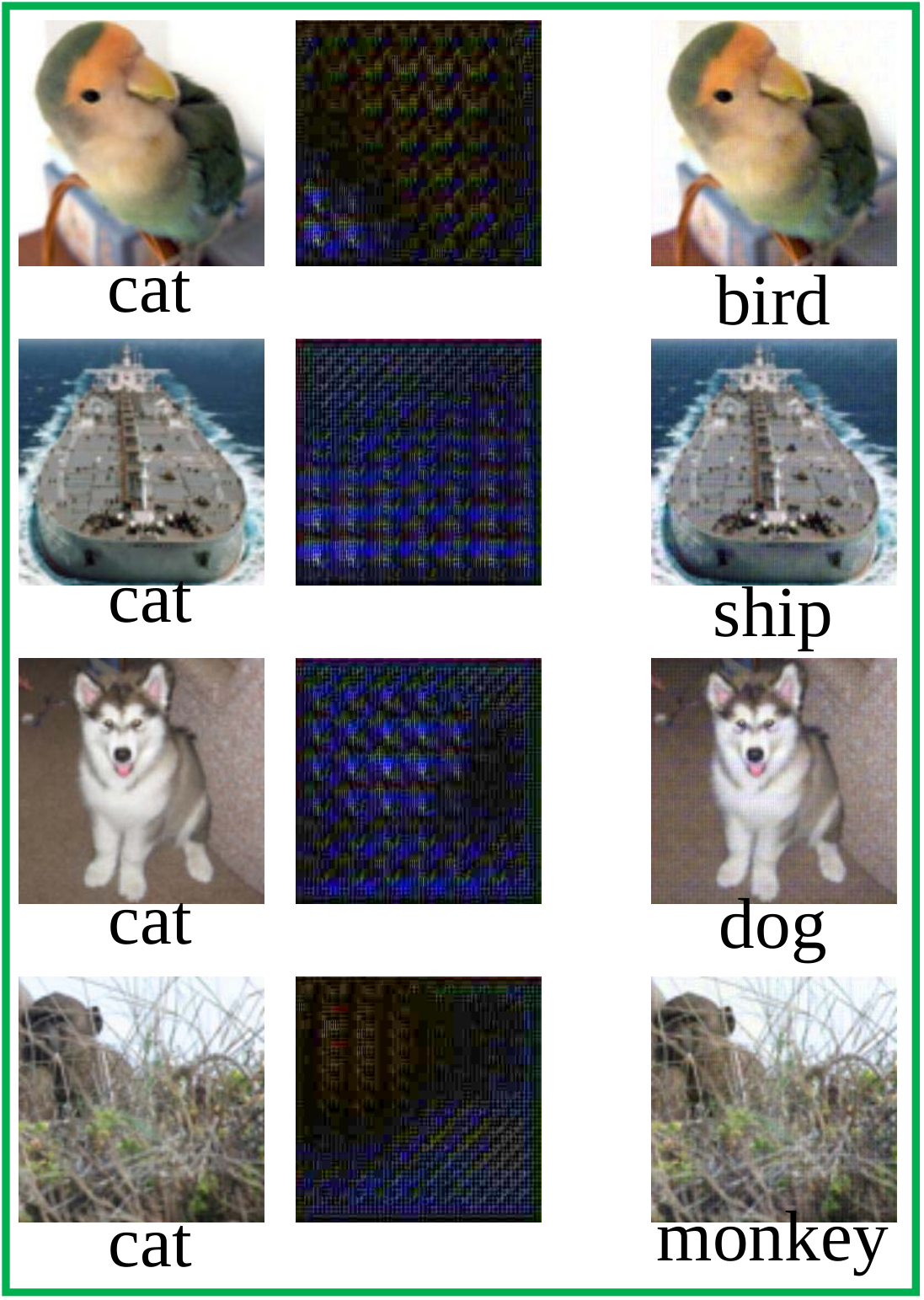}
\end{subfigure}\hfill
\begin{subfigure}[t]{0.23\textwidth}
\includegraphics[width=\linewidth,valign=t]{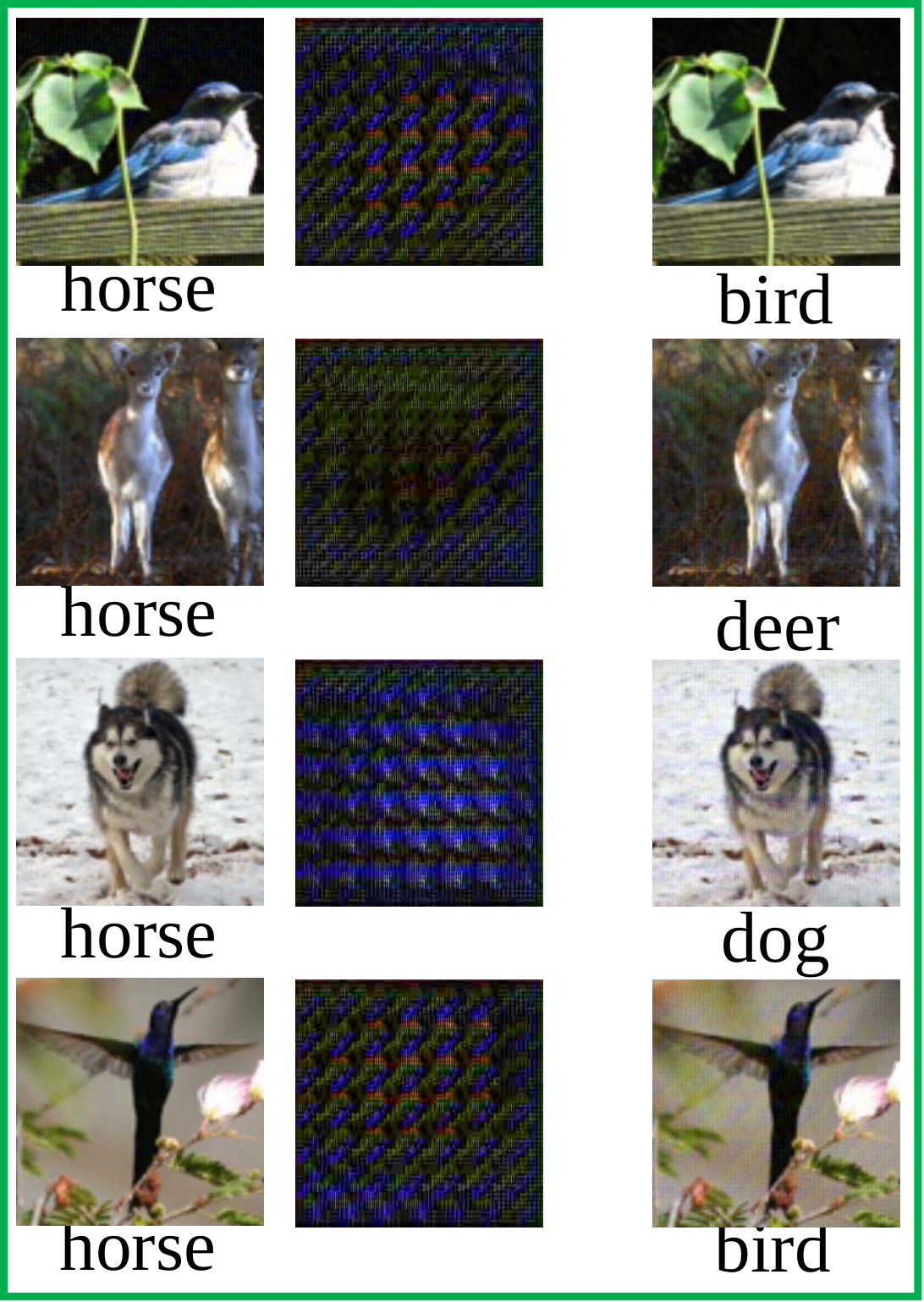}
\end{subfigure}\hfill
\begin{subfigure}[t]{0.03\textwidth}
\textbf{(B)}
\end{subfigure}
\begin{subfigure}[t]{0.23\textwidth}
\includegraphics[width=\linewidth,valign=t]{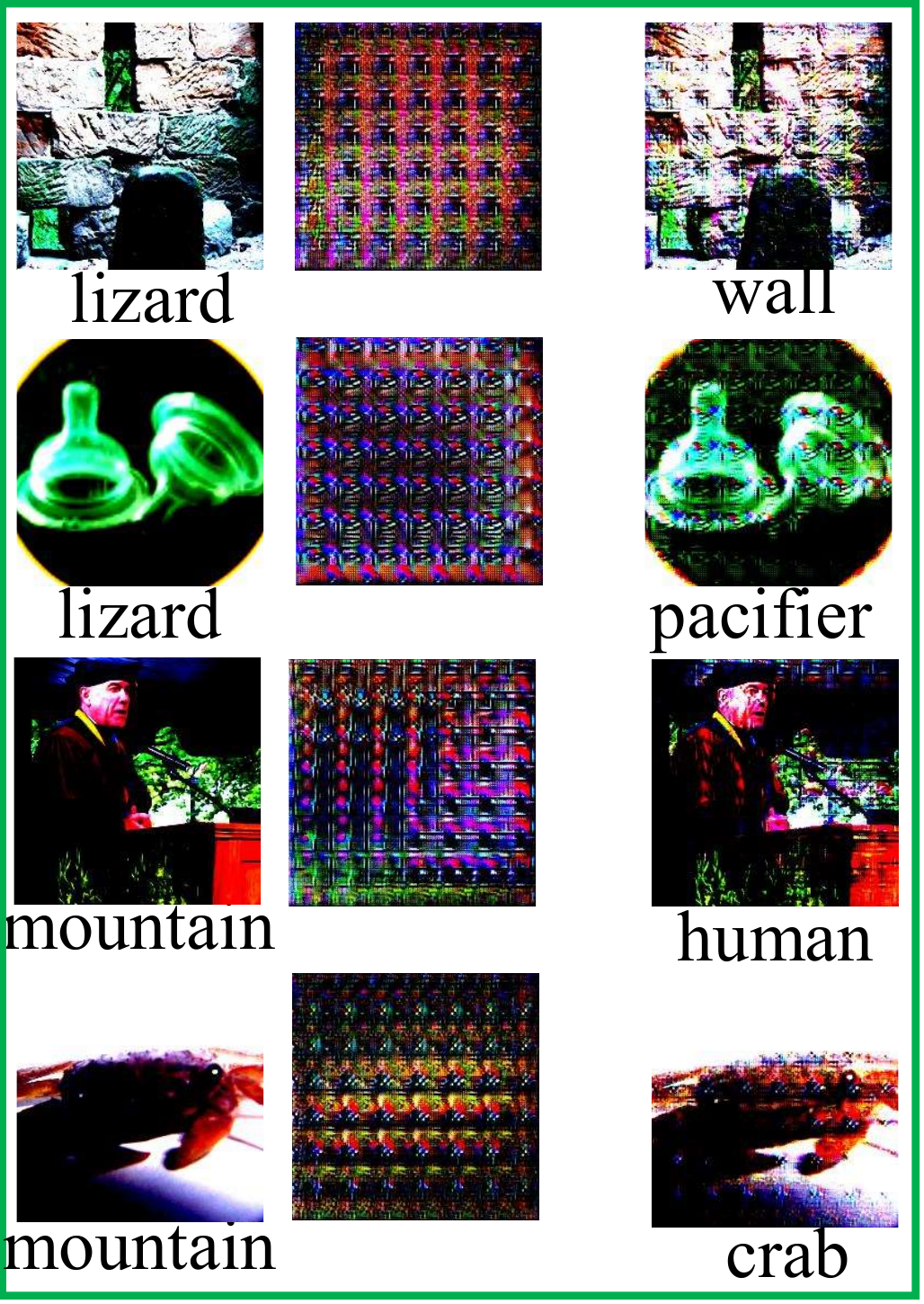}
\end{subfigure}\hfill
\begin{subfigure}[t]{0.23\textwidth}
\includegraphics[width=\linewidth,valign=t]{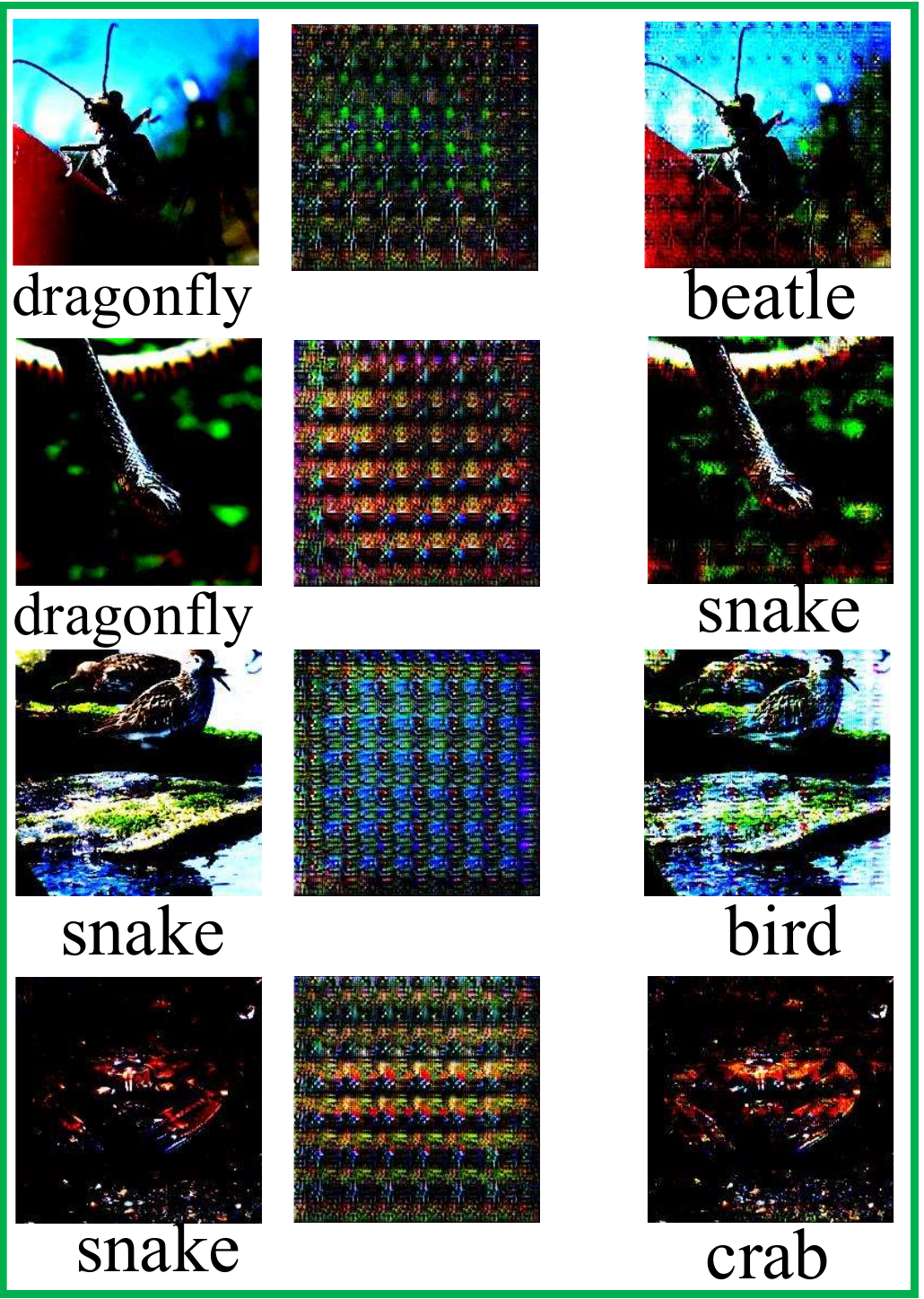}
\end{subfigure}
\vspace{-2mm}
\caption[This is for my LOF]{\nj{Examples} of the \nj{triplets for the enhancing problem:} original image (left), generated perturbation (mid), and perturbed image (right). 
The examples in (A) denote the case when we use vanilla images, and those in (B) show the results from normalized images. 
Intensity of the perturbation in (A) is ten times amplified \nj{for visualization}. (Best viewed in color)}
\label{fig:exp_main}
\vspace{-2mm}
\end{figure*}



\subsection{Training}
\label{sec:prop_training}
By using the defined variables and loss terms, we can train the proposed network in \nj{a similar way} to \nj{the} adversarial min-max training as introduced in \cite{goodfellow2014generative}.
The pseudo-code in Algorithm~\ref{algo:train_pgn} \nj{describes} the detailed training scheme of the proposed algorithm.
For each iteration, we first generate perturbation vector $M_i$ and check if the perturbation satisfies the desired goal in the form of target vector $g_i$.
After setting $g_i$, same as usual advesarial frameworks, we train the discriminator network with given peturbation.
Then, we re-calculate the output value $o_i$ of the discriminator using the updated discriminator and update the generator to deceive the discriminator.

\nj{In our implementation, the} network parameters are updated by Adam optimizer~\cite{kingma2014adam}.
We used fixed learning rate $lr = 0.0001$ to \nj{train both} the generator and the discriminator, and stopped the iteration at $20$ epoch, in practice. 
The additive parameter $\lambda$ is set to $1$ for the entire experiments.

\subsection{Implementation Detail}
\label{sec:network_Setting}
The proposed algorithm is ideally independent \nj{of} a \nj{classifier}. 
However, the Nash-equilibrium~\cite{goodfellow2014generative,mao2016least} for general adversarial framework is difficult to find, and hence an efficient design of initial condition is required.
Thus, For the enhancing problem, we initialize the discriminator network with the weight parameters of target classifier.
In this case, we share the classification network \nj{with} the discriminator, and only a fully connected layer with sigmoid activation \nj{is} additionally trained.
For \nj{the} adversarial problem, both cases \nj{of known and unknown classifier structure are considered.}
In unknown case, we apply ResNet101 network for the discriminator.
For the known case, the same network as target classifier is used.
In both known and unknown cases, ImageNet pre-trained parameters are used to initialize the discriminator, and 
the final fully connected layers are trained.

The generator network of the proposed algorithm consists of \knj{an} encoder and \knj{a} decoder.
We applied the Imagenet pre-trained \knj{ResNet} for the encoder with \knj{a} layer size \knj{of} 50 or 101, respectively.
To define the decoder, we use four deconvolution blocks each \knj{consisting} of three deconvolution layers with filter size $4\times4$, $3\times3$, and $3\times3$, and one final deconvolution layer with filter size $4\times4$ and stride $2$. For each deconvolution block, \knj{the} stride of the first deconvolution layer is set to $2$, and those of the last two deconvolution layer are set to $1$. The \knj{numbers of channels for the total $13$ deconvolution layers} are set to $2048,$$ 1024, 512, 512, 256, 256, 256, 128, 128, 128, 64, 64$, and $3$. 
\section{Experiments}
\label{sec:exp}
Now, we validate the performance of the proposed algorithm \nj{for} the two presented \nj{problems}: \nj{the} enhancing problem and \nj{the} adversarial problem.
Since this is the first attempt to solve \nj{the} enhancing problem, we analyze the proposed network \knj{by} varying parameters and networks.
For the adversarial problem, we compared the performance of the proposed algorithm with those of two representative algorithms that utilize target classifier information, since there \knj{has not been any} algorithms proposed to solve the adversarial problem without knowing the target classifiers.

\subsection{Experimental Settings and Datasets}


In the experiment, we examined recent classifiers such as ResNet~\cite{he2016deep} ($50, 101$), VGG~\cite{simonyan2014very}, and DenseNet ($169$)~\cite{huang2016densely} as target classifiers to be controlled.
For the encoder $E(\cdot)$ of the generator network, we tested two cases, 
\knj{each of which uses}  
ImageNet pre-trained ResNet $101$ (Proposed-101) and ResNet $50$ (Proposed-50) as a base type of the proposed model, respectively.
We also analyzed the effect of regularization loss \nj{$L_r$} by testing different regularization parameter $\gamma$ for \knj{both the} enhancing and \knj{the} adversarial problems. 
For the adversarial problem, the proposed algorithm is also tested with a black-box version `Proposed-B' whose network structure is unknown.
The adversarial performance is compared to the works \knj{of Moosavi \etal} (UAP)~\cite{Moosavi-Dezfooli_2017_CVPR} and \knj{Goodfellow \etal} (EHA)~\cite{goodfellow2014explaining}.
\nj{In all the} experiments, STL-10 dataset~\cite{coates2011analysis} and subsets of ImageNet dataset~\cite{russakovsky2015imagenet} were used.
\nj{To form the subsets of} ImageNet dataset, 10 and 50 classes were randomly \knj{selected}, respectively.
To verify the effect of image normalization, we experimented with the STL-10 dataset without normalization and performed same tests on the ImageNet subset with \knj{normalization applied}.
All the images are scaled to $224\times224$ in the experiments.
For main analysis, we set $\gamma = 0.0001$ for the enhancing problem (Figure~\ref{fig:exp_main} and Table~\ref{table:quant_enhance}) and set $\gamma = 3$ for the adversarial problem (Figure~\ref{fig:exp_main_adv} and Table~\ref{table:quant_adv}).
Target classifiers (Vanilla) were trained with $lr=0.0005$ and $20$ iterations, which sufficient for convergence.

\begin{table*}[t]
\centering
 \caption{Top-$1$ precision and mAP scores for performance enhancing problem: $\gamma$ = 0.0001}
 \label{table:quant_enhance}
  \resizebox{0.99\linewidth}{!}{
\begin{tabular}{|l|l|l|l|l||l|l|}
\hline  
\multirow{2}{*}{Dataset} & \multirow{2}{*}{Classifier} & \multicolumn{3}{|c||}{Encoder from scratch [A]} &\multicolumn{2}{c|}{Imagenet Pre-trained encoder [B]}\\ 
\cline{3-7}
 &  & Vanilla& Proposed-50 & \textbf{Proposed-101}&  Vanilla& \textbf{Proposed-101}\\
\hline
stl-10&ResNet50 & 61\% / 0.588 & 78.1\% / 0.760 & 84.9\% / 0.826 &  92.0\% / 0.883 & 93.6\% / 0.890 \\
	  &ResNet101& 63\% / 0.570 & 84.1\% / 0.785 & 89.6\% / 0.852 & 93.0\% / 0.896 & 94.2\% / 0.907\\
	  &VGG16& 52\% / 0.518 &92.4\% / 0.903 & 91.3\% / 0.866 & 83.4\% / 0.757 & 94.6\% / 0.930 \\
      &DenseNet169& 65\% / 0.564 & 86.5\% / 0.781 & 89.7\% / 0.829 & 95.4 \% 0.884 & 95.9\% / 0.897 \\
\hline 
ImageNet-10&ResNet50 & 78\% / 0.685 & 96.0\% / 0.927 & 91.6\% / 0.871 & 98.0\% / 0.969 & 99.0\% / 0.974 \\
	  	   &ResNet101& 77\% / 0.666 & 93.6\% / 0.898 & 90.4\% / 0.871 & 98.0\% / 0.970 & 98.6\% / 0.972 \\
	       &VGG16& 71\% / 0.613 &95.4\% / 0.886 & 96.2\% / 0.900 & 94.8\% / 0.936 & 96.0\% / 0.927 \\
           &DenseNet169& 77\% / 0.659 &97.0\% / 0.936 & 93.4\% / 0.884 & 98.0\% / 0.970 & 99.0\% / 0.971 \\
\hline
ImageNet-50&ResNet50 & 72\% / 0.649 &91.5\% / 0.883 & 91.3\% / 0.886 & 94.4\% / 0.922 & 95.6\% / 0.928 \\
	  	   &ResNet101& 71\% / 0.635 &89.0\% / 0.856 & 88.1\% / 0.832 & 95.7\% / 0.938 & 96.7\% / 0.949 \\
           &VGG16& 71\% / 0.616 &93.4\% / 0.894 & 94.2\% / 0.902 & 88.5\% / 0.855 & 92.0\% / 0.906\\
           &DenseNet169& 74\% / 0.626 &92.1\% / 0.861 & 93.1\% / 0.875 & 95.5\% / 0.927 & 96.3\% / 0.934\\
\hline
\end{tabular}} 
\end{table*}

\subsection{Enhancing Problem}
\textbf{Main Analysis: }
Figure~\ref{fig:exp_main} (A) shows the examples of the generated perturbation mask for \nj{the} enhancing problem from \knj{STL-10} images without normalization, \knj{\ie pixel values are in between 0 and 1}.
In (A), the original images are misclassified to a \nj{cat} or a horse, respectively.
However, if the proposed perturbation is added to the \nj{misclassified} original image, we can see that the target classifier correctly \nj{classifies} the image.
Figure~\ref{fig:exp_main} (B) presents similar results from the normalized images \knj{of ImageNet dataset, \ie pixel values are normalized to have zero mean and unit variance.}
\knj{In the figure, we can see that originally misclassified examples are correctly classified by adding the corresponding perturbations generated.}
These corrections are remarkable in that the perturbations are small enough that they do not compromise the main characteristics of the original image, and do not resemble the shape of the \nj{correct} target classes.
\begin{figure*}[t]
\centering
\begin{subfigure}[t]{0.03\textwidth}
\textbf{(A)}
\end{subfigure}
\begin{subfigure}[t]{0.21\textwidth}
\includegraphics[width=\linewidth,valign=t]{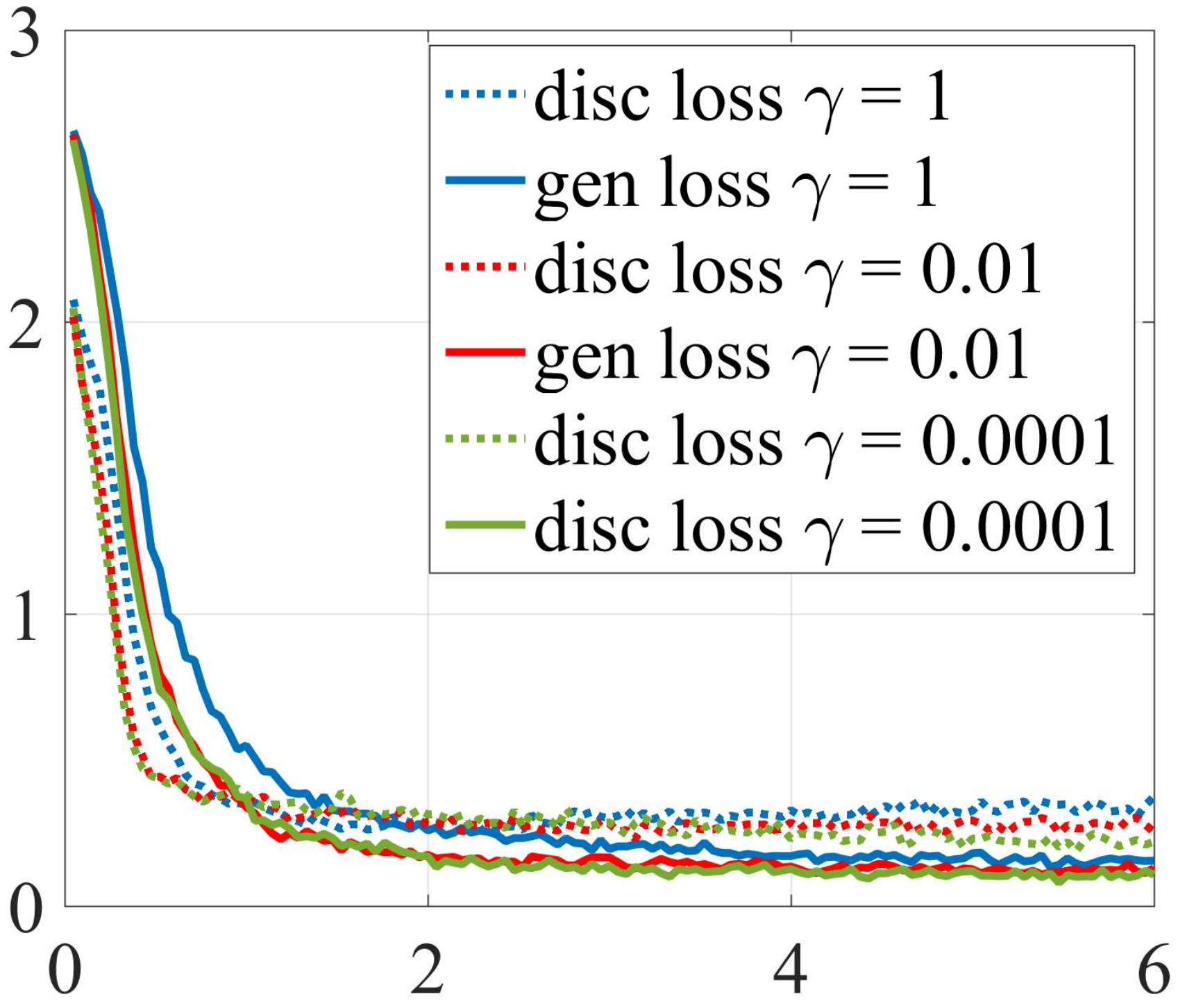}
\end{subfigure}\hfill
\begin{subfigure}[t]{0.03\textwidth}
\textbf{(B)}
\end{subfigure}
\begin{subfigure}[t]{0.21\textwidth}
\includegraphics[width=\linewidth,valign=t]{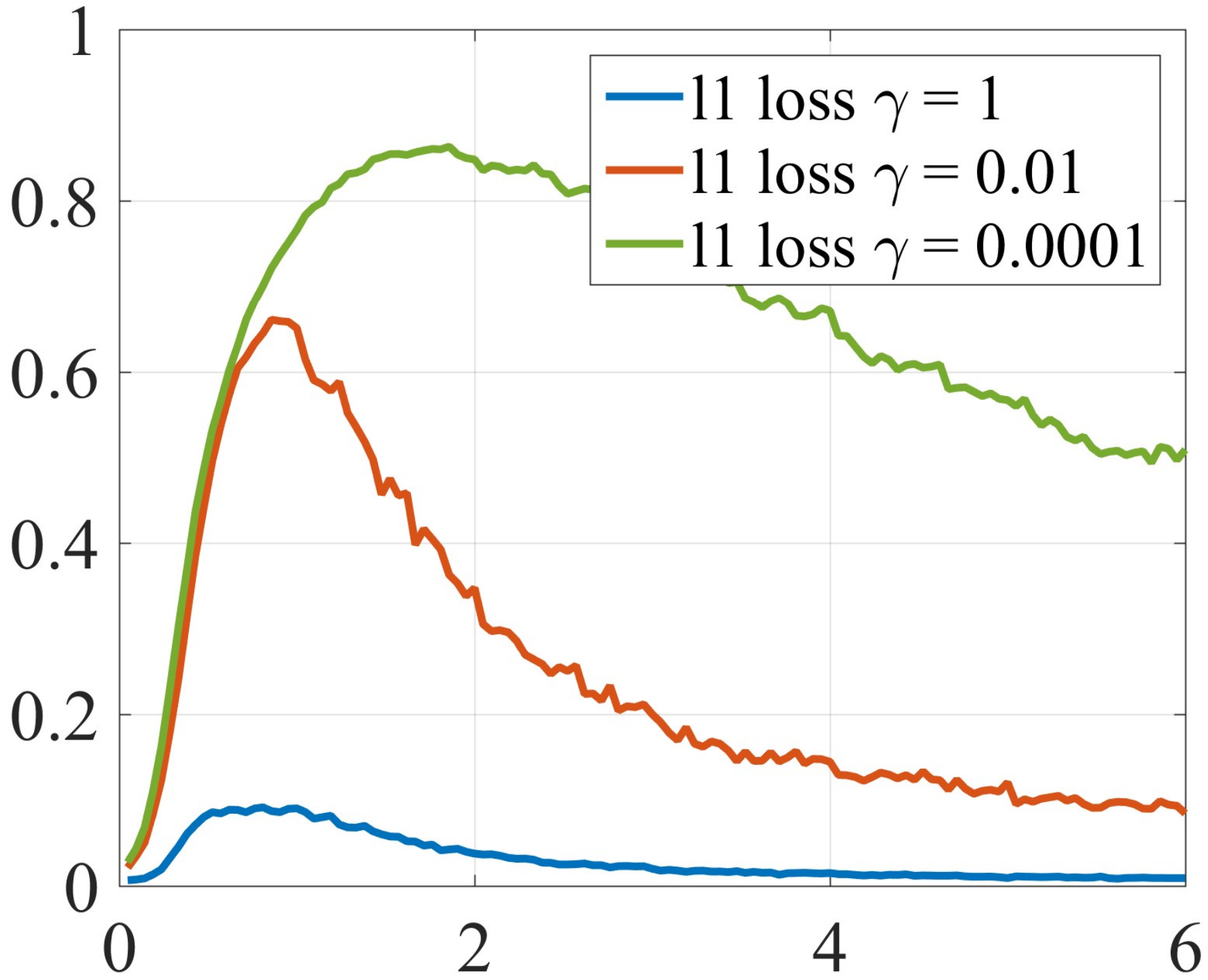}
\end{subfigure}\hfill
\begin{subfigure}[t]{0.03\textwidth}
\textbf{(C)}
\end{subfigure}
\begin{subfigure}[t]{0.21\textwidth}
\includegraphics[width=\linewidth,valign=t]{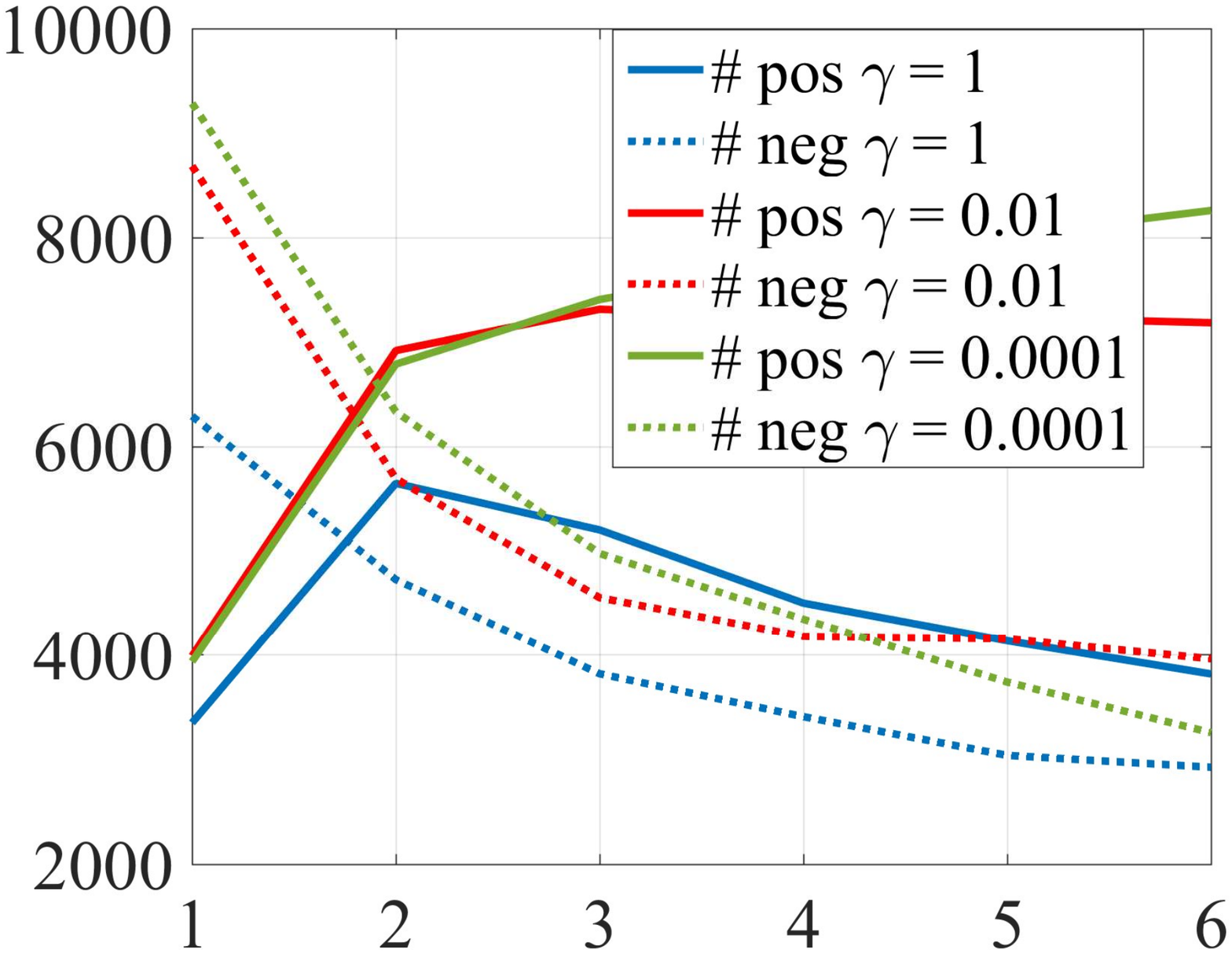}
\end{subfigure}\hfill
\begin{subfigure}[t]{0.03\textwidth}
\textbf{(D)}
\end{subfigure}
\begin{subfigure}[t]{0.21\textwidth}
\includegraphics[width=\linewidth,valign=t]{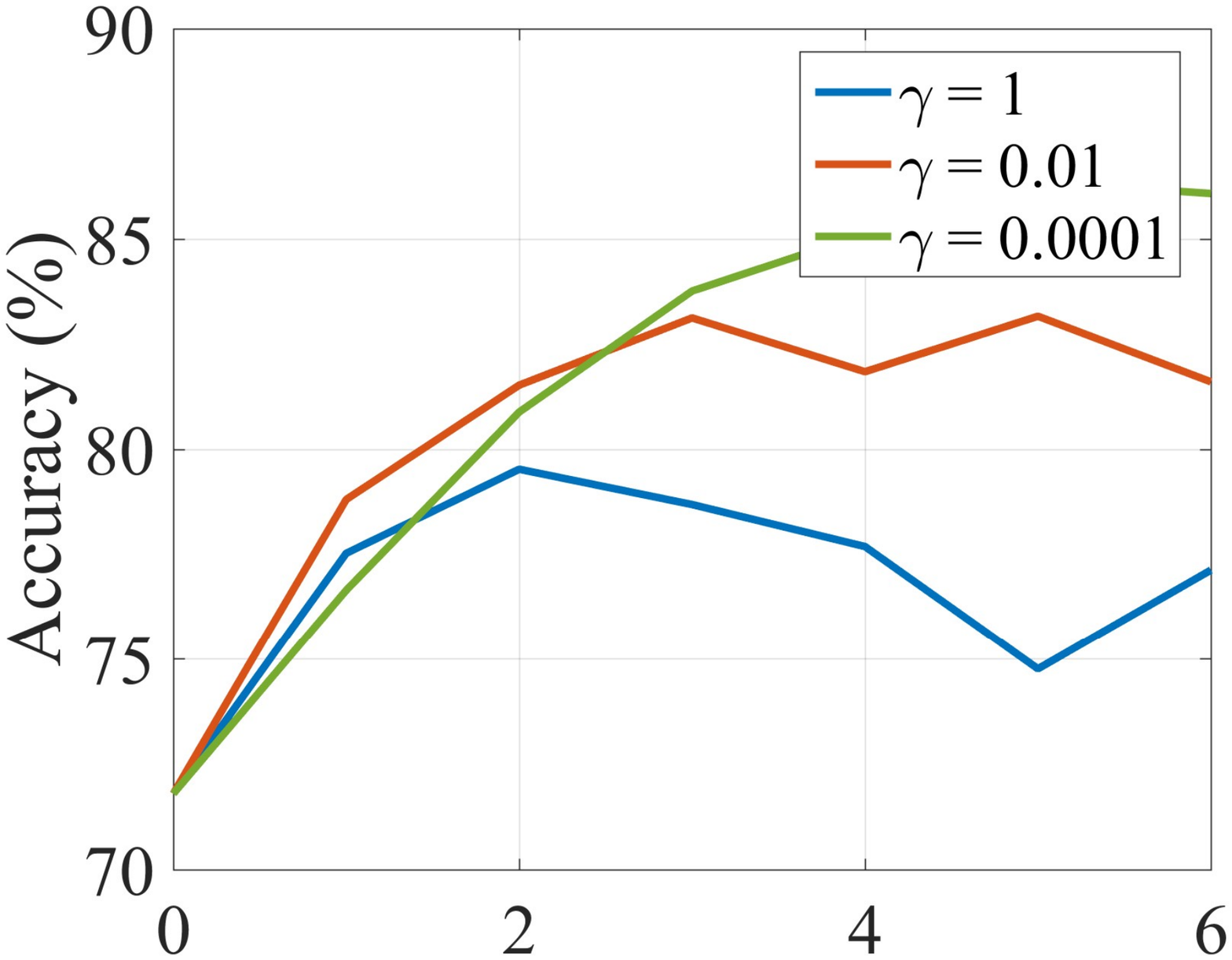}
\end{subfigure}
\vspace{-2mm}
\caption[This is for my LOF]{Graphs describing the convergence and performance enhancement of the proposed algorithm with different $\gamma$:
(A) discriminator loss and generator loss, (B) $\ell$1-loss, (C) positive and negative samples in training set, (D) Accuracy.
Horizontal axis denotes epoch.
The experiments were performed \knj{with the ResNet101 classifier on ImageNet50} using the `Proposed-50'.
}
\vspace{-2mm}
\label{fig:exp_graph_enhance}
\end{figure*}

Table~\ref{table:quant_enhance} presents the quantitative results \nj{showing} the enhanced performance of the proposed algorithm. 
Experiments were conducted on two cases of classifier: (A) classifiers trained from scratch, (B) classifiers trained from ImageNet pre-trained net.
For both cases, we set the target classifiers to evaluation \knj{the} mode which \knj{excludes} the randomness of the classifiers, such as that caused by batch normalization or drop-out, which are usual settings for deep learning testing. 
In Table~\ref{table:quant_enhance}(A), two proposed versions, `Proposed-50' and `Proposed-101', were examined to discover whether the proposed algorithm is affected by the structure of the encoder. 
The result shows that the proposed algorithm can enhance the classification performance of the listed target classifiers for \knj{both versions} in every dataset, and the performance difference is not significant for the two versions.
It \knj{is also worth noting that in many cases,} the classification performance \knj{enhancement} by the proposed perturbation is comparable to the \knj{results of the fine-tuned network initialized with the ImageNet pre-trained parameters}.
In particular, the VGG network achieved \knj{the highest classification performance enhancement by the proposed method.}
This is meaningful in that it shows that our perturbation can compensate for the insufficient information of the classifier.


In Table~\ref{table:quant_enhance}(B), the performance-enhancing results from `Proposed-101' for ImageNet pre-trained classifiers are presented. 
In this case, the classification performance of the vanilla classifiers is obviously higher than that of the scratch-trained version, and hence \nj{it is} more difficult to enhance the performance.
Nevertheless, our algorithm has succeeded in improving performance for all the listed cases.
In particular, we confirmed that the VGG classifier had more performance improvement than the other recent classifiers such as ResNet and DenseNet, and the performance gaps between VGG \nj{and these recent} classifiers were decreased \nj{by adding the perturbation}.
This is meaningful in that it shows the possibility that a relatively simple network like VGG can get better performance. 

\begin{figure*}[t]
\centering
\begin{subfigure}[t]{0.03\textwidth}
\textbf{(A)}
\end{subfigure}
\begin{subfigure}[t]{0.23\textwidth}
\includegraphics[width=\linewidth,valign=t]{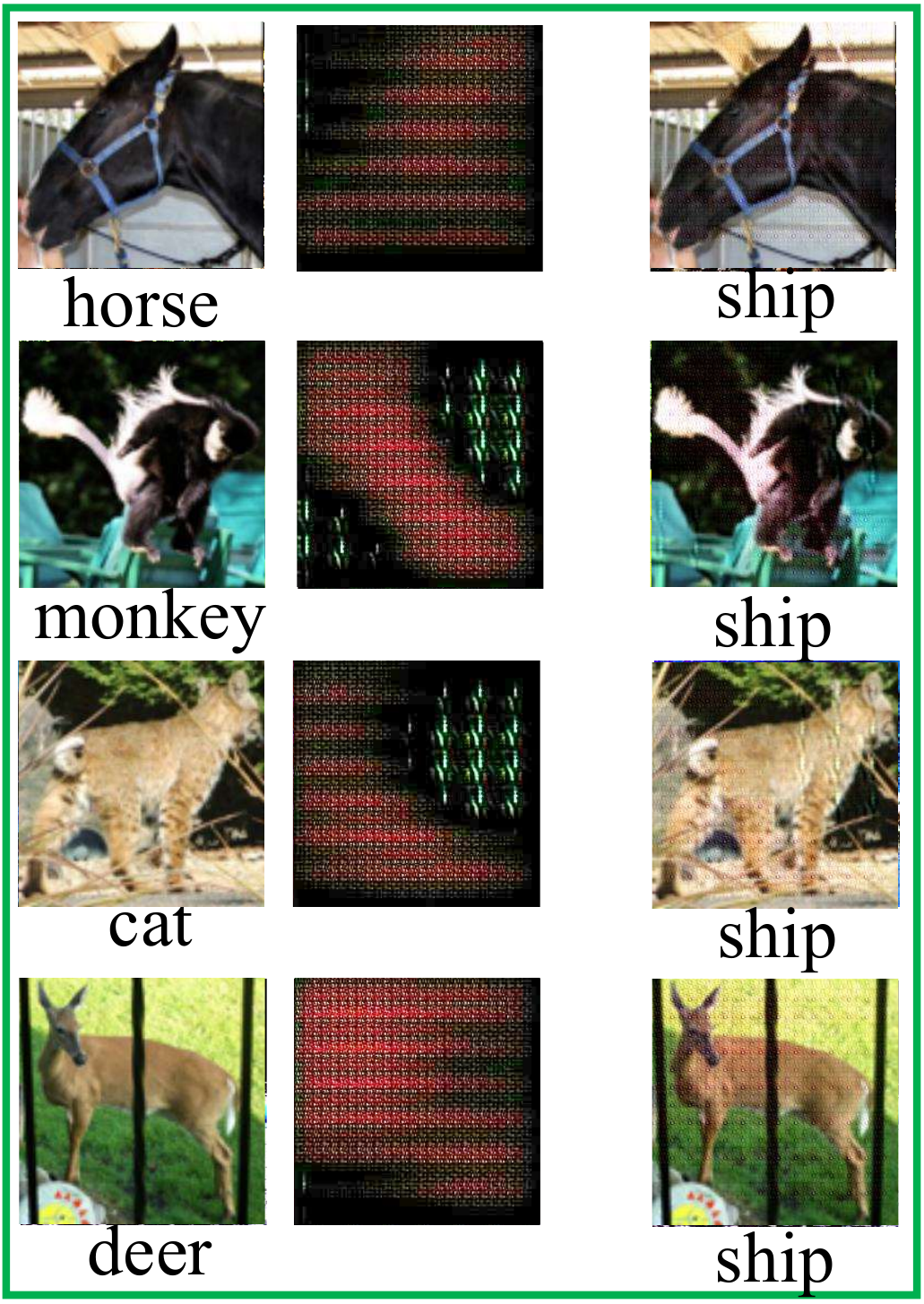}
\end{subfigure}\hfill
\begin{subfigure}[t]{0.23\textwidth}
\includegraphics[width=\linewidth,valign=t]{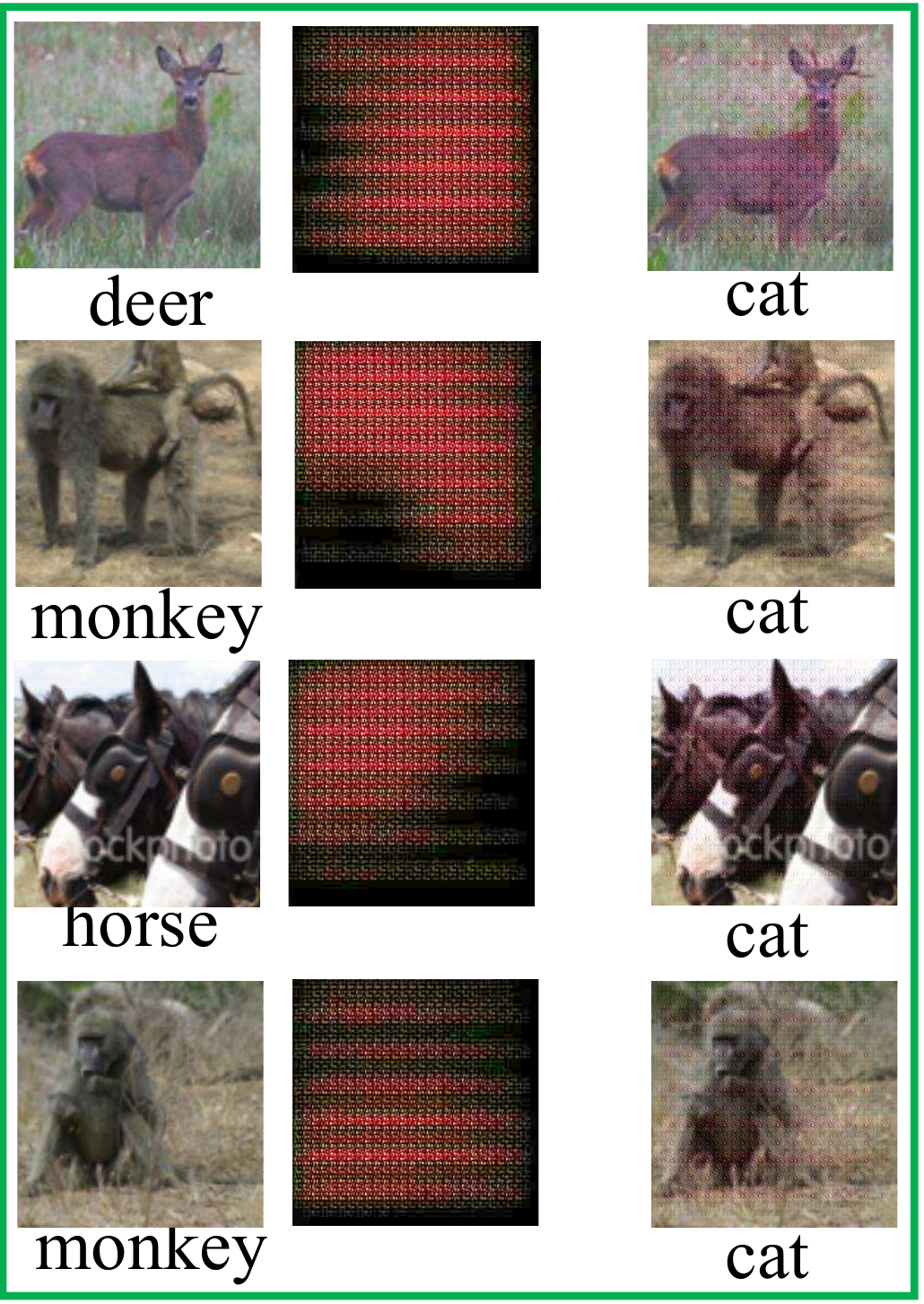}
\end{subfigure}\hfill
\begin{subfigure}[t]{0.03\textwidth}
\textbf{(B)}
\end{subfigure}
\begin{subfigure}[t]{0.23\textwidth}
\includegraphics[width=\linewidth,valign=t]{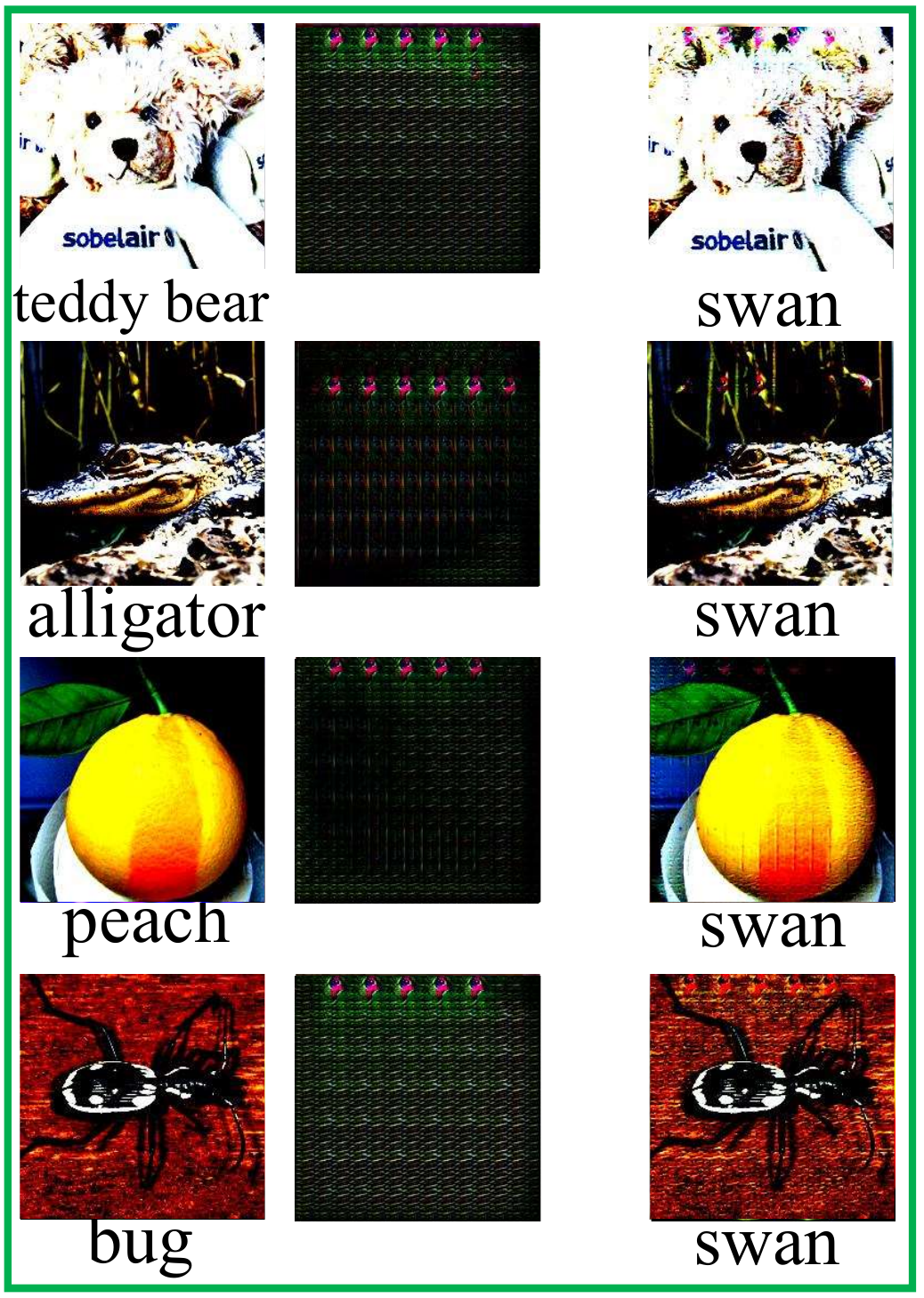}
\end{subfigure}\hfill
\begin{subfigure}[t]{0.23\textwidth}
\includegraphics[width=\linewidth,valign=t]{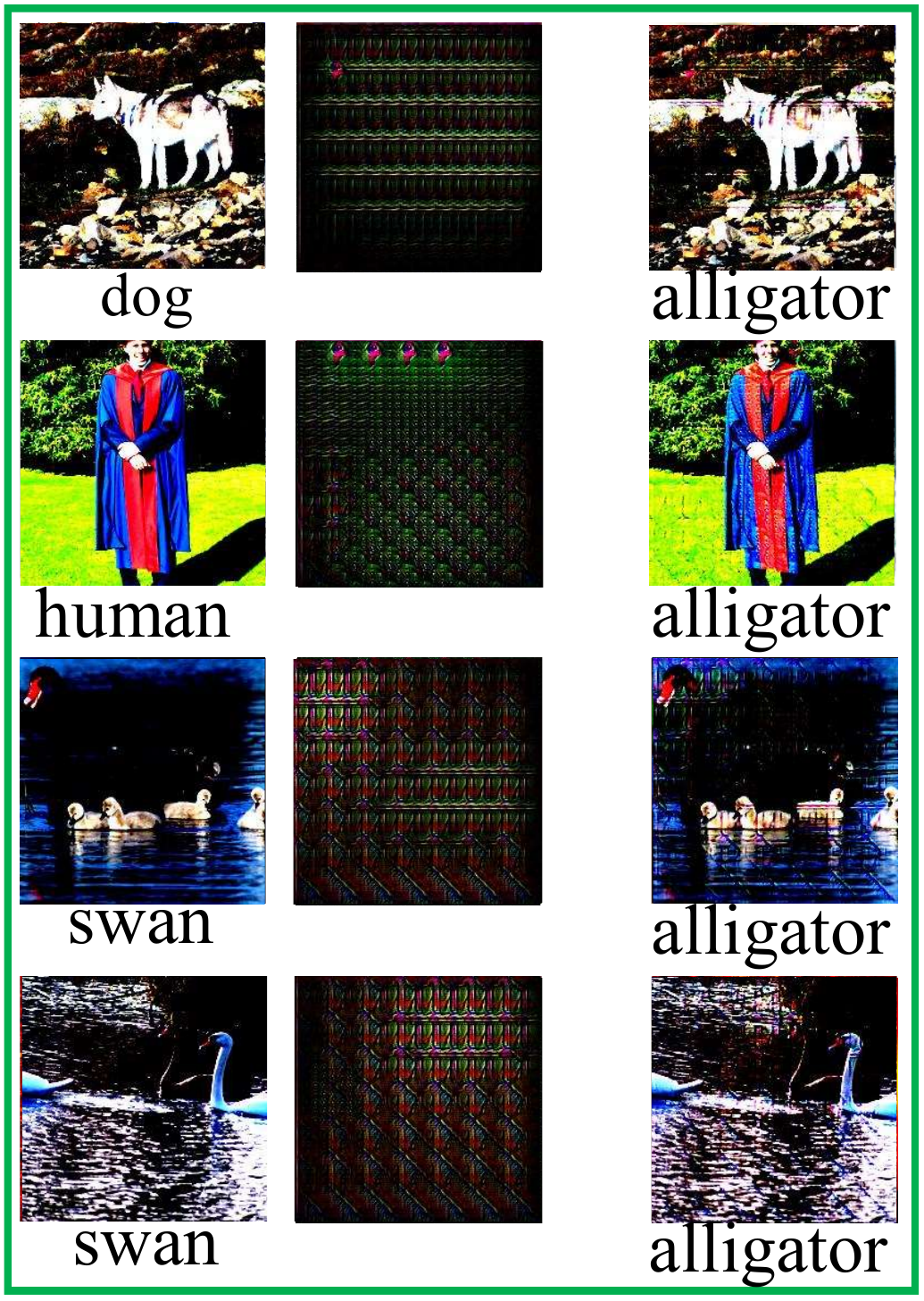}
\end{subfigure}
\caption[This is for my LOF]{\nj{Examples} of the \nj{triplets for the adversarial problem:} original image (left), generated perturbation (mid), and perturbed image (right). 
The examples in (A) denote the case when we use vanilla images, and those in (B) show the results from normalized images. 
Intensity of the perturbation in (A) is ten times amplified \nj{for visualization}. (Best viewed in color) }
\label{fig:exp_main_adv}
\end{figure*}
\begin{figure}[t]
\begin{center}
\includegraphics[width=0.99\linewidth]{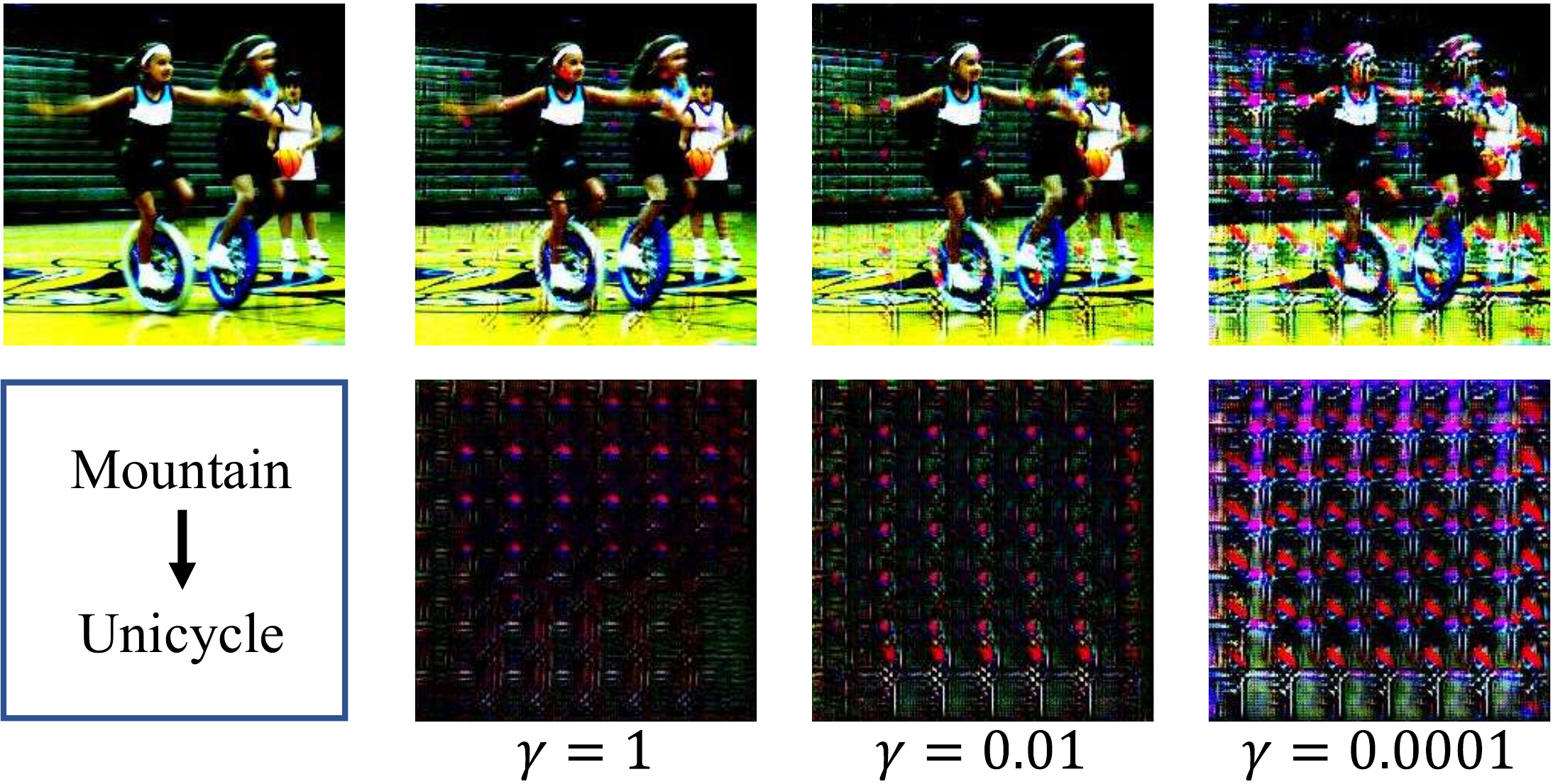}
\end{center}
\vspace{-3mm}
   \caption{
   Example perturbed images from the enhancing problem with different $\gamma = 1,0.01,$ and $0.0001$.
   }
\vspace{-2mm}   
\label{fig:exp_l1_diff}
\end{figure}

\textbf{Convergence and $\bm{\ell1}$-Regularization: }
\knj{The graphs in Figure~\ref{fig:exp_graph_enhance} describe the change of losses as the epoch progresses. 
The 
generator loss $L_g$, the discriminator loss $L_d$, and the $\ell1$-loss $L_r$ are plotted with different $\gamma$'s in equation~(\ref{eq:gen_loss}).
}
As shown in the graph (A), both $L_g$ and $L_d$ converge for every setting of $\gamma$, as desired.
Also, $L_r$ showed a similar convergence tendency for different $\gamma$, as in the graph (B).
However, as $\gamma$ decreases, the time $L_r$ start to decrease delayed, and the maximum value of $L_r$ increased.
From the graphs (C) and (D), We can see that these changes of $L_r$ have a direct impact on the enhancing performance.
The graph (C) describes the number of positive (false to correct) and negative (correct to false) samples in training set for every epoch.
As seen in the graph, The number of positive samples were increased and that of negative samples were decreased until the $\ell1$ intensity of the perturbation fell below a certain level.
We can also see through graph (D) that the accuracy did not rise from the point where the increase of the gap between positive and negative samples is slowed down.
In summary, the amount of possible performance enhancement and $\ell1$-regularization is a trade-off relationship, and the amount can be adjusted depending on how much perturbation is allowed.

The qualitative result of perturbed images for different $\gamma$ is presented in Figure~\ref{fig:exp_l1_diff}.
In the example, we can see that the image changes more to correct the classification result as $\gamma$ becomes smaller.
We confirmed from graph (C) that the performance improvement was about $10\%$ even when $\gamma =1$, and the change in the image would be very small in this case.
In fact, it may be wise to remove the $\ell1$ regularization term if performance gains are only concerned. The analysis is presented in Appendix~\ref{app:convergence}.

\subsection{Adversarial Problem}

\textbf{Main Analysis: }
The example adversarial results from the proposed algorithm are presented in Figure~\ref{fig:exp_main_adv}.
The black box version `Proposed-B' was used \knj{for} the results in the figure, and we can see that the classification results can be changed \knj{by} adding small perturbations for both normalized and \knj{original} images.
It is noteworthy that the \knj{appearances of the perturbations} produced by the normalized \knj{images} and \knj{those of the non-normalized images are} largely different.

Table~\ref{table:quant_adv} shows the performance drop by the adversarial perturbation. 
\knj{In most cases, the} proposed algorithm achieved better adversarial performances \knj{than the conventional} algorithms~\cite{Moosavi-Dezfooli_2017_CVPR,goodfellow2014explaining} \knj{in both cases of known and unknown classifier network.} 
We also confirmed that the proposed algorithm successfully degrades performance, regardless whether the image is normalized.
What is noteworthy is that even if there is no information in the target classification network, the adversarial \knj{performances} were not largely degraded compared to \knj{the case of known network information}.
This is significant in that the proposed algorithm enables more realistic applications than existing algorithms that require a network structure, because the structure of the classifier is usually concealed.
\begin{figure}[t]
\begin{center}
\includegraphics[width=0.99\linewidth]{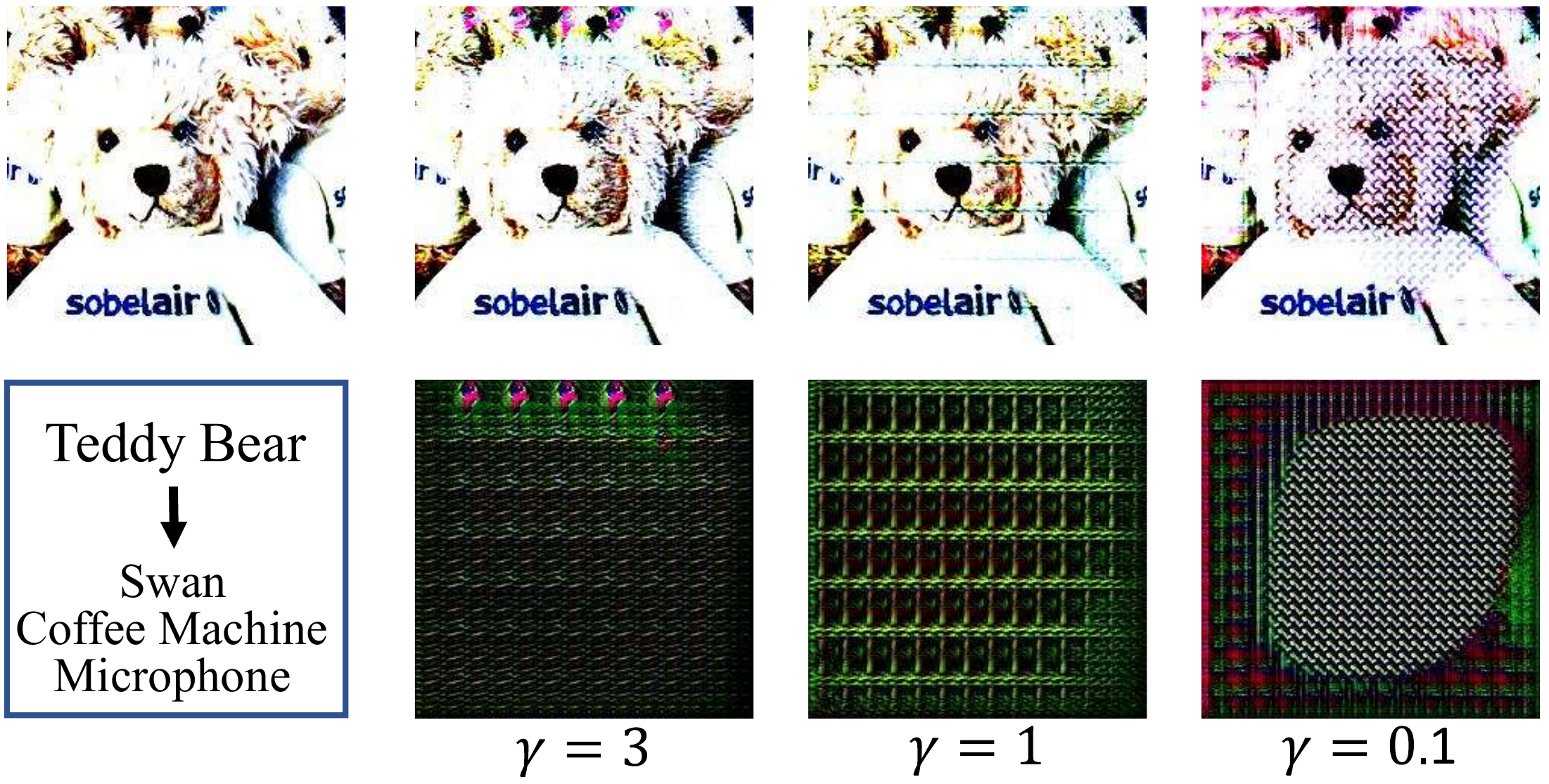}
\end{center}
\vspace{-3mm}
   \caption{
   Example perturbed images from the adversarial problem with different $\gamma = 3, 1,$ and $0.1$.
   }   
\vspace{-2mm}
\label{fig:exp_l1_diff_adv}
\end{figure}

\begin{table*}[t]
\centering
 \caption{Quantitative result for adversarial problem}
 \label{table:quant_adv}
  \resizebox{0.99\linewidth}{!}{
\begin{tabular}{|l|l|l|l|l|l|l|l|}
\hline  
\multicolumn{8}{|c|}{Quantitative result for adversarial problem: $\gamma$ = 3} \\ 
\hline
Dataset& Classifier & Vanilla& Proposed-50 & Proposed-101& proposed-B& UAP~\cite{Moosavi-Dezfooli_2017_CVPR}& EHA~\cite{goodfellow2014explaining}\\
\hline
stl-10&ResNet50 & 92.0\% / 0.883 &5.00\% / 0.071 & 5.15\% / 0.052 & 6.92\% / 0.073& 24.0\% / 0.176 & 20.4\% / 0.141 
\\
	  &ResNet101& 93.0\% / 0.896 &6.60\% / 0.081 & 5.32\% / 0.056 & 9.30\% / 0.094 & 22.3\% / 0.175 & 31.2\% / 0.200  \\
	  &VGG16& 83.4\% / 0.757 &7.00\% / 0.043 & 1.00\% / 0.028 & 7.49\% / 0.090 & 9.9\% / 0.099 & 77.7\% / 0.589  \\
      &DenseNet169& 95.4 \% 0.884 & 14.0\% / 0.136& 9.21\% / 0.104 & 2.80\% / 0.038 & 22.9\% / 0.169 & 19.4\% / 0.145 
      \\
\hline 
ImageNet-10&ResNet50 & 98.0\% / 0.969 & 6.10\% / 0.071 & 9.80\% / 0.110 & 12.4\% / 0.137 & 53.8\% / 0.319 & 9.60\% / 0.093 \\
	  	   &ResNet101& 98.0\% / 0.970 &6.00 \% / 0.078 & 7.00\% / 0.086 & 7.80\% / 0.086 & 31.2\% / 0.212 & 16.6\% / 0.137 \\
	       &VGG16& 94.8\% / 0.936 &4.00\% / 0.064 & 1.00\% / 0.032 & 5.40\% / 0.049 & 11.0\% / 0.115 & 73.4\% / 0.311 \\
           &DenseNet169& 98.0\% / 0.970  &5.00\% / 0.053 & 5.80\% / 0.063 & 3.00\% / 0.039 & 10.0\% / 0.113 & 14.8\% / 0.146 \\
\hline
ImageNet-50&ResNet50 & 94.4\% / 0.922 &3.72\% / 0.043 & 5.44\% / 0.052 & 10.8\% / 0.101 & 25.9\% / 0.162 & 14.2\% / 0.069  \\
	  	   &ResNet101& 95.7\% / 0.938 &12.0\% / 0.110 & 10.7\% / 0.092 & 11.6\% / 0.101 & 14.4\% /0.074 & 19.0\% / 0.088 \\
	       &VGG16& 88.5\% / 0.855 &2.00\% / 0.029 & 2.00\% / 0.022 & 3.20\% / 0.027  & 2.23\% / 0.021 & 56.4\% / 0.126 \\
           &DenseNet169& 95.5\% / 0.927 & 7.04\% / 0.062 & 8.20\% / 0.074 & 6.56\% / 0.063 & 39.7\% / 0.272 & 10.8\% / 0.070 \\
\hline

\hline
\end{tabular}} 
\end{table*}
\begin{figure*}[t]
\centering
\begin{subfigure}[t]{0.03\textwidth}
\textbf{(A)}
\end{subfigure}
\begin{subfigure}[t]{0.21\textwidth}
\includegraphics[width=\linewidth,valign=t]{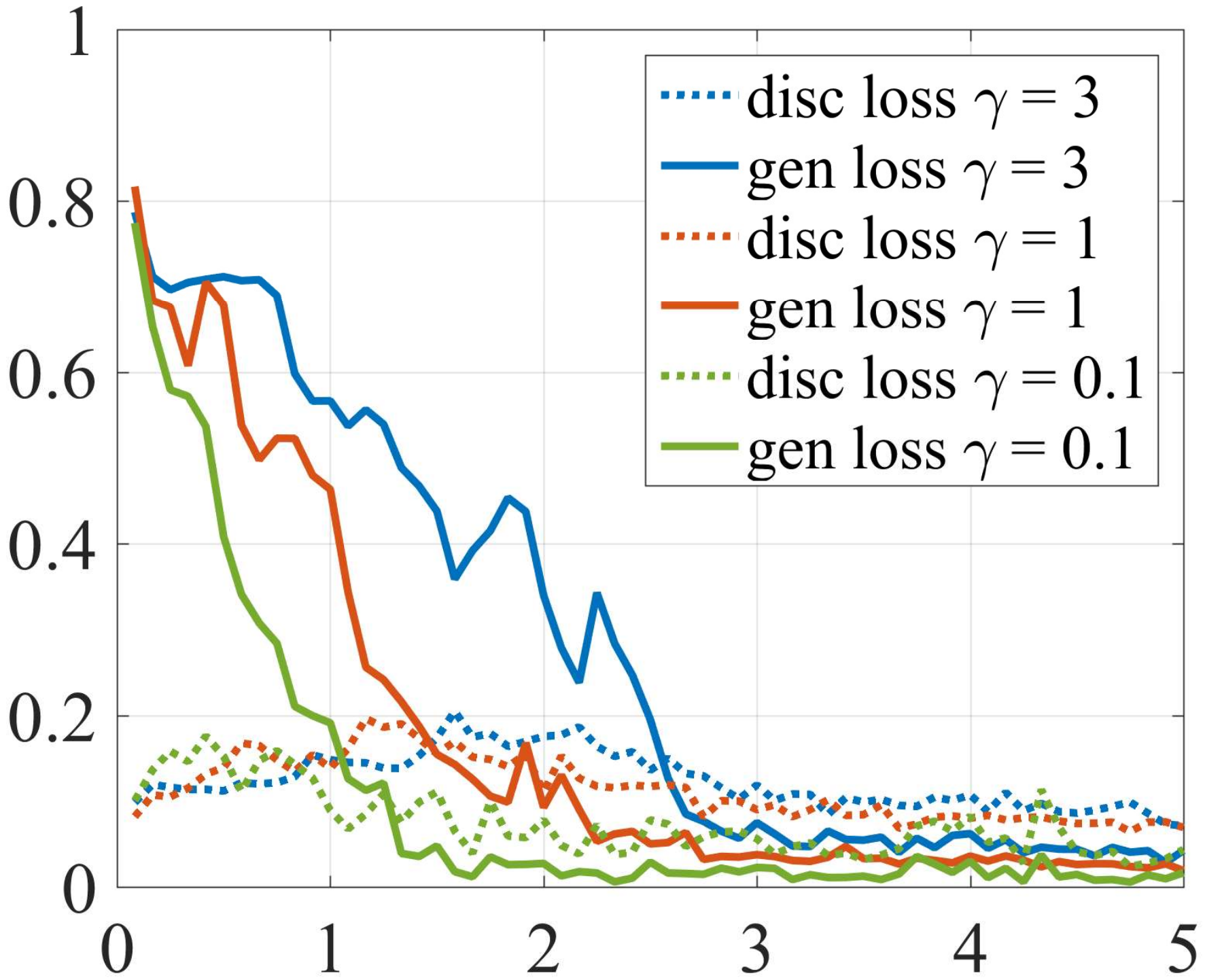}
\end{subfigure}\hfill
\begin{subfigure}[t]{0.03\textwidth}
\textbf{(B)}
\end{subfigure}
\begin{subfigure}[t]{0.21\textwidth}
\includegraphics[width=\linewidth,valign=t]{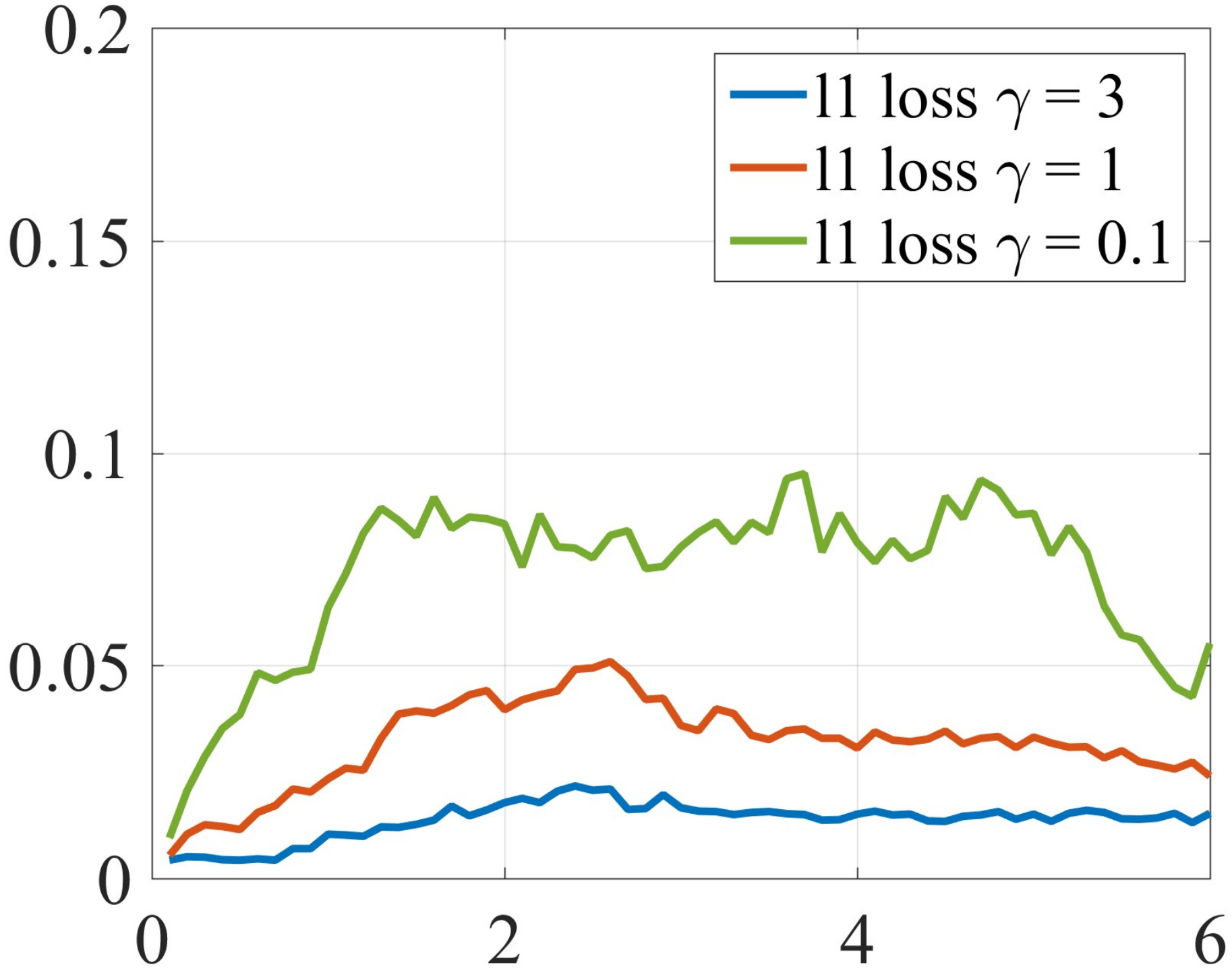}
\end{subfigure}\hfill
\begin{subfigure}[t]{0.03\textwidth}
\textbf{(C)}
\end{subfigure}
\begin{subfigure}[t]{0.21\textwidth}
\includegraphics[width=\linewidth,valign=t]{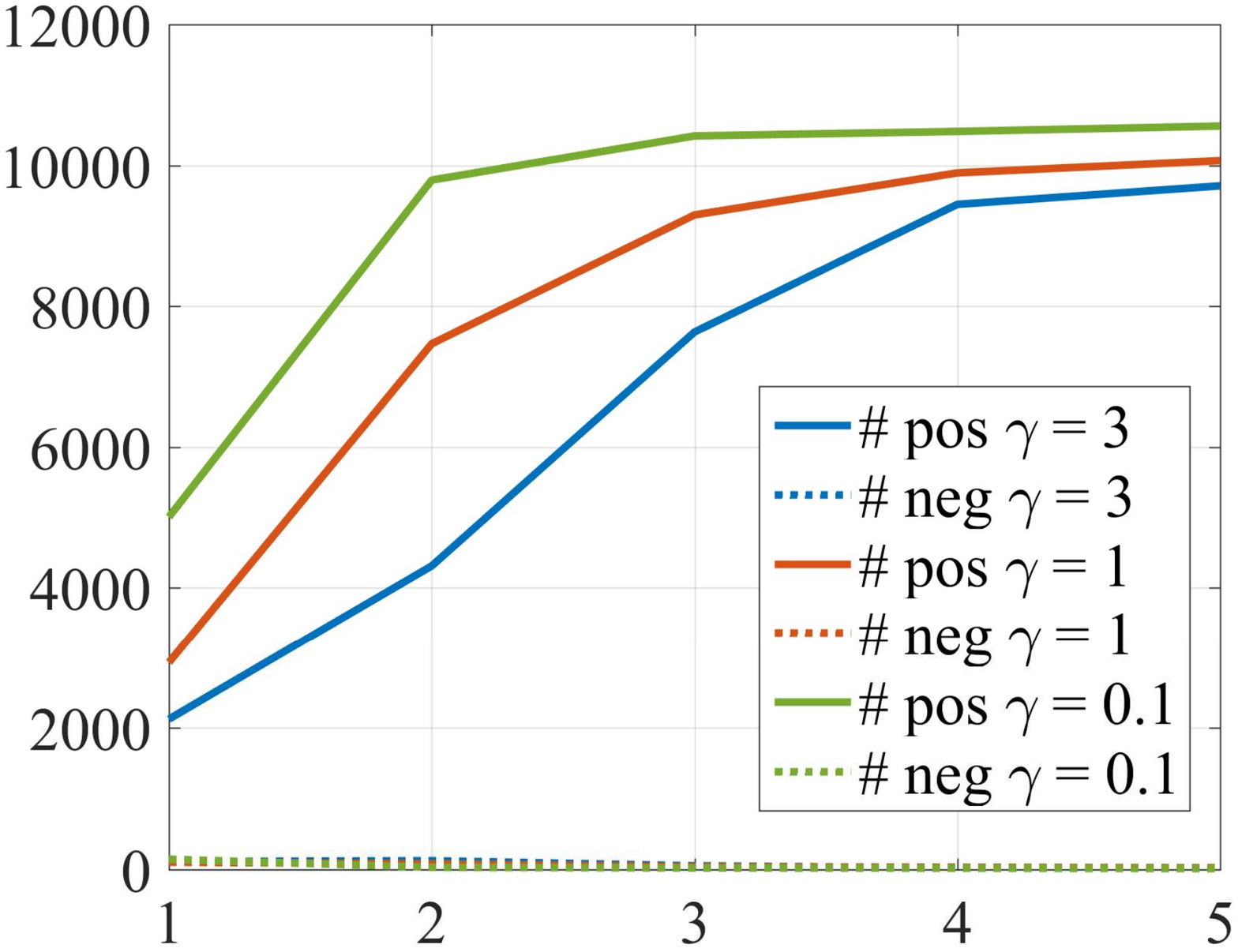}
\end{subfigure}\hfill
\begin{subfigure}[t]{0.03\textwidth}
\textbf{(D)}
\end{subfigure}
\begin{subfigure}[t]{0.21\textwidth}
\includegraphics[width=\linewidth,valign=t]{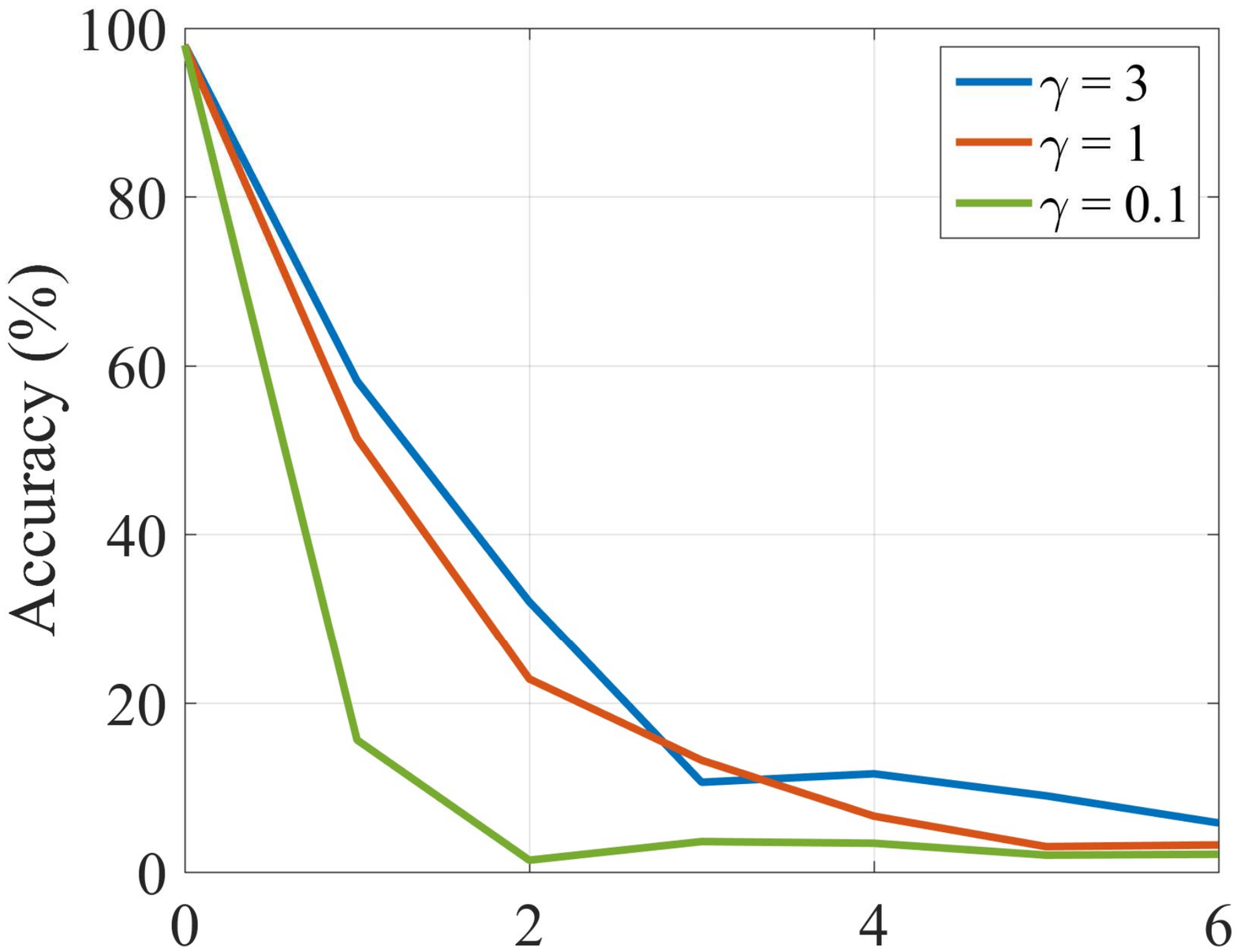}
\end{subfigure}
\vspace{-2mm}
\caption[This is for my LOF]{Graphs describing the convergence and degradation performance of the proposed algorithm for different $\gamma$:
(A) discriminator loss and generator loss, (B) $\ell$1-loss, (C) positive and negative samples in training set, (D) Accuracy.
Horizontal axis denotes epoch.
The experiments were performed on DenseNet169 and ImageNet10 using the `Proposed-B' version.
}
\vspace{-2mm}
\label{fig:exp_graph_adv}
\end{figure*}

\textbf{Convergence and $\bm{\ell1}$-Regularization: }
Unlike with the enhancing problem, we can not expect the proposed adversarial structure to suppress the intensity of the perturbation. 
In the case of the enhancing problem, too much perturbation may be disadvantageous because correct image classification results should be preserved. 
\yj{In fact, $\ell1$ intensity of the perturbation does not increase more than a certain level, even if there is no $\ell1$-regularization term in the enhancing problem case. (See Appendix~\ref{app:further_exp} for more explanation).}
Conversely, in the case of the adversarial problem, it can be predicted that a larger perturbation may more easily ruin the classification result.
Therefore, we can conjecture that the role of $\ell1$-regularization is very important for controlling the intensity of the perturbation in the adversarial problem.

The graphs in Figure~\ref{fig:exp_graph_adv} describe the changes of the loss terms in the proposed framework in the form shown in Figure~\ref{fig:exp_graph_enhance}.
The one main difference to the enhancing problem is that the discriminator loss $L_d$ converges much faster than the enhancing problem case, as seen in graph (A).
This is because, in the adversarial problem case, much more positive samples (correct to false, in this case) compared to negative samples (false to correct) can be obtained from the beginning than the enhancing problem case as shown in graph (C).
This is natural in that the number of correctly classified samples in the training set is much larger than the false sample.
Also, as the gamma decreases, we can see that the generator loss $L_g$ fluctuates drastically, which means that the $\ell1$-loss $L_r$ and $L_g$ are in a competitive relationship. 
The other notable difference to the enhancing problem is that the tendency of $L_r$ decrease is not clear compared to the enhancing problem case.
Rather, we can see that $L_r$ converges on different values according to the value of $\gamma$, as in graph (B).
From the graph (D), we can see that the classification accuracy falls more fast when $L_r$ is kept at larger value, as we expected.

Figure~\ref{fig:exp_l1_diff_adv} shows the amount of perturbation for different value of $\gamma$.
We can see that the intensity of the perturbation get larger when $\gamma$ increases, and the deformation of the image is increased accordingly.
Therefore, we performed the adversarial task for $\gamma = 3$, and we obtained satisfactory performance without compromising the \knj{quality of the perturbed image} significantly.
\section{Conclusion}
In this paper, we have proposed a novel adversarial framework for generating \knj{an} additive perturbation vector that can control the performance of the classifier either in positive or negative directions without changing the network parameters of the classifier.
Through the qualitative and quantitative analysis, we have confirmed that the proposed algorithm can enhance classification performance \knj{significantly by} just adding a small perturbation, marking the first attempt in this field.
Furthermore, we have confirmed that the proposed method can be directly applied to generate an adversarial perturbation vector that \knj{degrades} classification performance, even when the framework \knj{of the} target classifier is concealed, which is another first \sd{attempt}.
These results show that the \knj{parameters of the existing CNNs are} not ideally estimated, we have made the meaningful progress toward influencing the network's outcome in desired direction from outside. 
{\small
\bibliographystyle{ieee}
\bibliography{egbib}
}

\clearpage
\onecolumn
\begin{appendices}
\section{Convergence}
\label{app:convergence}

In this section, we prove that our model has \nj{a} global optima at $p_g(J) = 1$, and converges to the global optima, theoretically.
Also, we show that the same conditions hold when we \nj{replace} the equations (\ref{eq:dsc_loss}, \ref{eq:fool_loss}) in the paper \nj{with the} cross-entropy loss, 
which means that the discriminator is defined as a binary logistic regressor.

\subsection{Proposed Case \nj{(Least Squared Problem)} }
\label{sec:proposed}
In this case, we update the generator and the discriminator using the following equations:
\begin{eqnarray} 
&L_{d}(D,G) =& \frac{1}{2}\mathbf{E}_{p_g}[{(D(J)-1)}^2]+\frac{1}{2}\mathbf{E}_{p_{\bar{g}}}[{(D(J))}^2] \label{eq:dsc_loss}, \\
&L_{g}(D,G) =& \frac{1}{2}\mathbf{E}_{p_g}[{(D(J)-1)}^2]+\frac{1}{2}\mathbf{E}_{p_{\bar{g}}}[{(D(J)-1)}^2] \label{eq:fool_loss}.
\end{eqnarray}
where the distributions $p_g$ and $p_{\bar{g}}$ denote $p(g = 1) = p(r=l)$ and $p(g = 0) = p(r \neq l)$, respectively.
Note that both $r$ and $g$ \nj{depend} on the generated sample \nj{$J$}.
Since \nj{$p_g$} is \nj{a} function of $J$ and hence the function of $g$, we can write $L_d(D,G)$ and $L_g(D,G)$ as $L_d(D,p_g(J)), L_g(D,p_g(J))$, \nj{respectively}.
Therefore, our objective functions become
\begin{eqnarray} 
\min_{D} &L_{d}(D,p_g(J)) \label{eq:dsc_obj}, \\
\min_{p_g(J)} &L_{g}(D,p_g(J)) \label{eq:fool_obj}.
\end{eqnarray}
\begin{proposition}
\label{prop:1_1}
For fixed $G$, \nj{an} optimal discriminator is 
\begin{equation}
D^{*}_{G}(J) = p_g(J).
\end{equation}
\begin{proof}
The training criterion for the discriminator $D$ given $G$ (in this case, $p_g(J)$) is conducted by minimizing the quantity $L_d(D,G)$
\begin{equation}
\label{eq:prop1_ld}
\begin{split}
L_d(D,p_g(J)) &= \int{p_g(J)(D(J)-1)^2+p_{\bar{g}}(J)(D(J))^2}dJ\\
&= \int{p_g(J)(D(J)-1)^2+(1-p_{g}(J))(D(J))^2}dJ\\
&= \int{(D(J))^2-2p_g(J)D(J)+p_g(J)}dJ.
\end{split}
\end{equation}
$\Rightarrow$: 
The term $L_d(D,p_g(J))$ has a local extremum at the point $D^{*}_G(J)$, where
\begin{equation}
\begin{split}
\nabla_D L_d(D^*_G,p_g(J)) &= \int{2D^*_G(J)-2p_g(J)}dJ\\
&= 0.
\end{split}
\end{equation}
%
Therefore, \nj{a sufficient condition for an} optimal $D^{*}_G(J)$ becomes
\begin{equation}
D^{*}_G(J) = p_g(J).
\end{equation}
\\
\noindent$\Leftarrow$:
From equation~(\ref{eq:prop1_ld}), The discriminator loss function is converted as
\begin{equation}
\label{eq:prop1_ld}
\begin{split}
L_d(D,p_g(J)) &= \int{p_g(J)(D(J)-1)^2+p_{\bar{g}}(J)(D(J))^2}dJ\\
&= \int{(D(J))^2-2p_g(J)D(J)+p_g(J)}dJ.\\
&= \int{[D(J)-p_g(J)]^2}dJ + 1 - \int{(p_g(J))^2}dJ~~~~\because \int{p_g(J)}dJ = 1,\\
&= \int{[D(J)-p_g(J)]^2}dJ + C.
\end{split}
\end{equation}
Therefore, the term $L_d(D,p_g(J))$ achieves its minimum over $D$ necessarily at the point $D^{*}_G(J)$ where
\begin{equation}
D^{*}_G(J) = p_g(J).
\end{equation}
\end{proof}
\end{proposition}
\begin{proposition}
\label{prop:1_2}
For \nj{the} optimal $D^{*}_G(J) = p_g(J)$, the optimal generator is achieved at
\begin{equation}
p^{*}_g(J) = 1.
\end{equation}
\begin{proof}
The training criterion for the generator $G$, hence $p^{*}_g(J)$ is obtained by minimizing the quantity $L_g(D^*_G,p_g(J))$
\begin{equation}
\begin{split}
L_g(D^*_G,p_g(J)) &= \int{p_g(J)(D^*_G(J)-1)^2+p_{\bar{g}}(J)(D^*_G(J)-1)^2}dJ\\
&= \int{p_g(J)(D^*_G(J)-1)^2+(1-p_{g}(J))(D^*_G(J)-1)^2}dJ\\
&= \int{(D^*_G(J)-1)^2}dJ\\
&= \int{(p_{g}(J)-1)^2}dJ \quad \because D^*_G(J) = p_{g}(J).
\end{split}
\end{equation}
The term $L_g(D^*_G,p_g(J))$ is monotonically decreasing function in $p_g(J)\in [0,1]$, as
\begin{equation}
\begin{split}
\nabla_{p_g} L_g(D^*_G,p_g(J)) &= \int{2(p_g(J)-1)}dJ\\
&\leq 0 \quad \because 0\leq p_g(J)\leq 1,\forall J.
\end{split}
\end{equation}
Here, we wrote $\nabla_{p_g(J)}$ as $\nabla_{p_g}$ for simplicity.
Therefore, the optimal $p^*_g(J)$ \nj{that minimizes $L_g(D_G^*, p_g(J))$} becomes
\begin{equation}
p^*_g(J) = 1.
\end{equation}
\end{proof}
\end{proposition}
\begin{theorem}
\label{sec:theorem1}
If G and D have enough capacity, and at each iterative step minimizing the equations~(\ref{eq:dsc_loss}) \nj{and} (\ref{eq:fool_loss}), $D$ is allowed to reach the optimal given $G$, then the optimal distribution $p^{*}_g(J)$ converges to $1$.
\end{theorem}
\begin{proof}
According to the proposition~\ref{prop:1_2}, 
\nj{the} supremum of $U(p_g) = \sup_D L_g(D,p_g(J))$ is convex.
The \nj{subdifferential}
of $U(p_g)$ includes the derivative of the function $L_g(D,p_g(J))$ at the point \nj{$D_G^*$} where the maximum \nj{of $L_g$} 
is attained.
\nj{Applying a subgradient method to $U(p_g)$ is} equivalent to conducting the gradient descent update \nj{of $p_g(J)$ for the function $L_g(D_G^*, \cdot)$.}
Therefore, $p_g(J)$ converges to the global minimum $1$ with small iterative update of $p_g$, concluding the proof.
\end{proof}

In practice, the generator function has limited capacity to satisfy the desired conditions, and training dataset also has limited representativeness \nj{for} the test dataset.
However, careful designing of the generator function with multi-layered perceptron and employing sufficient amount of training dataset, our model can achieve reasonable performance in spite of the mentioned difficulties.

\subsection{Logistic Regression Case}
\label{sec:logistic}
In logistic regression case, the generator and discriminator loss are defined as follows:
\begin{eqnarray} 
&L_{d}(D,G) =& \frac{1}{2}\mathbf{E}_{p_g}[{-\log(D(J))}]+\frac{1}{2}\mathbf{E}_{p_{\bar{g}}}[-\log(1-D(J))] \label{eq:dsc_loss_log}, \\
&L_{g}(D,G) =& \frac{1}{2}\mathbf{E}_{p_g}[{-\log(D(J))}]+\frac{1}{2}\mathbf{E}_{p_{\bar{g}}}[-\log(D(J))] \label{eq:fool_loss_log}.
\end{eqnarray}
As similar process to section~\ref{sec:proposed}, we use \nj{the} same objective functions in equation~(\ref{eq:dsc_obj}) and ~(\ref{eq:fool_obj}).
\begin{proposition}
\label{prop:2_1}
For fixed $G$, the optimal discriminator is 
\begin{equation}
D^{*}_{G}(J) = p_g(J).
\end{equation}
\begin{proof}
The training criterion for the discriminator $D$ given $G$ (in this case, $p_g(J)$) is conducted by minimizing the quantity $L_d(D,G)$
\begin{equation}
\begin{split}
L_d(D,p_g(J)) &= \int{p_g(J)\log(D(J))+p_{\bar{g}}(J)\log(1-D(J))}dJ\\
&= \int{p_g(J)\log(D(J))+(1-p_{g}(J))\log(1-D(J))}dJ.
\end{split}
\end{equation}

For \nj{all $(a,b)\in {\mathbb{R}^2}$, the function $x\rightarrow a\log(x)+b\log(1-x)$, $x \in  [0,1]$, gets its maximum value} at $x^* = a/(a+b)$ (see \cite{goodfellow2014generative} for detailed explanation).
Using this property, the optimal point $D^*_G(J)$ becomes
\begin{equation}
\begin{split}
D^*_G(J) &= \frac{p_g(J)}{p_g(J)+(1-p_g(J))}\\
&= p_g(J).
\end{split}
\end{equation}
\end{proof}
\end{proposition}

\begin{proposition}
\label{prop:2_2}
For optimal $D^{*}_G(J)$, the optimal generator is achieved at
\begin{equation}
p^{*}_g(J) = 1.
\end{equation}
\begin{proof}
The training criterion for the generator $G$, hence $p^{*}_g(J)$ is obtained by minimizing the quantity $L_g(D^*_G,p_g(J))$
\begin{equation}
\begin{split}
L_g(D^*_G,p_g(J)) &= -\int{p_g(J)\log(D^*_G(J))+p_{\bar{g}}(J)\log(D^*_G(J))}dJ\\
&= -\int{p_g(J)\log(D^*_G(J))+(1-p_{g}(J))\log(D^*_G(J))}dJ\\
&= -\int{\log(D^*_G(J))}dJ\\
&= -\int{\log(p_g(J))}dJ. \quad \because D^*_G(J) = p_{g}(J).
\end{split}
\end{equation}
$\Rightarrow$: derivative of $L_g(D^*_G,p_g(J))$ over $p_g$ becomes
\begin{equation}
\begin{split}
\nabla_{p_g}L_g(D^*_G,p_g(J)) &= -\nabla_{p_g}\int{\log(p_g(J))}dJ\\
&= -\nabla_{p_g}\int{\log(p_g(J))(\nabla_{J}p_g(J))^{-1}}dp_g\\
&= -\log(p_g(J))(\nabla_{J}p_g(J))^{-1}.\\
\label{eq:derivatives}
\end{split}
\end{equation}
We wrote $\nabla_{p_g(J)}$ as $\nabla_{p_g}$, and $d{p_g(J)}$ as $d{p_g}$ for simplicity.
By (\ref{eq:derivatives}), \nj{$L_g(D_G^*, \cdot)$ has a local extremum point at $p_g^*(J) = 1$}, and \nj{the} corresponding \nj{value}
$L_g(D^*_G,p_g(J)) = 0$ at that point, which is the global minimum value.
\\
\\
\noindent$\Leftarrow$: Assume that $p_g(J) = 1$ for $\forall J$. In this case, $L_g(D^*_G,p_g(J)) = 0$ and it is the global minimum of the quantity $L_g(D^*_G,p_g(J))$.
\\

Therefore, the optimal $p^*_g(J)$ becomes
\begin{equation}
p^*_g(J) = 1.
\end{equation}
\end{proof}
\end{proposition}

\begin{theorem}
\label{sec:theorem2}
If G and D have enough capacity, and at each iterative step minimizing the equations~(\ref{eq:dsc_loss_log})(\ref{eq:fool_loss_log}), $D$ is allowed to reach the optimal given $G$, then the optimal distribution $p^{*}_g(J)$ converges to $1$.
\end{theorem}
\begin{proof}
We follow \nj{the} same steps of the proof \nj{for Theorem~\ref{sec:theorem1}}.
In this case also, the supremum of $U(p_g) = \sup_D L_g(D,p_g(J))$ is convex except the point $p_g(J) = 0$, as in proposition~\ref{prop:2_2} (negative log function).
Therefore, $p_g(J)$ converges to the global minimum $1$ with small iterative update of $p_g$, concluding the proof.
\end{proof}

\section{Further Experiments}
\label{app:further_exp}
\begin{table*}[t]
\centering
 \caption{Top-$1$ precision and mAP scores for performance enhancing and adversarial problem for Imagenet50 Dataset: \\
 $\gamma$ = 0.0001 for the enhancing problem, $\gamma$ = 3 for the adversarial problem.}
 \label{table:enhance_sup}
  \resizebox{0.99\linewidth}{!}{
\begin{tabular}{|l|l|l|l||l|l|l|}
\hline  
\multirow{2}{*}{Classifier} & \multicolumn{3}{|c||}{Enhancing problem} &\multicolumn{3}{c|}{Adversarial problem}\\ 
\cline{2-7}
 &  Vanilla (1)& Proposed-ls & Proposed-lr&  vanilla (2)& proposed-m (B)& Proposed-l (B)\\
\hline
ResNet50 & 72\% / 0.649 &91.5\% / 0.883 & 88.6\% / 0.875 & 94.4\% / 0.922 &10.8\% / 0.101 & 16.3\% / 0.156 \\
ResNet101& 71\% / 0.635 &89.0\% / 0.856 & 88.4\% / 0.854 & 95.7\% / 0.938 &11.6\% / 0.101 & 22.2\% / 0.206 \\
VGG16& 71\% / 0.616 &93.4\% / 0.894 & 92.6\% / 0.891 & 88.5\% / 0.855 &3.20\% / 0.027 & 8.20\% / 0.078\\
DenseNet169& 74\% / 0.626 &92.1\% / 0.861 & 94.2\% / 0.919 & 95.5\% / 0.927 &6.56\% / 0.063 & 18.4\% / 0.160\\
\hline
\end{tabular}} 
\end{table*}
\begin{figure*}[t]
\centering
\begin{subfigure}[t]{0.03\textwidth}
\textbf{(A)}
\end{subfigure}
\begin{subfigure}[t]{0.29\textwidth}
\includegraphics[width=\linewidth,valign=t]{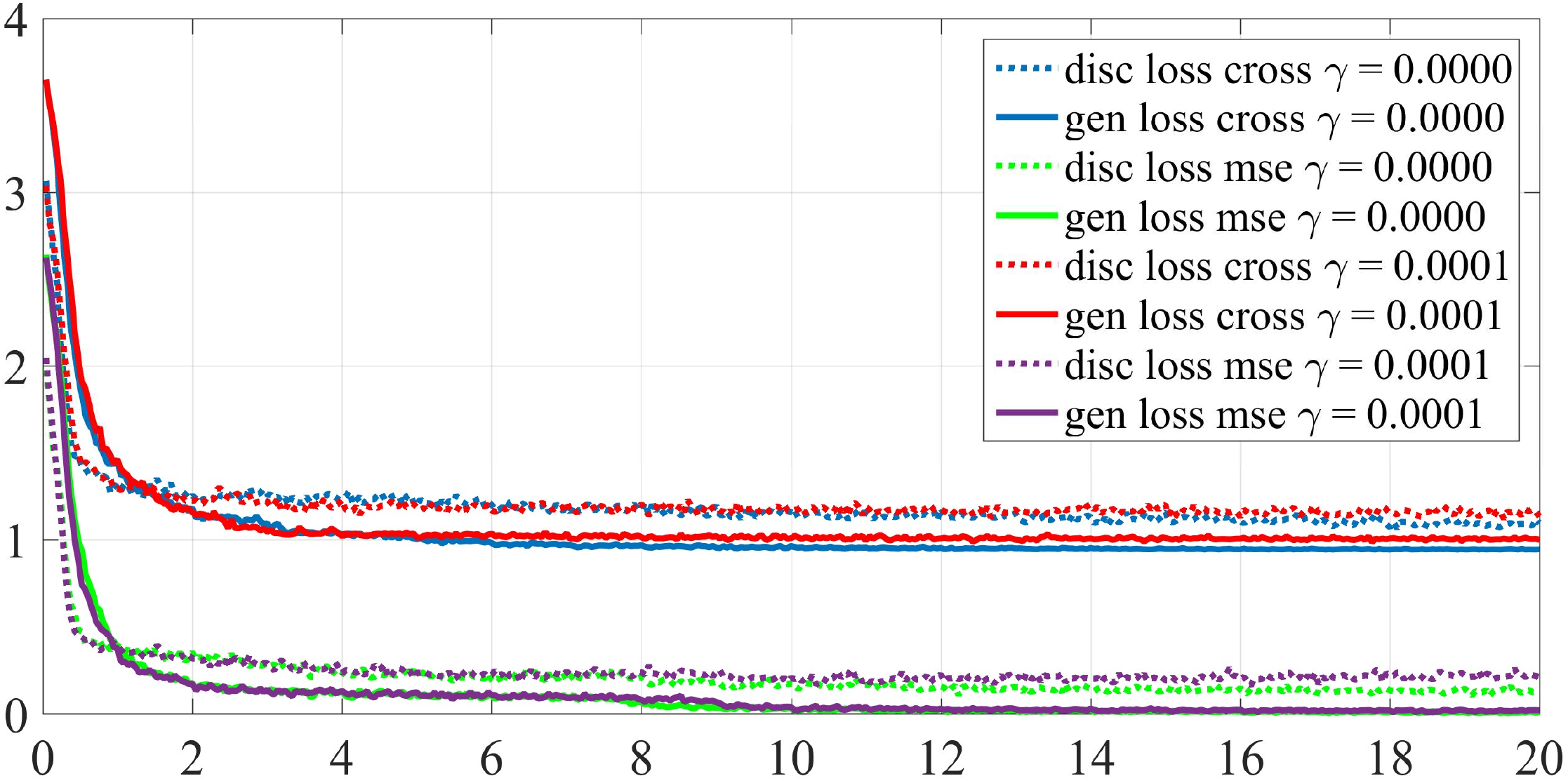}
\end{subfigure}\hfill
\begin{subfigure}[t]{0.03\textwidth}
\textbf{(B)}
\end{subfigure}
\begin{subfigure}[t]{0.29\textwidth}
\includegraphics[width=\linewidth,valign=t]{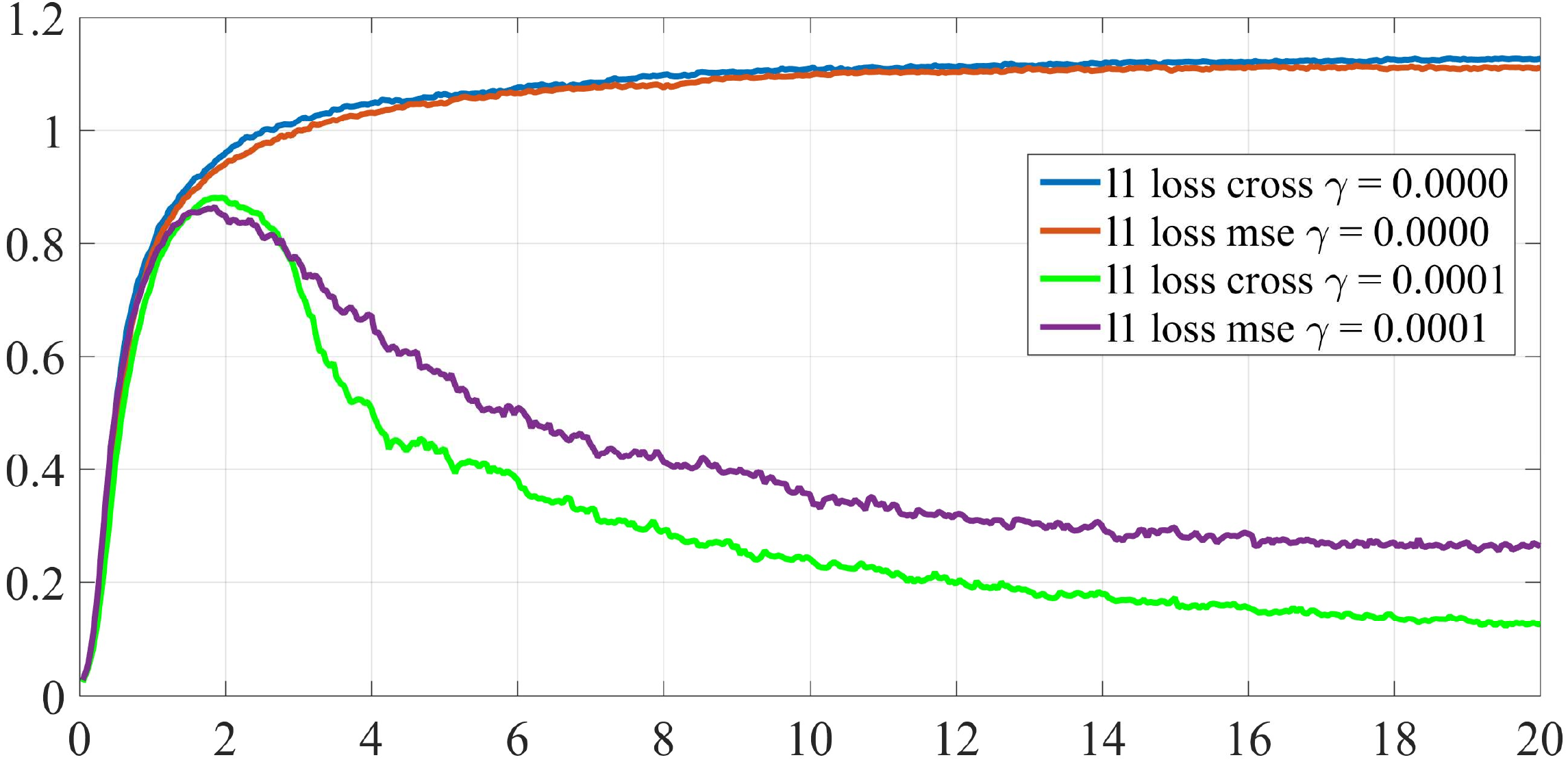}
\end{subfigure}\hfill
\begin{subfigure}[t]{0.03\textwidth}
\textbf{(C)}
\end{subfigure}
\begin{subfigure}[t]{0.29\textwidth}
\includegraphics[width=\linewidth,valign=t]{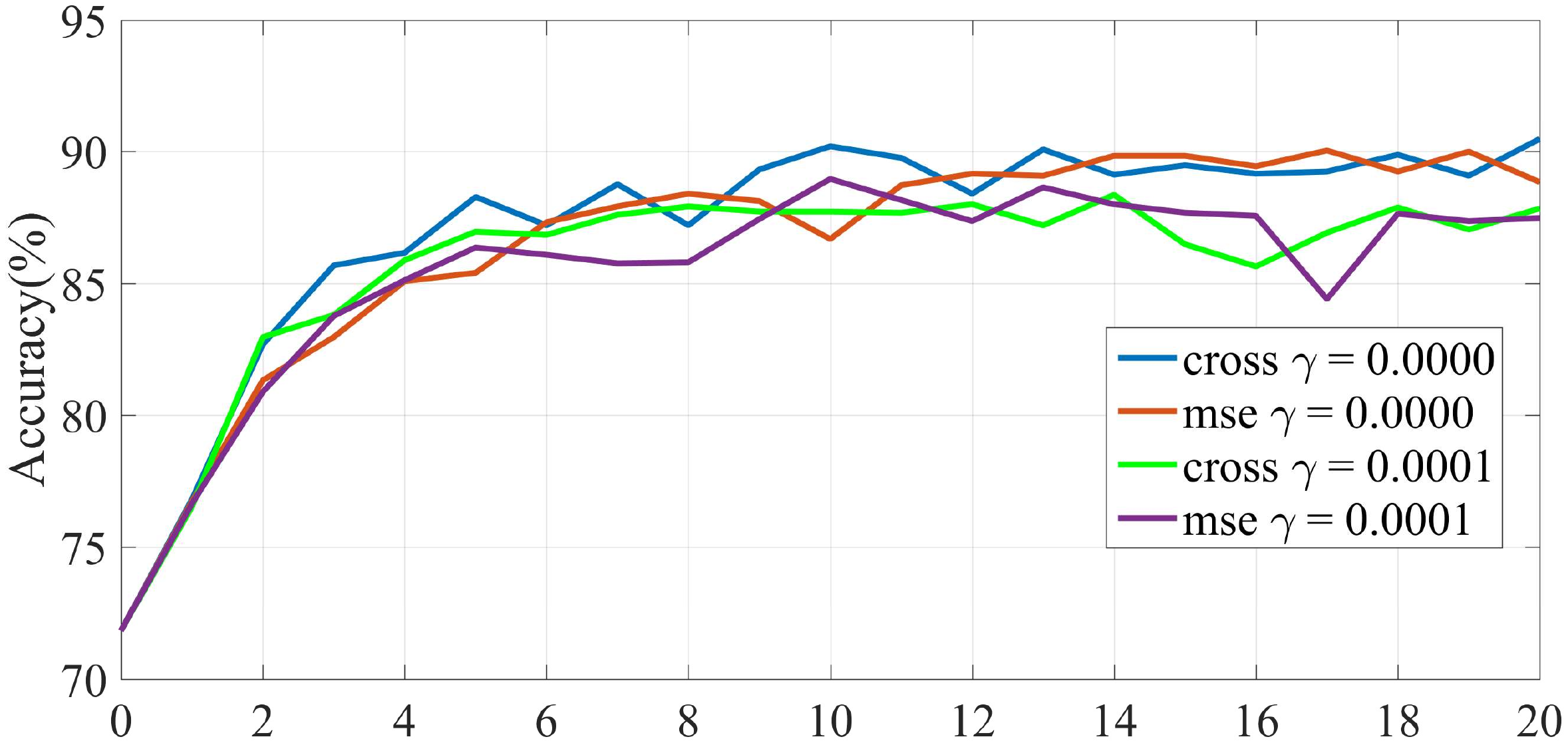}
\end{subfigure}\hfill
\caption[This is for my LOF]{Graphs describing the convergence and classification performance enhancement of the proposed algorithm with different $\gamma$:
(A) discriminator loss and generator loss, (B) $\ell$1-loss, (C) Accuracy.
Horizontal axis denotes epoch.
The experiments were performed \knj{with the ResNet101 classifier on ImageNet50} using the `Proposed-50'.
}
\label{fig:exp_graph_enhance_sup}
\end{figure*}
\begin{figure*}[t]
\centering
\begin{subfigure}[t]{0.03\textwidth}
\textbf{(A)}
\end{subfigure}
\begin{subfigure}[t]{0.29\textwidth}
\includegraphics[width=\linewidth,valign=t]{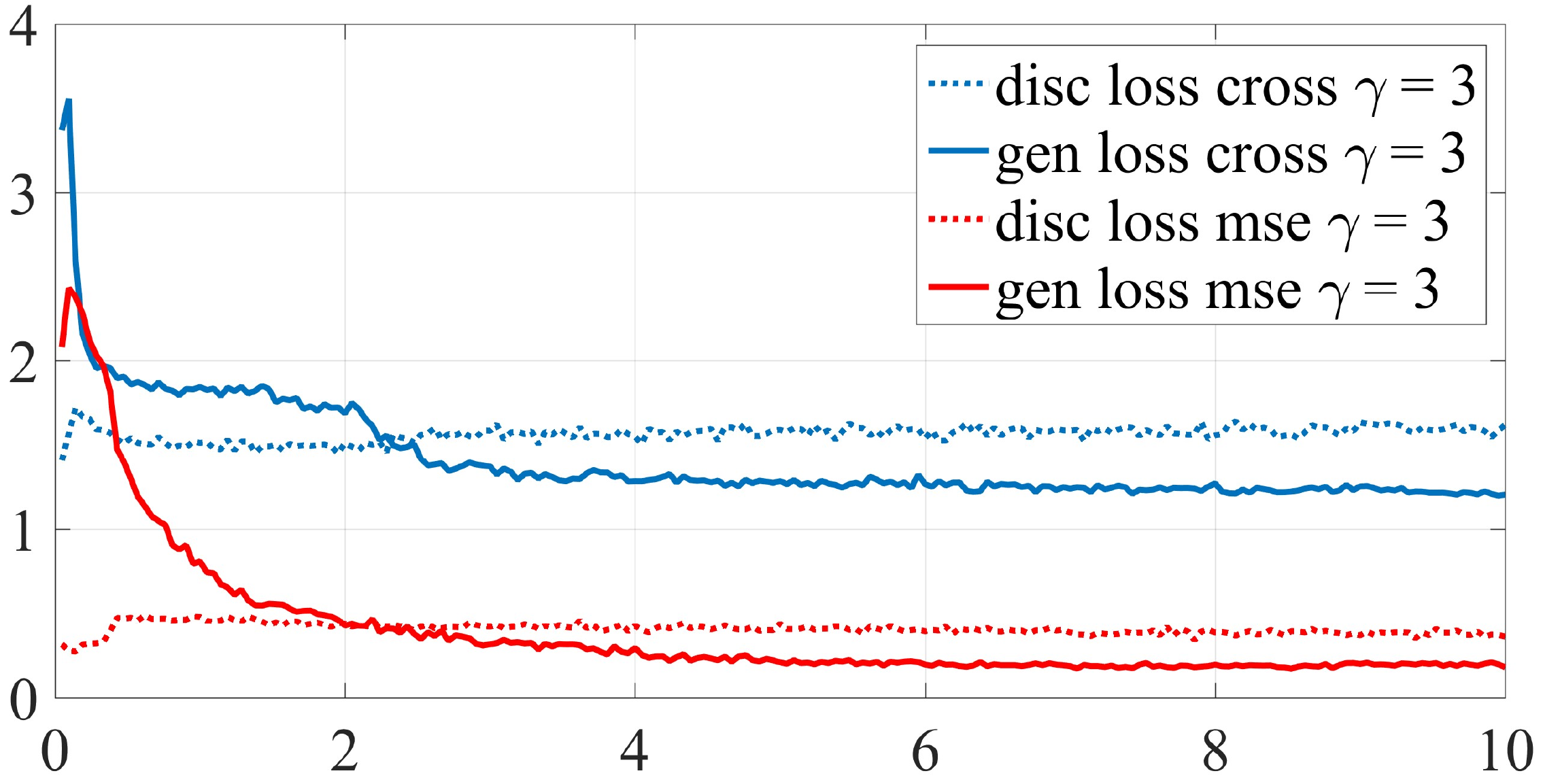}
\end{subfigure}\hfill
\begin{subfigure}[t]{0.03\textwidth}
\textbf{(B)}
\end{subfigure}
\begin{subfigure}[t]{0.29\textwidth}
\includegraphics[width=\linewidth,valign=t]{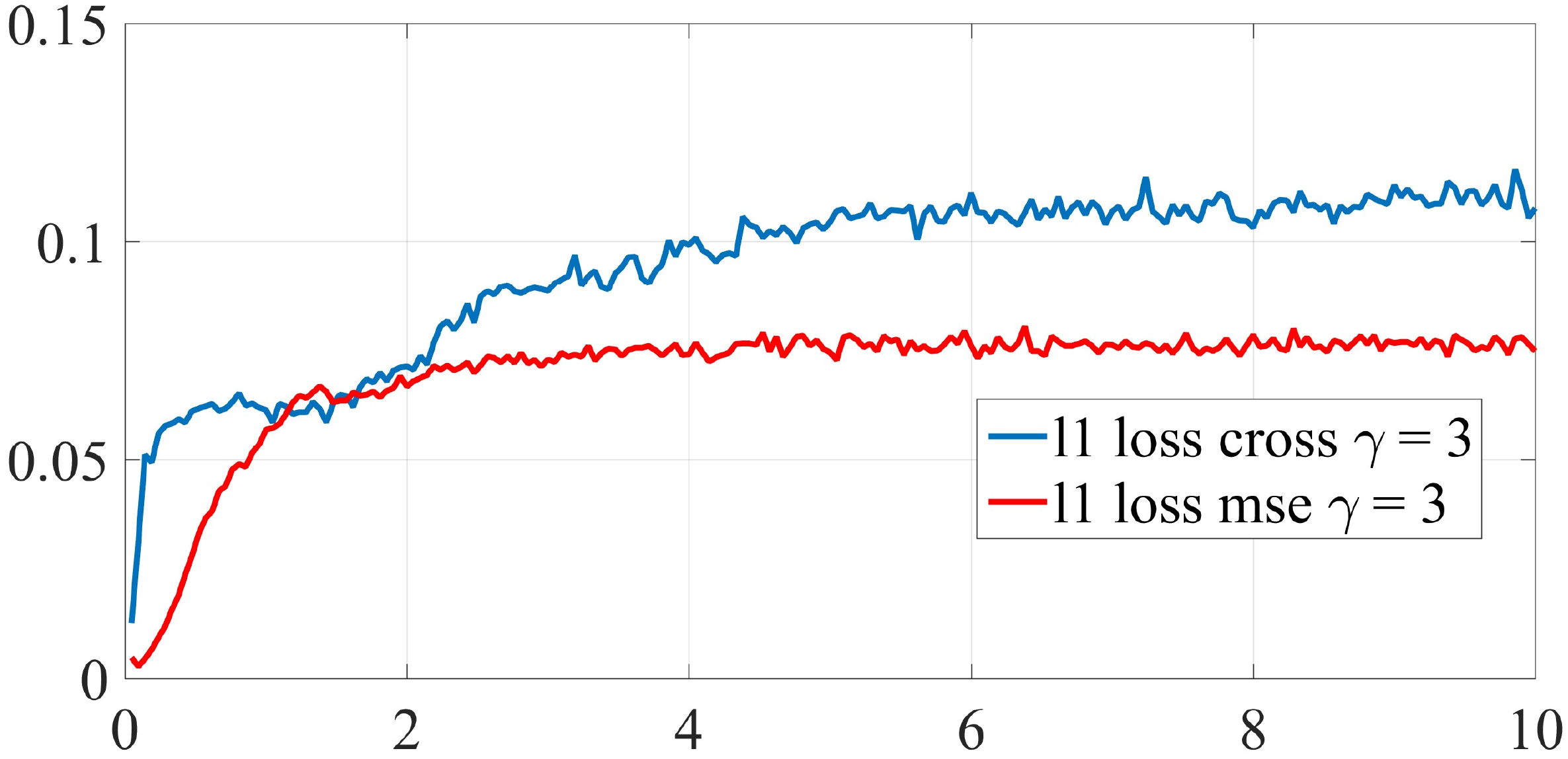}
\end{subfigure}\hfill
\begin{subfigure}[t]{0.03\textwidth}
\textbf{(C)}
\end{subfigure}
\begin{subfigure}[t]{0.29\textwidth}
\includegraphics[width=\linewidth,valign=t]{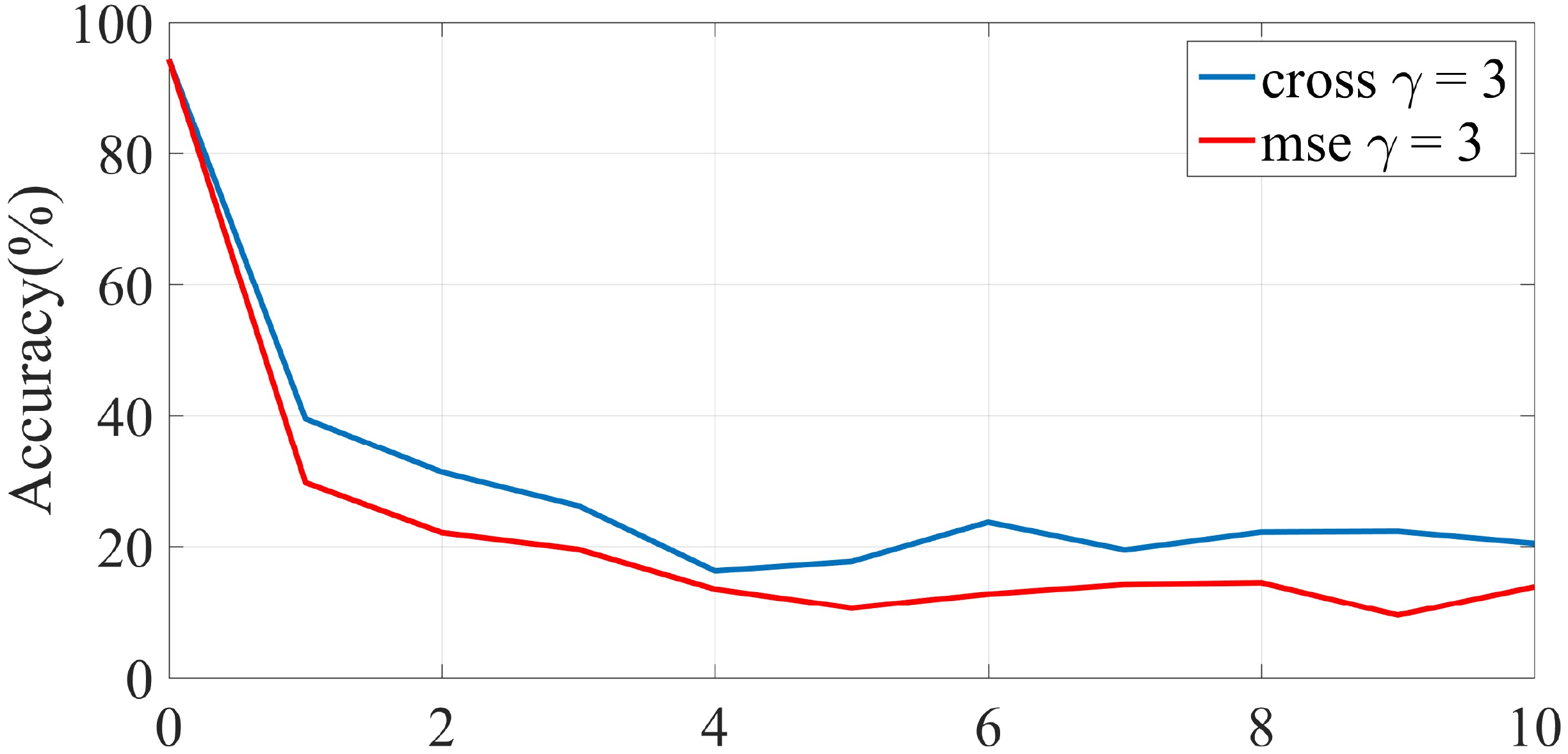}
\end{subfigure}\hfill
\caption[This is for my LOF]{Graphs describing the convergence and classification performance degradation of the proposed algorithm with different $\gamma$:
(A) discriminator loss and generator loss, (B) $\ell$1-loss, (C) Accuracy.
Horizontal axis denotes epoch.
The experiments were performed \knj{with the ResNet50 classifier on ImageNet50} using the `Proposed-B'.
}
\label{fig:exp_graph_adv_sup}
\end{figure*}
Table~\ref{table:enhance_sup} shows the performance enhancement and degradation of the classifiers by proposed algorithm each applying the two different loss functions: least square loss (proposed-ls), and cross-entropy loss(proposed-lr).
The experiment was tested with Imagnet50, the largest dataset in the paper.
For the enhancement case, we confirmed that the `proposed-lr' obtained comparable results to the `proposed-ls'.
In the degradation case, 
the `proposed-lr' also dropped the classification performance significantly, but the decrease was smaller than that of `proposed-ls' case.

In graphs in Figure~\ref{fig:exp_graph_enhance_sup}, changes of the losses and classification accuracy over epoch for the enhancement problem are presented.
As seen in the graph (A), we have confirmed that both generator and discriminator losses are converged for both `proposed-ls' and `proposed-lr' cases. 
We note that the negative log losses in `proposed-lr' do not converge to zero.
Interesting thing is that $\ell1$-loss has lower value in `proposed-lr' than `proposed-ls'.
Since the performance difference between `proposed-lr' and `proposed-ls' is not that significant, as in graph (C), we can conclude that we can enhance the classification performance with smaller perturbation when using the cross-entropy loss than the least square loss. 
The graphs also show the changes of the losses in the case $\ell1$-regularization term was detached.
In this case, the $\ell1$-intensity of the perturbation was converged to a specific value (about 1.1 as in graph (B)), as mentioned in the paper.
We also confirmed from graphs (C) and (D) that excluding the $\ell1$-regularization improved the enhancement performance, but the increase was insignificant.

Graphs in Figure~\ref{fig:exp_graph_adv_sup} show the same changes presented in Figure~\ref{fig:exp_graph_adv_qual} for the adversarial problem.
From the graph (A), we have confirmed that the generator loss and the discriminator loss both converge for both `proposed-ls' and `proposed-lr'.
Different from the performance enhancement problem, the discriminator loss converged very fast, which is also reported in the paper.
What is noteworthy is that `proposed-ls' achieved better degradation performance than `proposed-lr' with small $\ell1$ intensity of the perturbation (see graphs (B) and (C)).
In the adversarial problem case, it seems that applying least square loss can be more efficient choice than applying cross-entropy loss.

\begin{figure*}[t]
\centering
\begin{subfigure}[t]{0.03\textwidth}
\textbf{(A)}
\end{subfigure}
\begin{subfigure}[t]{0.21\textwidth}
\includegraphics[width=\linewidth,valign=t]{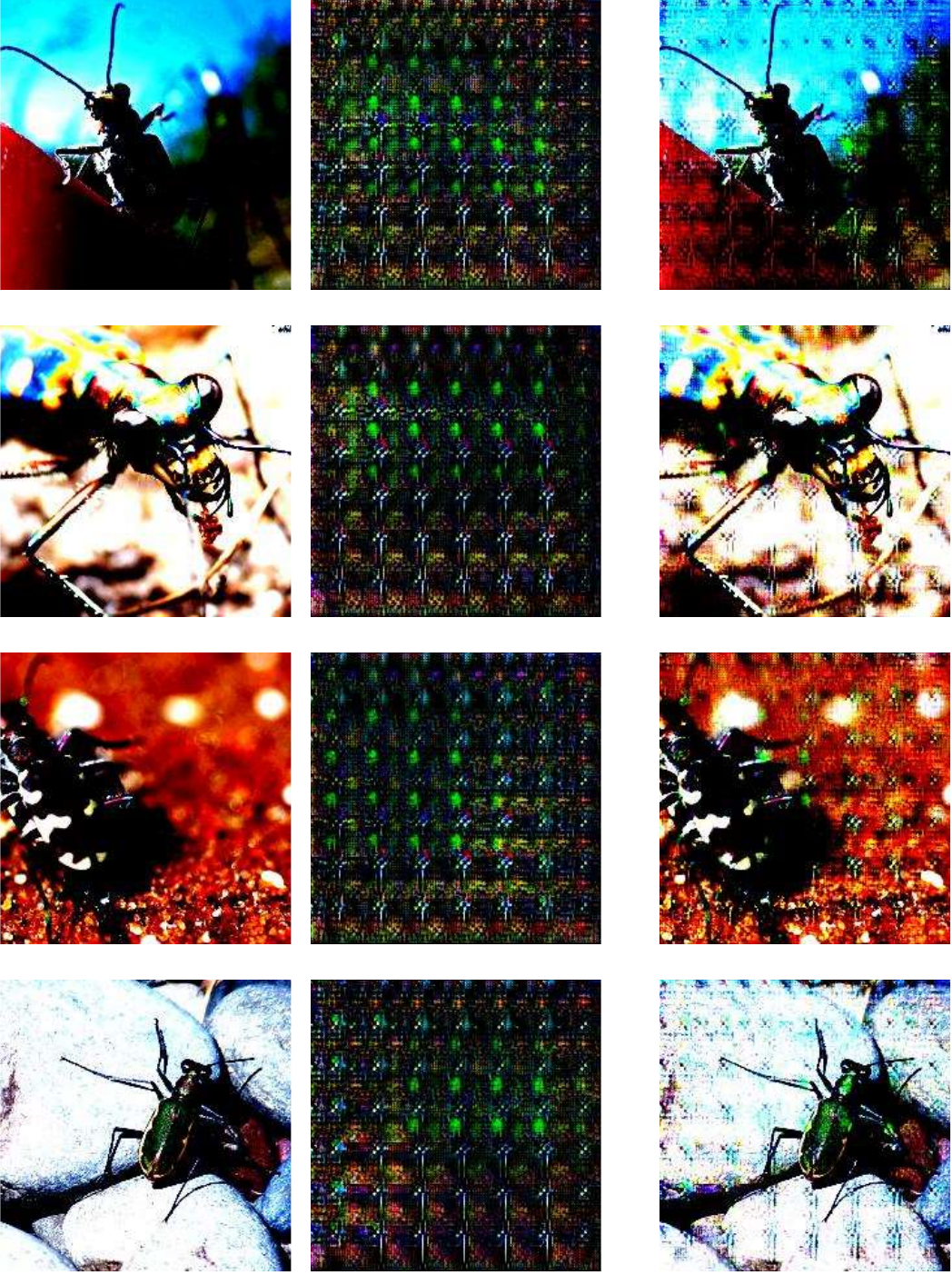}
\end{subfigure}\hfill
\begin{subfigure}[t]{0.03\textwidth}
\textbf{(B)}
\end{subfigure}
\begin{subfigure}[t]{0.21\textwidth}
\includegraphics[width=\linewidth,valign=t]{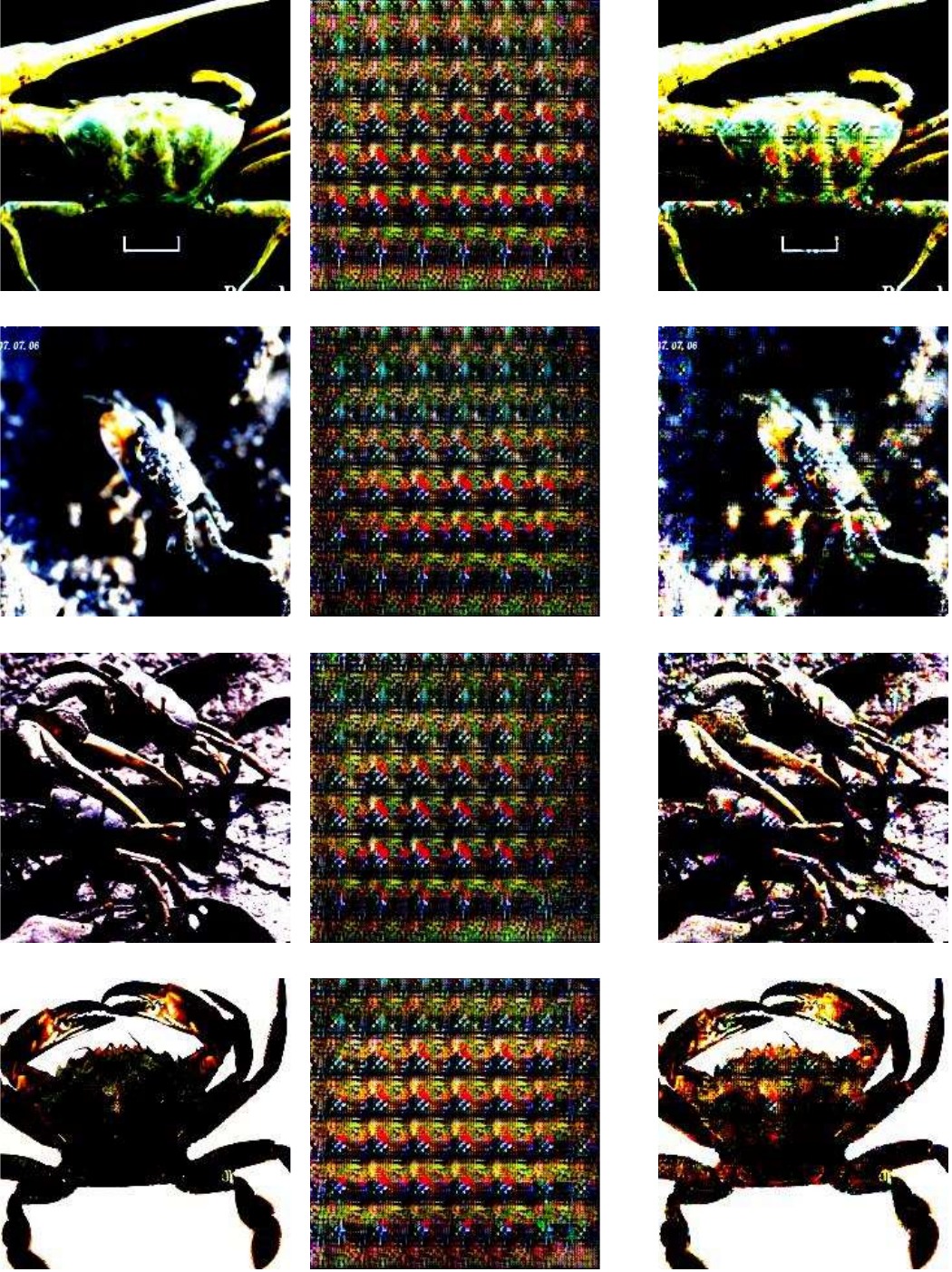}
\end{subfigure}\hfill
\begin{subfigure}[t]{0.03\textwidth}
\textbf{(C)}
\end{subfigure}
\begin{subfigure}[t]{0.21\textwidth}
\includegraphics[width=\linewidth,valign=t]{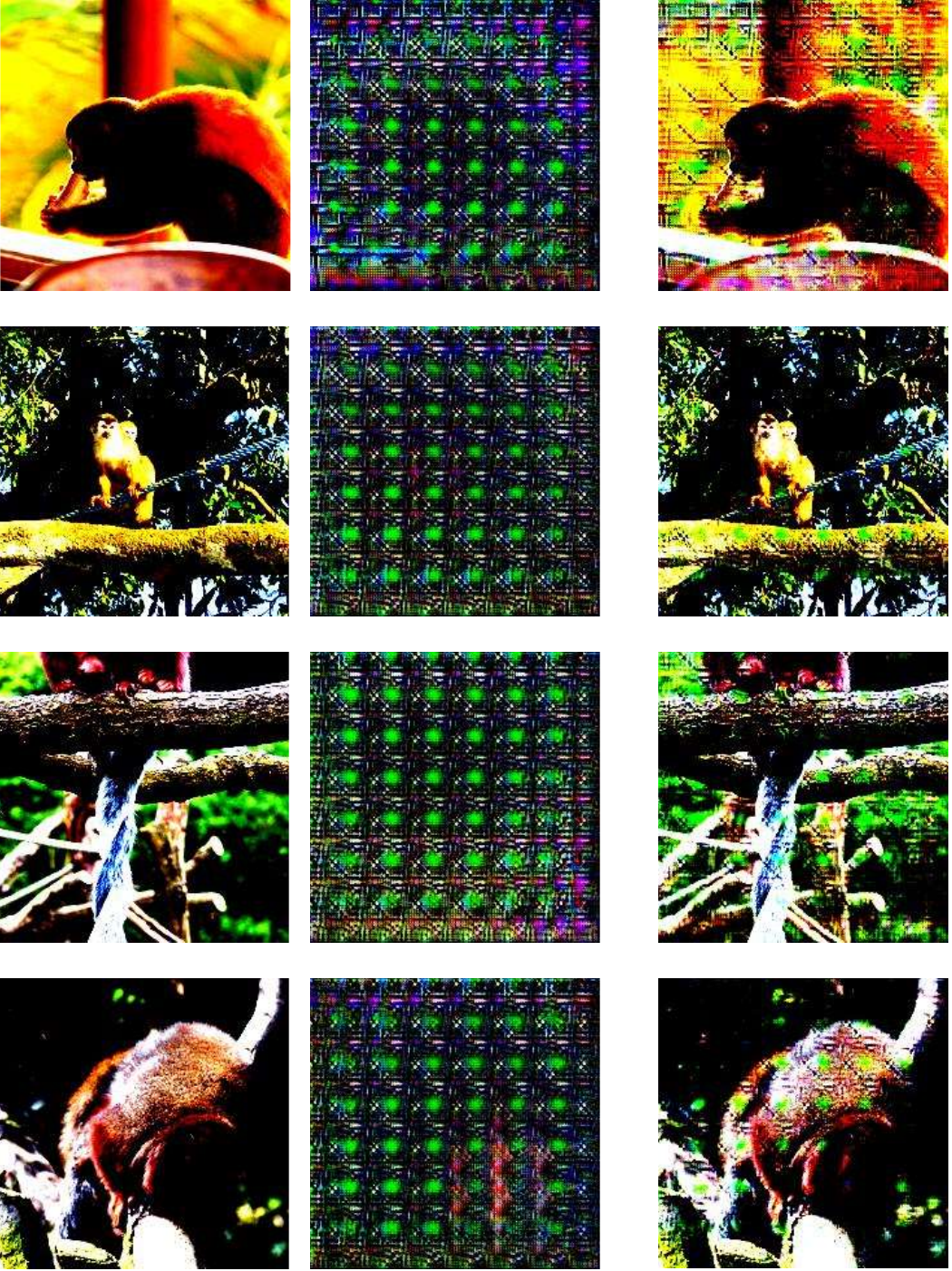}
\end{subfigure}\hfill
\begin{subfigure}[t]{0.03\textwidth}
\textbf{(D)}
\end{subfigure}
\begin{subfigure}[t]{0.21\textwidth}
\includegraphics[width=\linewidth,valign=t]{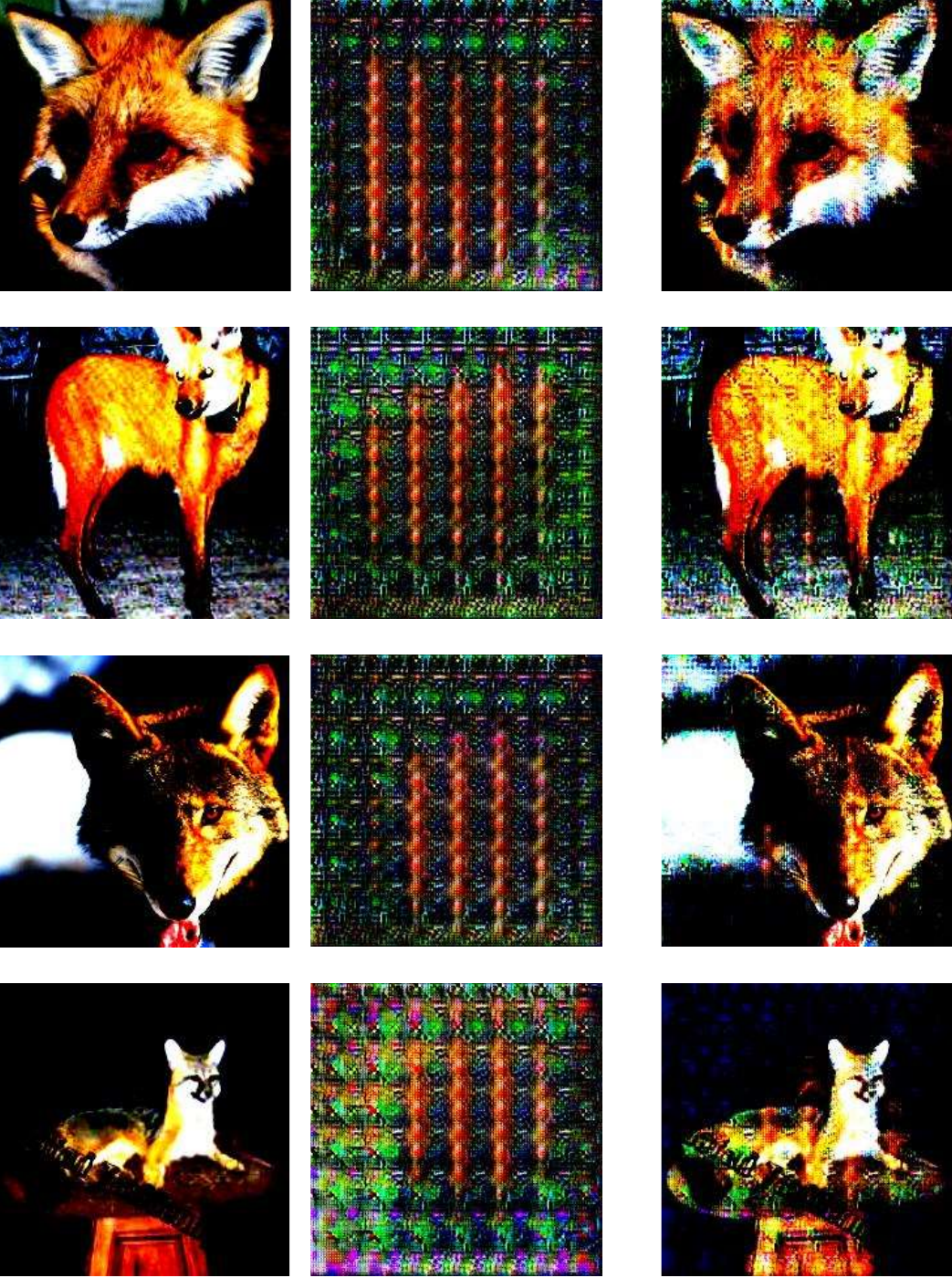}
\end{subfigure}\hfill
\caption[This is for my LOF]{Examples of generated perturbation from normalized images in diverse image classes for performance enhancement problem. (A) bug, (B) crab, (C) monkey, and (D) fox.
}
\label{fig:exp_graph_enhance_qual}
\end{figure*}
\begin{figure*}[t]
\centering
\begin{subfigure}[t]{0.03\textwidth}
\textbf{(A)}
\end{subfigure}
\begin{subfigure}[t]{0.21\textwidth}
\includegraphics[width=\linewidth,valign=t]{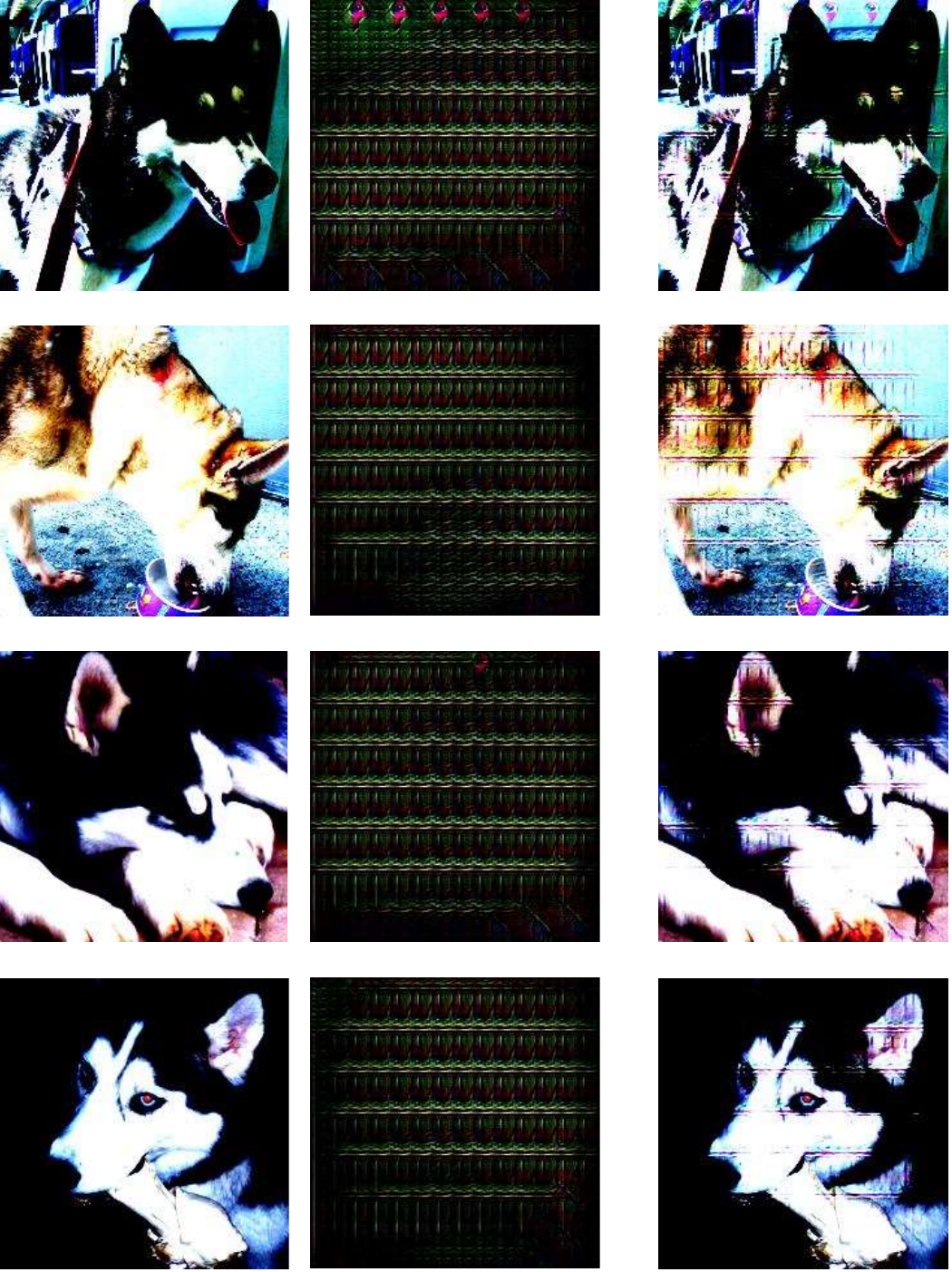}
\end{subfigure}\hfill
\begin{subfigure}[t]{0.03\textwidth}
\textbf{(B)}
\end{subfigure}
\begin{subfigure}[t]{0.21\textwidth}
\includegraphics[width=\linewidth,valign=t]{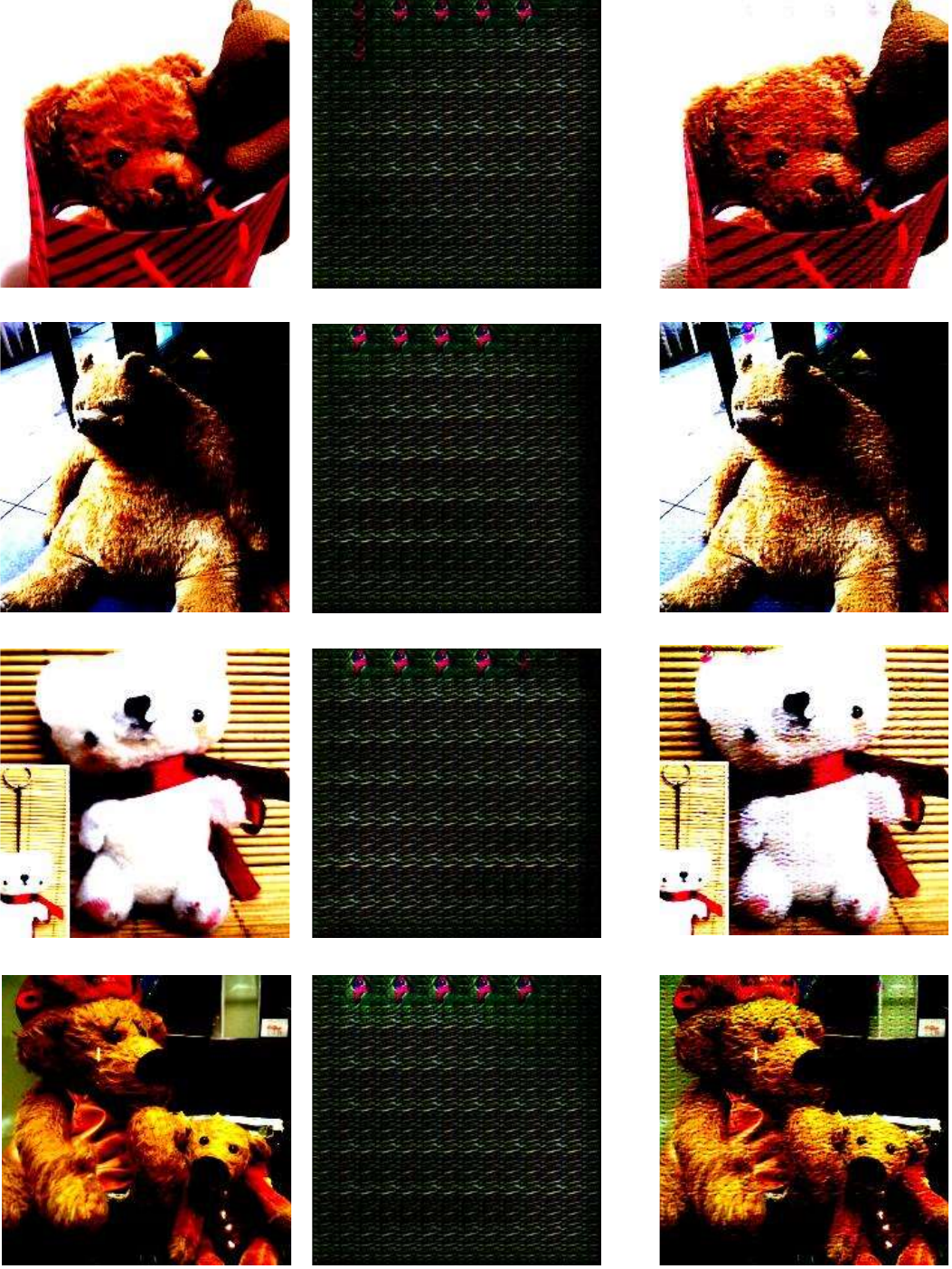}
\end{subfigure}\hfill
\begin{subfigure}[t]{0.03\textwidth}
\textbf{(C)}
\end{subfigure}
\begin{subfigure}[t]{0.21\textwidth}
\includegraphics[width=\linewidth,valign=t]{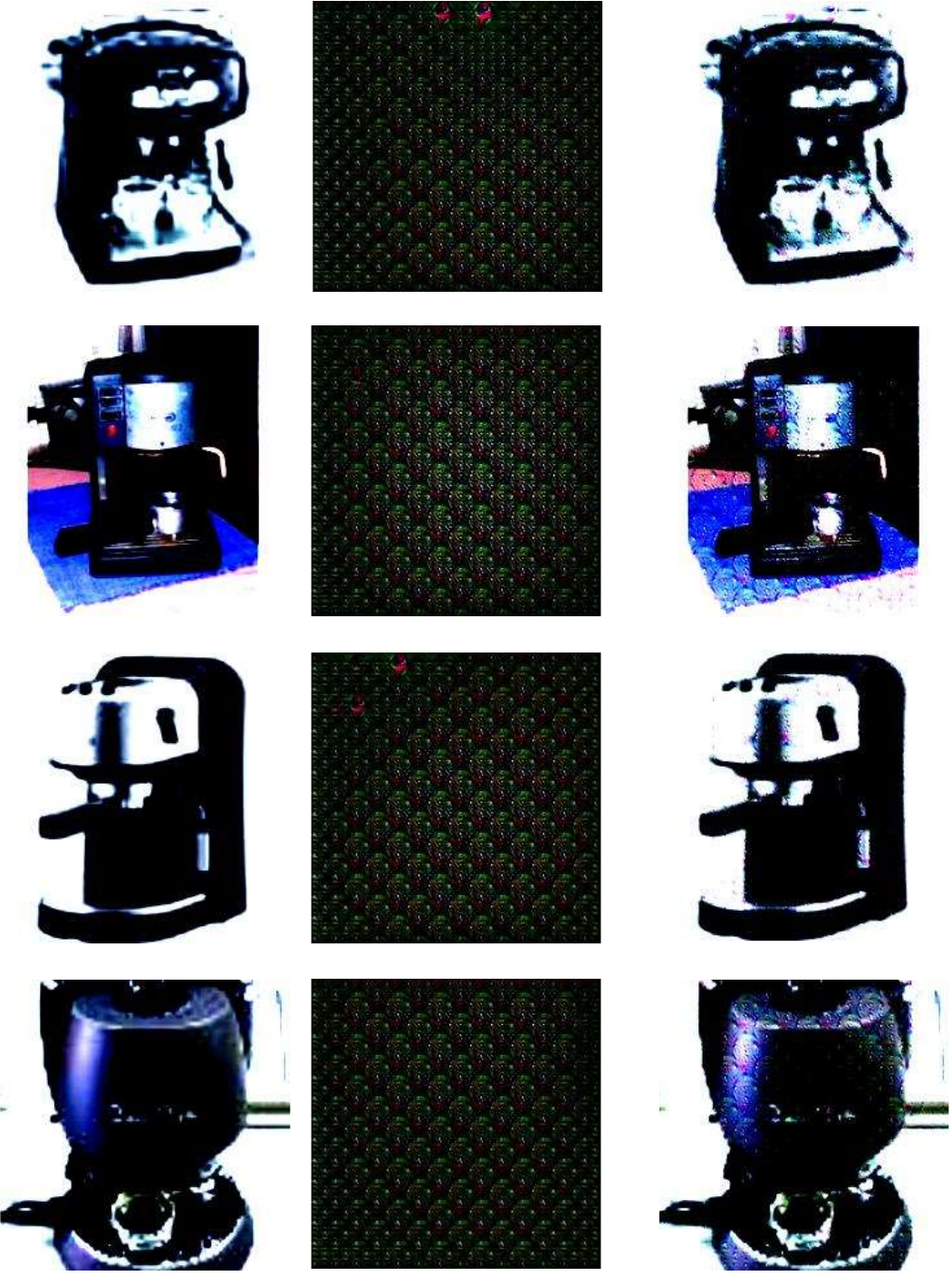}
\end{subfigure}\hfill
\begin{subfigure}[t]{0.03\textwidth}
\textbf{(D)}
\end{subfigure}
\begin{subfigure}[t]{0.21\textwidth}
\includegraphics[width=\linewidth,valign=t]{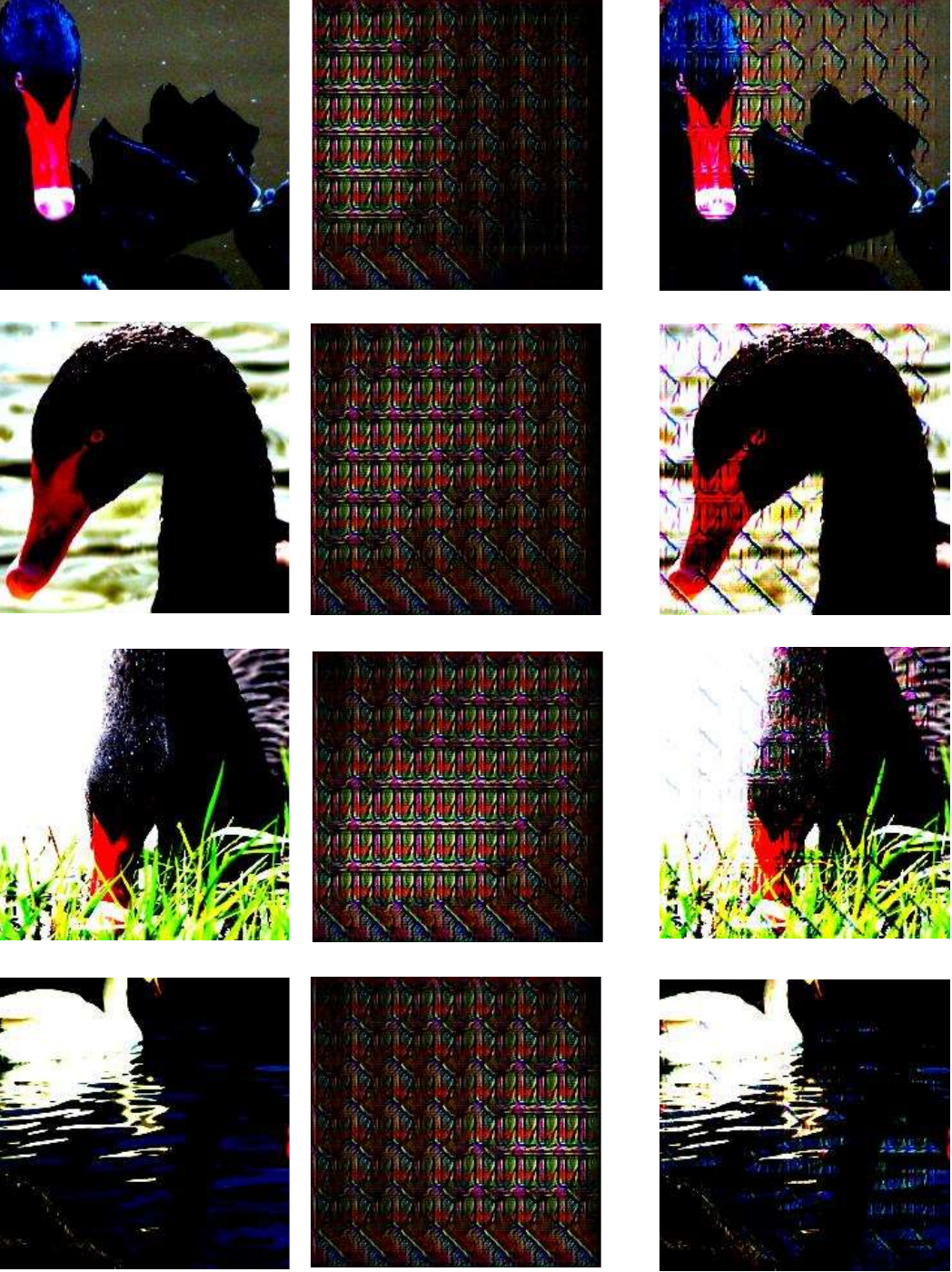}
\end{subfigure}\hfill
\caption[This is for my LOF]{Examples of generated perturbation from normalized images in diverse image classes for performance degradation problem. (A) dog, (B) teddy bear, (C) coffee machine, and (D) swan.
}
\label{fig:exp_graph_adv_qual}
\end{figure*}
\begin{figure*}[t]
\centering
\begin{subfigure}[t]{0.03\textwidth}
\textbf{(A)}
\end{subfigure}
\begin{subfigure}[t]{0.21\textwidth}
\includegraphics[width=\linewidth,valign=t]{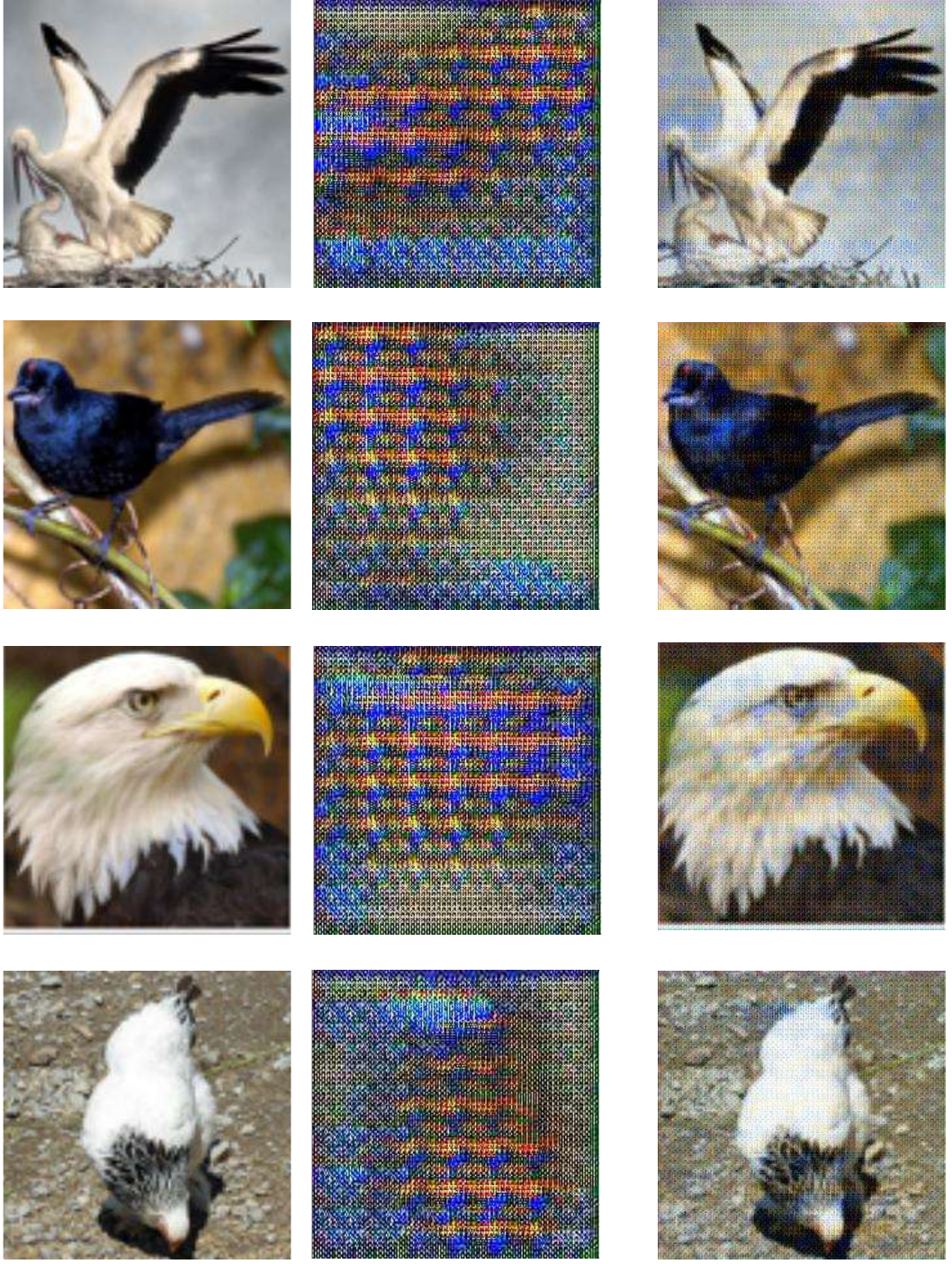}
\end{subfigure}\hfill
\begin{subfigure}[t]{0.03\textwidth}
\textbf{(B)}
\end{subfigure}
\begin{subfigure}[t]{0.21\textwidth}
\includegraphics[width=\linewidth,valign=t]{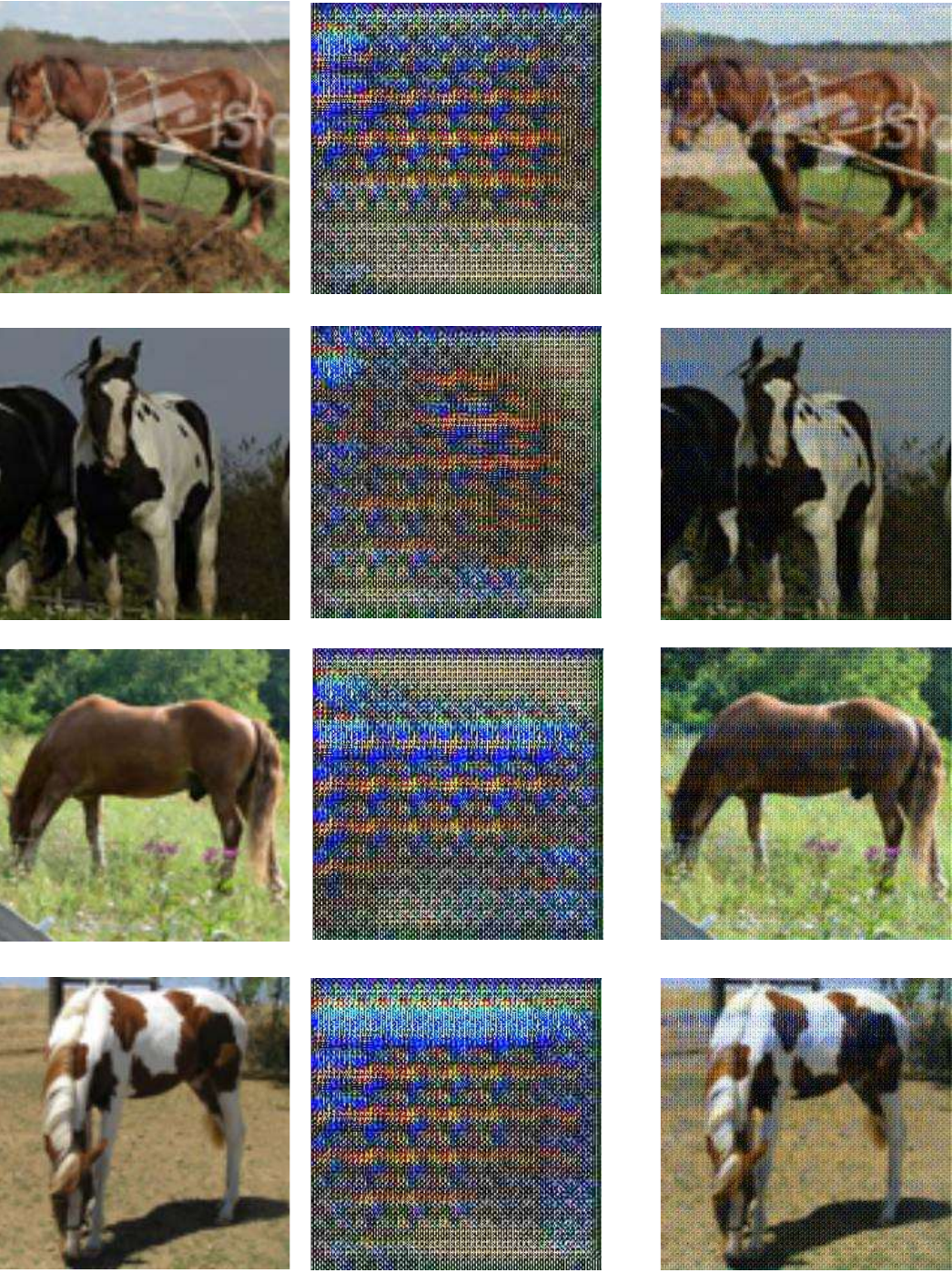}
\end{subfigure}\hfill
\begin{subfigure}[t]{0.03\textwidth}
\textbf{(C)}
\end{subfigure}
\begin{subfigure}[t]{0.21\textwidth}
\includegraphics[width=\linewidth,valign=t]{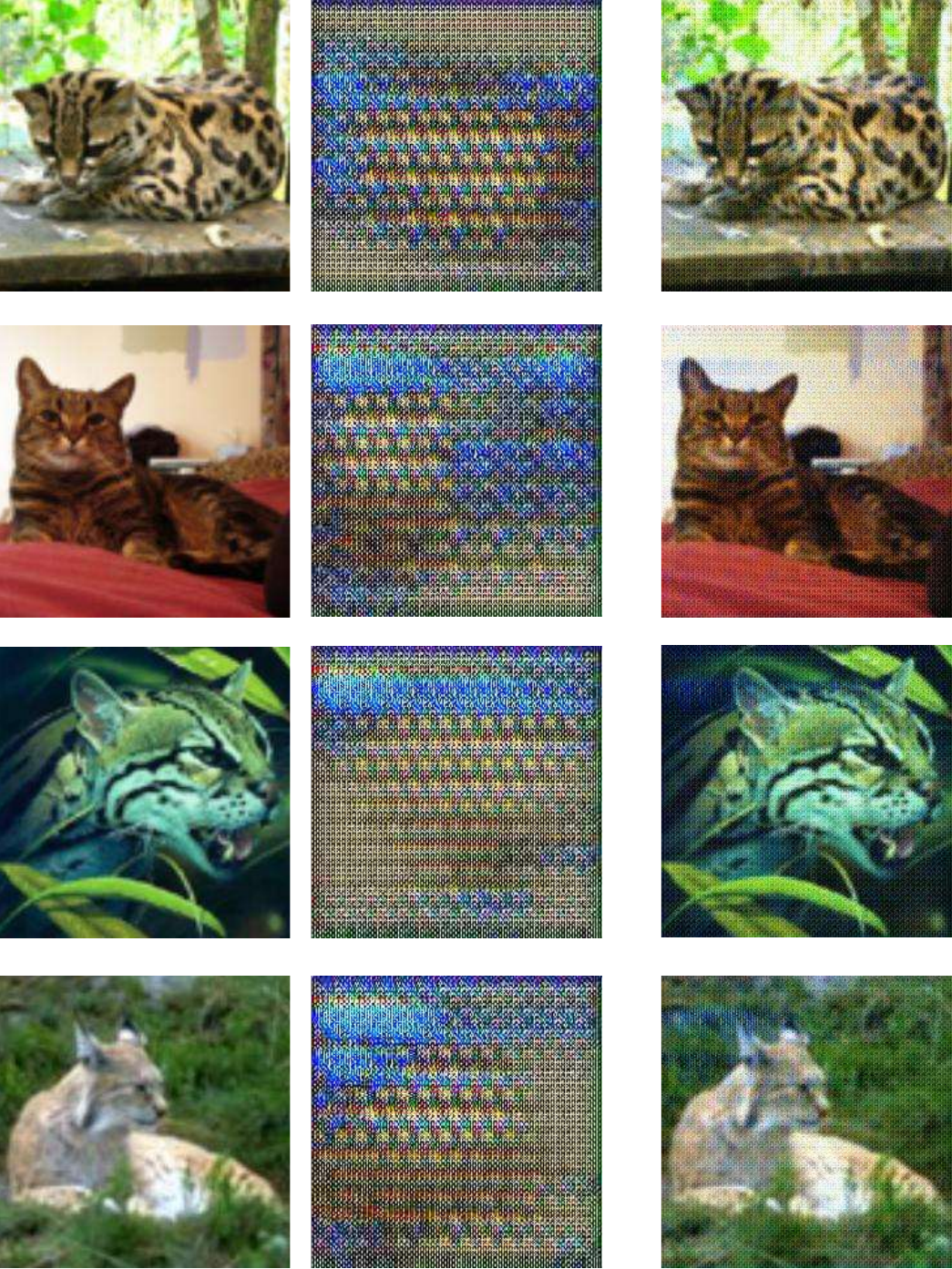}
\end{subfigure}\hfill
\begin{subfigure}[t]{0.03\textwidth}
\textbf{(D)}
\end{subfigure}
\begin{subfigure}[t]{0.21\textwidth}
\includegraphics[width=\linewidth,valign=t]{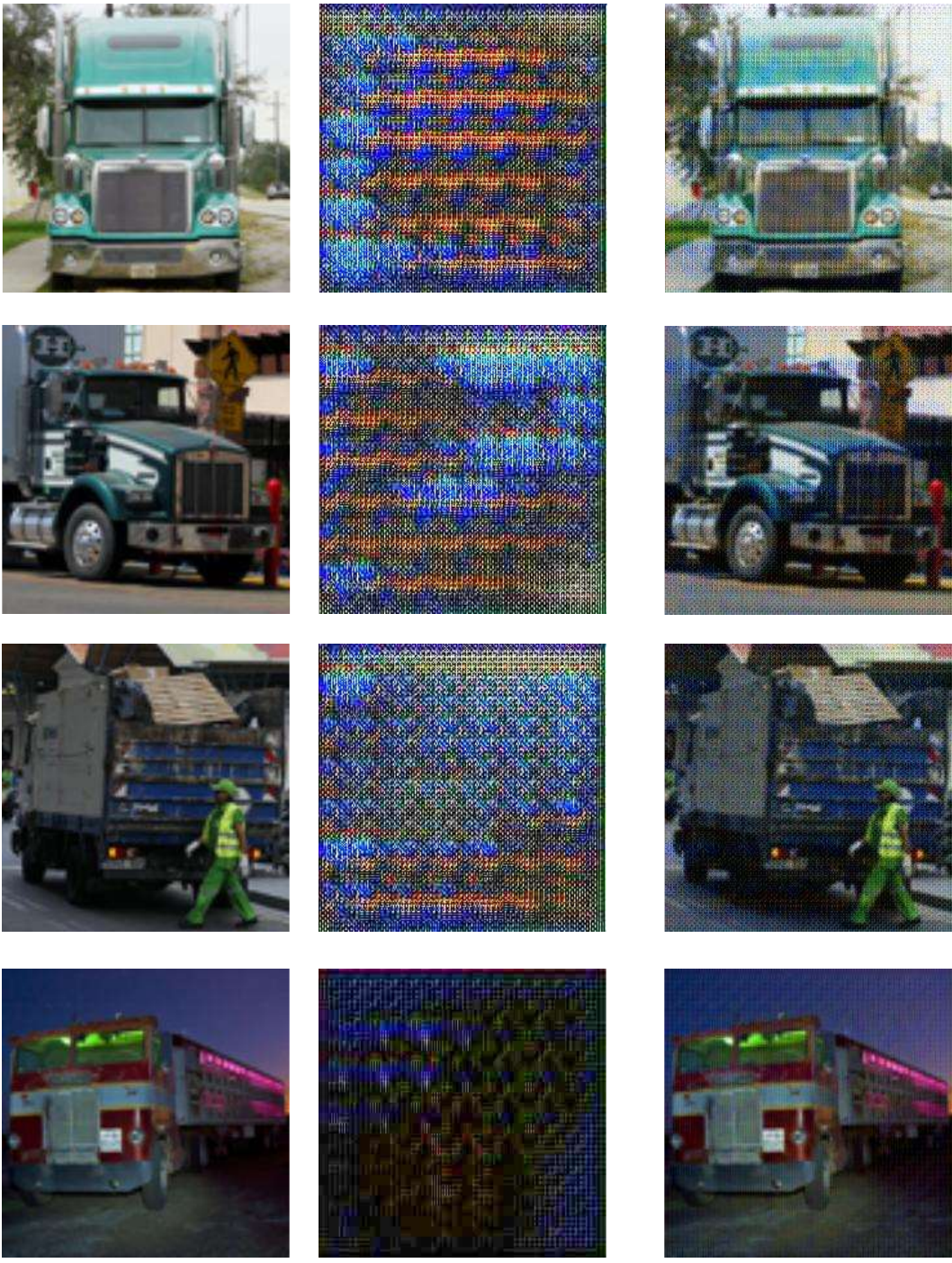}
\end{subfigure}\hfill
\caption[This is for my LOF]{Examples of generated perturbation from non-normalized images in diverse image classes for performance enhancement problem. (A) eagle, (B) horse, (C) cat, and (D) truck.
}
\label{fig:exp_graph_enhance_qual_vanilla}
\end{figure*}

In Figures~\ref{fig:exp_graph_enhance_qual},~\ref{fig:exp_graph_adv_qual},~\ref{fig:exp_graph_enhance_qual_vanilla}, and~\ref{fig:exp_graph_adv_qual_vanilla}, additional perturbations and corresponding perturbed images for diverse image classes are presented.
The perturbations in Figures~\ref{fig:exp_graph_enhance_qual},~\ref{fig:exp_graph_adv_qual} were generated from normalized images, and those in Figures~\ref{fig:exp_graph_enhance_qual_vanilla}, and~\ref{fig:exp_graph_adv_qual_vanilla} were generated from non-normalized images.
Figures~\ref{fig:exp_graph_enhance_qual} and \ref{fig:exp_graph_enhance_qual_vanilla} show the examples of performance enhancement problem, and Figures~\ref{fig:exp_graph_adv_qual} and \ref{fig:exp_graph_adv_qual_vanilla} present the examples of performance degradation problem.
It is interesting that the perturbations for each image class seem to share some similar visual characteristics among them when seeing Figure~\ref{fig:exp_graph_enhance_qual}.
For example, the perturbations generated from `fox' images and that from `bug' images have clearly discriminative shapes from each other.
\begin{figure*}[t]
\centering
\begin{subfigure}[t]{0.03\textwidth}
\textbf{(A)}
\end{subfigure}
\begin{subfigure}[t]{0.21\textwidth}
\includegraphics[width=\linewidth,valign=t]{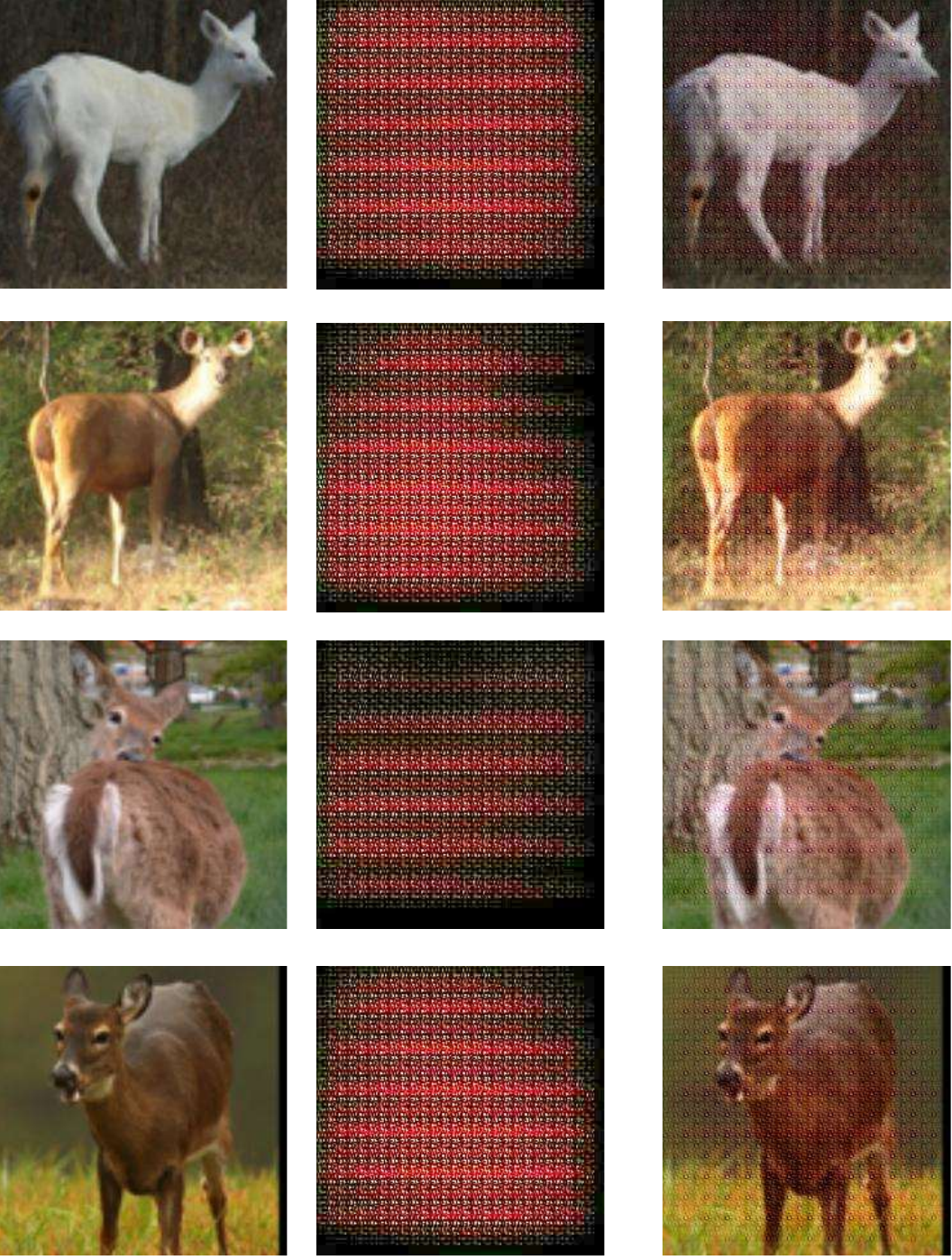}
\end{subfigure}\hfill
\begin{subfigure}[t]{0.03\textwidth}
\textbf{(B)}
\end{subfigure}
\begin{subfigure}[t]{0.21\textwidth}
\includegraphics[width=\linewidth,valign=t]{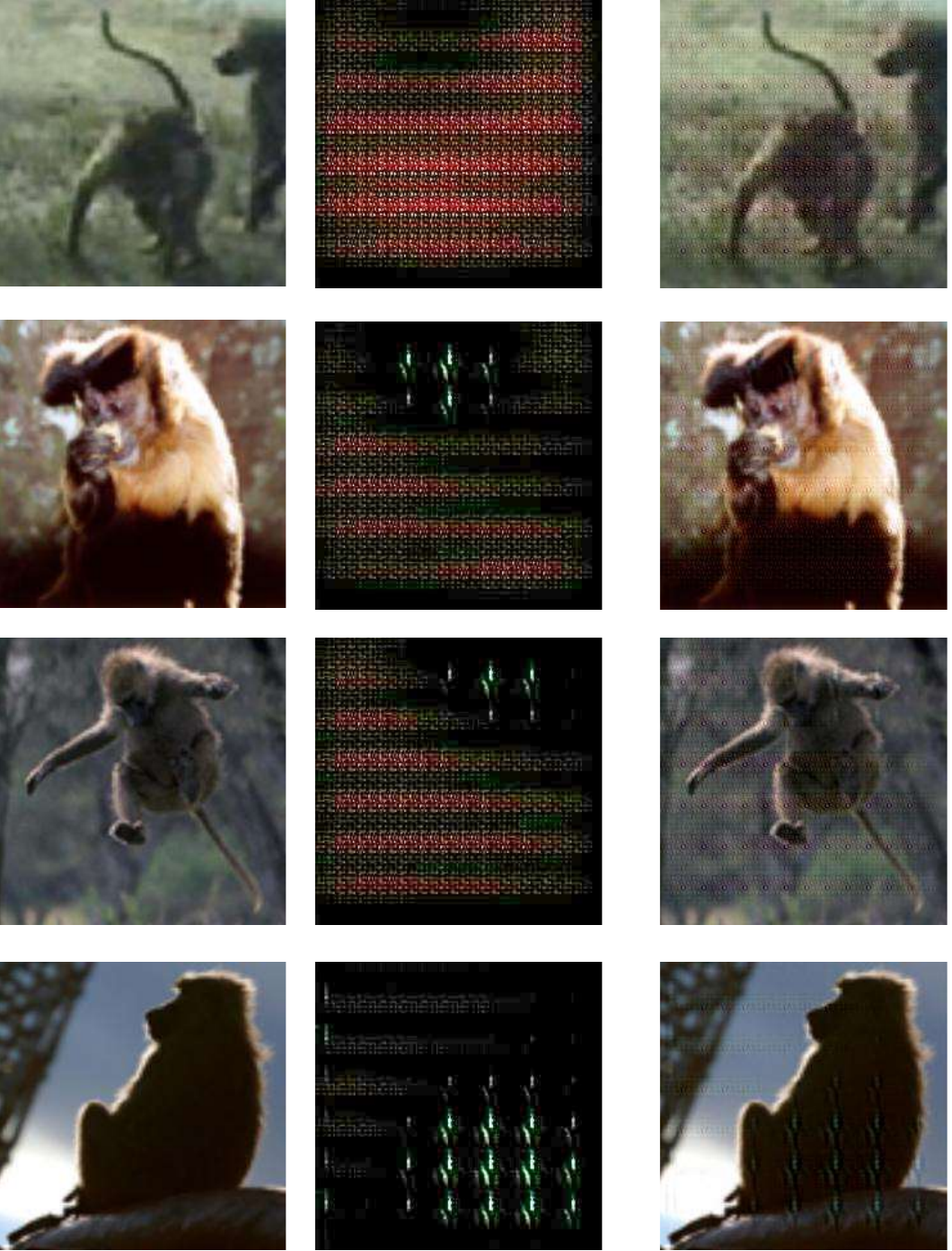}
\end{subfigure}\hfill
\begin{subfigure}[t]{0.03\textwidth}
\textbf{(C)}
\end{subfigure}
\begin{subfigure}[t]{0.21\textwidth}
\includegraphics[width=\linewidth,valign=t]{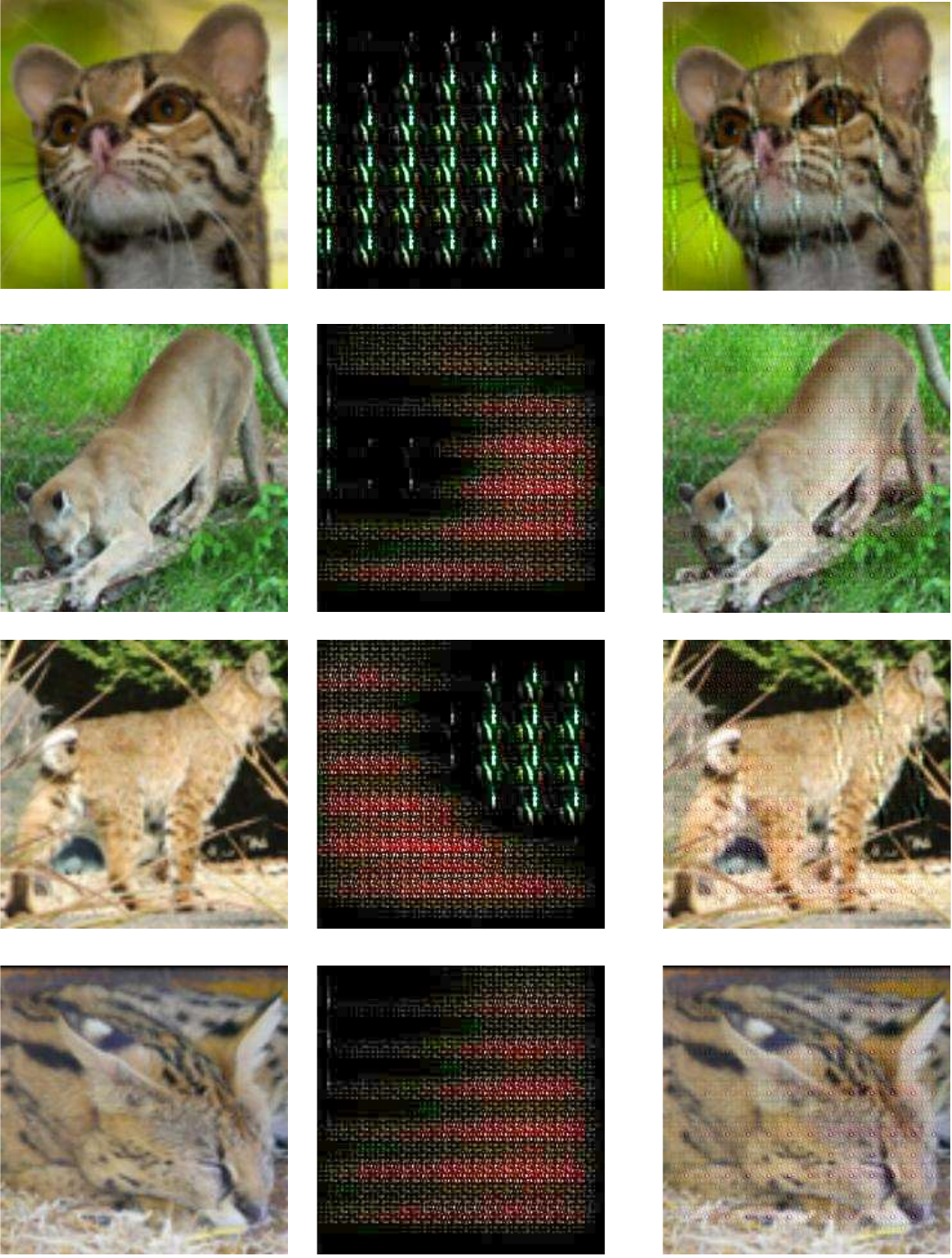}
\end{subfigure}\hfill
\begin{subfigure}[t]{0.03\textwidth}
\textbf{(D)}
\end{subfigure}
\begin{subfigure}[t]{0.21\textwidth}
\includegraphics[width=\linewidth,valign=t]{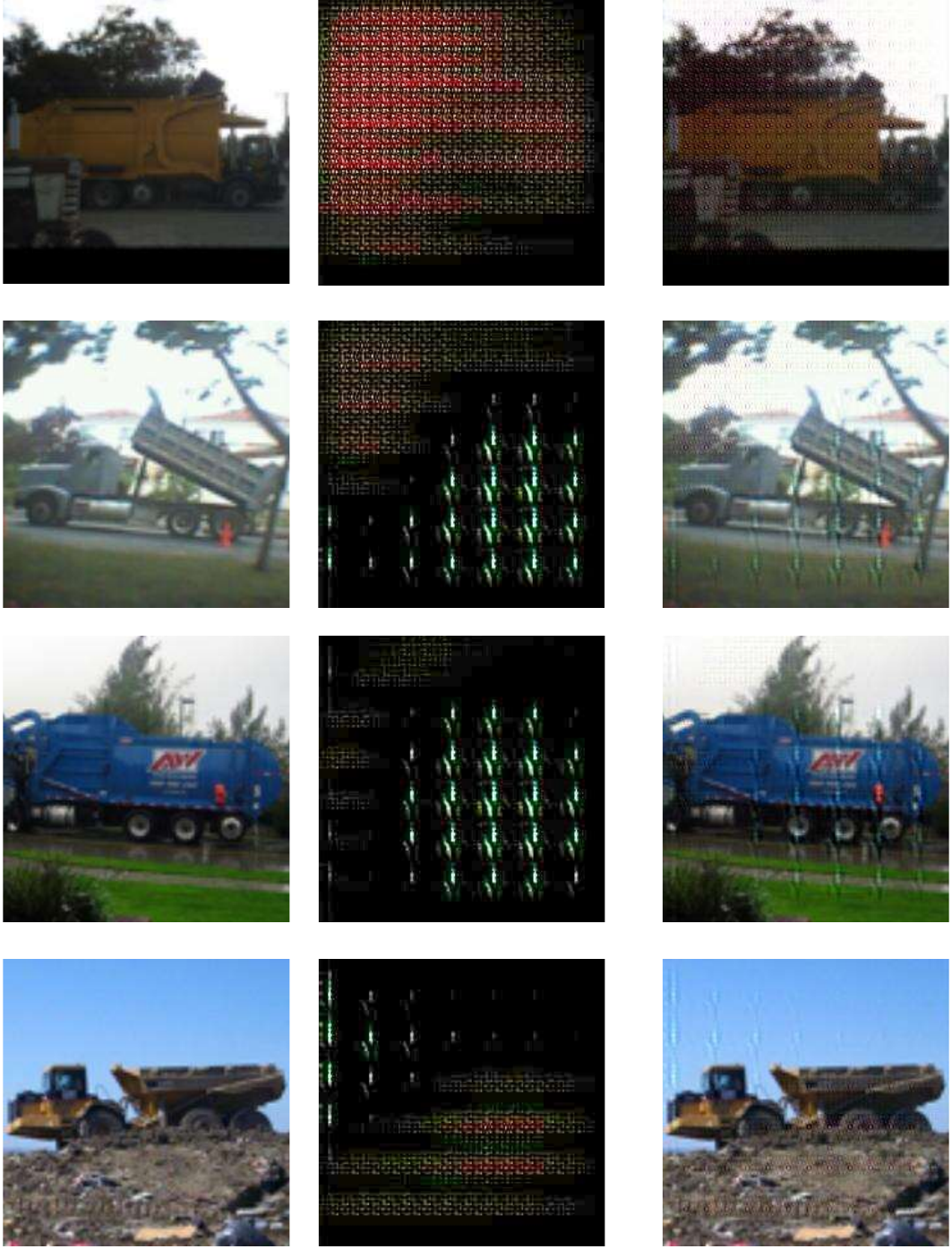}
\end{subfigure}\hfill
\caption[This is for my LOF]{Examples of generated perturbation from non-normalized images in diverse image classes for performance degradation problem. (A) deer, (B) monkey, (C) cat, and (D) truck.
}
\label{fig:exp_graph_adv_qual_vanilla}
\end{figure*}
\end{appendices}

\end{document}